\documentclass{article}

\PassOptionsToPackage{numbers, compress}{natbib}

\usepackage[preprint]{neurips_2026}

\usepackage[utf8]{inputenc} 
\usepackage[T1]{fontenc}    
\PassOptionsToPackage{unicode=false}{hyperref}
\usepackage{hyperref}       
\usepackage{url}            
\usepackage{booktabs}       
\usepackage{amsmath}        
\usepackage{amsfonts}       
\usepackage{amsthm}         
\newtheorem{lemma}{Lemma}
\newtheorem{proposition}{Proposition}

\usepackage{algorithm}      
\usepackage{algpseudocode}  

\usepackage{nicefrac}       
\usepackage{microtype}      
\usepackage{xcolor}         
\usepackage{placeins}       

\newif\ifreviewmode
\reviewmodefalse
\ifreviewmode
  \newcommand{\PPc}[1]{{\color{orange}[\textbf{PP}: #1]}}
  \newcommand{\PP}[1]{{\color{orange}{#1}}}
  \definecolor{nicolascolor}{RGB}{0,128,128}
  \newcommand{\nicolasc}[1]{{\color{nicolascolor}[\textbf{ND}: #1]}}
  \newcommand{\nicolas}[1]{{\color{nicolascolor}{#1}}}
\else
  \newcommand{\PPc}[1]{}
  \newcommand{\PP}[1]{#1}
  \newcommand{\nicolasc}[1]{}
  \newcommand{\nicolas}[1]{#1}
\fi

\newcommand{\myparagraph}[1]{\vspace{0.08cm}\noindent {\bf #1}~}

\usepackage{amsmath}
\usepackage{amssymb}
\usepackage{pgfplots}
\usepackage{pgfplotstable}
\pgfplotsset{compat=1.18}
\usepgfplotslibrary{groupplots, statistics, fillbetween, colormaps}
\usepgfplotslibrary{statistics}
\usetikzlibrary{patterns, decorations.pathreplacing, fadings, calc, arrows.meta}
\usepackage{tikzviolinplots}

\usetikzlibrary{external}
\tikzset{external/mode=only graphics}

\definecolor{palCoral}     {HTML}{E89B9B}  
\definecolor{palCoralDark} {HTML}{B66B6B}  
\definecolor{palSlate}     {HTML}{7BA5C4}  
\definecolor{palSlateDark} {HTML}{4A7898}
\definecolor{palSage}      {HTML}{A8C99A}  
\definecolor{palSageDark}  {HTML}{6F9460}

\definecolor{palRose}      {HTML}{E8B4C4}  
\definecolor{palRoseDark}  {HTML}{B47D8E}
\definecolor{palLavender}  {HTML}{B5A8D9}  
\definecolor{palLavDark}   {HTML}{7C6FA8}
\definecolor{palTeal}      {HTML}{6FA0A8}  
\definecolor{palTealDark}  {HTML}{426E76}
\definecolor{palOchre}     {HTML}{D9B870}  
\definecolor{palOchreDark} {HTML}{96772A}

\definecolor{palPeach}     {HTML}{F0B98D}  
\definecolor{palMint}      {HTML}{9DCFB6}  
\definecolor{palMauve}     {HTML}{C896B6}  
\definecolor{palButter}    {HTML}{E6CB85}  
\definecolor{palGray}      {HTML}{ADA7A0}  

\definecolor{palAmber}     {HTML}{F0C070}
\definecolor{palAmberDark} {HTML}{B07020}

\definecolor{hmGood}       {HTML}{8FBF9A}
\definecolor{hmMid}        {HTML}{F4F0E8}
\definecolor{hmBad}        {HTML}{D88A8A}

\pgfplotsset{
  /pgfplots/colormap={fidmap}{
    rgb255(0cm)=(143,191,154)
    rgb255(1cm)=(244,240,232)
    rgb255(2cm)=(216,138,138)
  },
  /pgfplots/colormap={rhomap}{
    rgb255(0cm)=(74,120,152)
    rgb255(1cm)=(244,240,232)
    rgb255(2cm)=(182,107,107)
  },
}

\pgfplotsset{
  fidlottery/.style={
    width=0.98\linewidth,
    height=5.2cm,
    font=\small,
    tick label style={font=\footnotesize},
    label style={font=\small},
    title style={font=\small\bfseries, yshift=-2pt},
    legend style={
      font=\footnotesize,
      draw=black!20,
      fill=white,
      fill opacity=0.92,
      text opacity=1,
      inner sep=2pt,
      rounded corners=1pt,
    },
    legend cell align=left,
    grid=major,
    major grid style={line width=0.25pt, draw=black!12},
    minor grid style={line width=0.2pt, draw=black!8},
    axis line style={draw=black!40, line width=0.4pt},
    tick style={draw=black!40, line width=0.4pt},
    every axis plot/.append style={line width=0.9pt},
  },
}

\pgfplotsset{
  thinmark/.style={mark size=1.1pt, mark options={line width=0.4pt}},
  midmark/.style={mark size=1.6pt, mark options={line width=0.5pt}},
}

\newcommand{\flsf}{\fontfamily{phv}\selectfont}
\newcommand{\fltitle}{\fontfamily{pag}\selectfont}

\usepackage{tcolorbox}
\tcbuselibrary{skins, breakable}
\newtcolorbox{advicebox}{%
  enhanced,
  colback=palSage!10, colframe=palSageDark,
  boxrule=0.5pt, arc=2pt,
  left=8pt, right=8pt, top=10pt, bottom=4pt,
  fonttitle=\bfseries, coltitle=white,
  colbacktitle=palSageDark,
  title={Advice for Practitioners},
  attach boxed title to top left={xshift=8pt, yshift=-6pt},
  boxed title style={colback=palSageDark, sharp corners, boxrule=0pt},
  before skip=8pt, after skip=8pt,
}

\newtcolorbox{tldrbox}{%
  enhanced, breakable,
  colback=palSlate!10, colframe=palSlateDark,
  boxrule=0.5pt, arc=2pt,
  left=8pt, right=8pt, top=10pt, bottom=4pt,
  fonttitle=\bfseries, coltitle=white,
  colbacktitle=palSlateDark,
  title={TL;DR},
  attach boxed title to top left={xshift=8pt, yshift=-6pt},
  boxed title style={colback=palSlateDark, sharp corners, boxrule=0pt},
  before skip=6pt, after skip=8pt,
}
\BeforeBeginEnvironment{tldrbox}{\tikzexternaldisable}
\AfterEndEnvironment{tldrbox}{\tikzexternalenable}

\title{The FID Lottery: Quantifying Hidden Randomness\\in Generative Model Evaluation}

\author{
  Nicolas Dufour \\
  Kyutai
 \And 
  Alexei A.\ Efros \\
  UC Berkeley
\And
  Patrick P\'erez \\
  Kyutai
}

\begin{document}

\maketitle

\begin{abstract}
The Fr\'echet Inception Distance (FID) is the de facto arbiter of image generation, yet most papers report just a single number from a single trained model using a single sampling seed. How reproducible is that number if we retrain the model, or merely resample from it?   In this paper, we treat FID as a random variable on a two-axis panel of training and generation seeds, and measure its variance directly on several hundred SiT networks trained on class-conditional ImageNet $256{\times}256$. We report surprising findings: (a)~Retraining the model using the same recipe with a different seed moves FID $3.2{\times}$ more (in Inception feature space) than redrawing samples from a fixed network. (b)~That gap is driven by three factors: random initialisation, data ordering, and the per-step Gaussian noise of the flow-matching loss. (c)~Increasing compute or model size barely tightens the spread, holding the FID coefficient of variation (CoV) inside a $1$--$2\%$ band. (d)~Per-cell classifier-free-guidance tuning halves the spread but reshuffles which seeds work best, and a lucky training seed reaches the same FID with up to $2{\times}$ less compute than an unlucky one.  Based on these findings, we recommend a new FID evaluation protocol:   evaluate under per-cell optimal guidance, treat any FID gap below the empirically measured $\approx\!1.3\%$ CoV as inconclusive, and report an error bar over several training seeds rather than a single FID number.
Project page: \href{https://kyutai.org/fid-lottery}{\url{https://kyutai.org/fid-lottery}}
\end{abstract}

\section{Introduction}

\begin{quote}
\raggedleft\itshape\small
``If the lottery is an intensification of chance, a periodic
infusion of chaos into the cosmos, would it not be desirable for
chance to intervene at all stages of the lottery and not merely
in the drawing?''
\par\smallskip
\upshape\textemdash\ Jorge Luis Borges, \emph{The Lottery in Babylon}
\end{quote}
\label{sec:intro}

\definecolor{cInk}    {HTML}{2A2620}   
\definecolor{cInkSoft}{HTML}{6B645B}   
\definecolor{cRail}   {HTML}{B8AFA2}   
\definecolor{cTrainBg}{HTML}{EEF3F7}   
\definecolor{cGenBg}  {HTML}{FBF5EC}   
\definecolor{cFrame}  {HTML}{D9CFC0}   
\definecolor{palSteel}    {HTML}{8C8FA8} 
\definecolor{palSteelDark}{HTML}{5B5E78}

\newcommand{\diepips}[2]{%
  \def\uo{0.255}\def\ur{0.083}
  \ifcase#2\relax
  \or \fill[#1] (0.5,0.5) circle(\ur);
  \or \fill[#1] (0.5-\uo,0.5+\uo) circle(\ur) (0.5+\uo,0.5-\uo) circle(\ur);
  \or \fill[#1] (0.5-\uo,0.5+\uo) circle(\ur) (0.5,0.5) circle(\ur) (0.5+\uo,0.5-\uo) circle(\ur);
  \or \fill[#1] (0.5-\uo,0.5+\uo) circle(\ur) (0.5+\uo,0.5+\uo) circle(\ur)
                 (0.5-\uo,0.5-\uo) circle(\ur) (0.5+\uo,0.5-\uo) circle(\ur);
  \or \fill[#1] (0.5-\uo,0.5+\uo) circle(\ur) (0.5+\uo,0.5+\uo) circle(\ur) (0.5,0.5) circle(\ur)
                 (0.5-\uo,0.5-\uo) circle(\ur) (0.5+\uo,0.5-\uo) circle(\ur);
  \or \fill[#1] (0.5-\uo,0.5+\uo) circle(\ur) (0.5+\uo,0.5+\uo) circle(\ur)
                 (0.5-\uo,0.5) circle(\ur) (0.5+\uo,0.5) circle(\ur)
                 (0.5-\uo,0.5-\uo) circle(\ur) (0.5+\uo,0.5-\uo) circle(\ur);
  \fi
}

\newcommand{\onedie}[8]{%
  \pgfmathsetmacro{\ds}{#4}%
  \pgfmathsetmacro{\dz}{0.46*#4}
  \pgfmathsetmacro{\ccx}{#1+0.5*#4+0.5*\dz}
  \pgfmathsetmacro{\ccy}{#2+0.5*#4+0.5*\dz}%
  \begin{scope}[rotate around={#5:(\ccx,\ccy)}]
    \filldraw[fill=#3!58!black, draw=#3!72!black, line width=0.6pt, line join=round]
      (#1+\ds,#2) -- ++(\dz,\dz) -- ++(0,\ds) -- ++(-\dz,-\dz) -- cycle;
    \begin{scope}[cm={\dz,\dz,0,\ds,(#1+\ds,#2)}]\diepips{#3!22!white}{#8}\end{scope}
    \filldraw[fill=#3!36, draw=#3!72!black, line width=0.6pt, line join=round]
      (#1,#2+\ds) -- ++(\dz,\dz) -- ++(\ds,0) -- ++(-\dz,-\dz) -- cycle;
    \begin{scope}[cm={\ds,0,\dz,\dz,(#1,#2+\ds)}]\diepips{#3!75!black}{#7}\end{scope}
    \filldraw[fill=white, draw=#3!72!black, line width=0.8pt, line join=round]
      (#1,#2) rectangle (#1+\ds,#2+\ds);
    \begin{scope}[cm={\ds,0,0,\ds,(#1,#2)}]\diepips{#3!75!black}{#6}\end{scope}
  \end{scope}
}

\newcommand{\dietwo}[3]{%
  \onedie{#1-0.73}{#2+0.03}{#3}{0.30}{16}{5}{1}{3}
  \onedie{#1-0.43}{#2-0.15}{#3}{0.28}{-13}{2}{4}{6}
}

\newcommand{\gpucard}[4]{%
  \pgfmathsetmacro{\gw}{#3}\pgfmathsetmacro{\gh}{0.56*#3}%
  \fill[#4!58!black, rounded corners=0.4pt]
    (#1-0.5*\gw,#2-0.5*\gh-0.11*\gw) rectangle (#1+0.5*\gw,#2-0.5*\gh+0.05*\gw);
  \fill[#4!58!black] (#1+0.30*\gw,#2+0.5*\gh) rectangle (#1+0.44*\gw,#2+0.5*\gh+0.07*\gw);
  \filldraw[fill=#4!22, draw=#4!72!black, line width=0.6pt, rounded corners=1.4pt]
    (#1-0.5*\gw,#2-0.5*\gh) rectangle (#1+0.5*\gw,#2+0.5*\gh);
  \foreach \fx in {-0.235,0.235}{%
    \filldraw[fill=#4!40, draw=#4!72!black, line width=0.4pt] (#1+\fx*\gw,#2) circle (0.185*\gw);
    \foreach \a in {30,150,270}{%
      \draw[#4!72!black, line width=0.32pt] (#1+\fx*\gw,#2) -- ++(\a:0.16*\gw);}
    \fill[#4!72!black] (#1+\fx*\gw,#2) circle (0.05*\gw);}
}

\newcommand{\photoicon}[4]{%
  \begin{scope}
    \clip[rounded corners=1.2pt] (#1-#3,#2-#3) rectangle (#1+#3,#2+#3);
    \fill[#4!28] (#1-#3,#2-#3) rectangle (#1+#3,#2+#3);
    \fill[palAmber] ($(#1+0.42*#3,#2+0.44*#3)$) circle (0.26*#3);
    \fill[#4!58!black] (#1-#3,#2-#3) -- ($(#1-0.18*#3,#2+0.30*#3)$) -- ($(#1+0.34*#3,#2-#3)$) -- cycle;
    \fill[#4!72!black] ($(#1-0.06*#3,#2-#3)$) -- ($(#1+0.52*#3,#2+0.40*#3)$) -- (#1+#3,#2-#3) -- cycle;
  \end{scope}
  \draw[rounded corners=1.2pt, black!35, line width=0.4pt]
    (#1-#3,#2-#3) rectangle (#1+#3,#2+#3);
}

\newcommand{\noisetile}[4]{%
  \pgfmathsetseed{#4}%
  \pgfmathsetmacro{\stp}{2*#3/16}%
  \begin{scope}
    \clip[rounded corners=1.5pt] (#1-#3,#2-#3) rectangle (#1+#3,#2+#3);
    \foreach \i in {0,...,15}{%
      \foreach \j in {0,...,15}{%
        \pgfmathtruncatemacro{\g}{100*rnd}%
        \fill[black!\g] (#1-#3+\i*\stp,#2-#3+\j*\stp) rectangle ++(\stp,\stp);%
      }%
    }%
  \end{scope}
  \draw[rounded corners=1.5pt, black!45, line width=0.5pt]
    (#1-#3,#2-#3) rectangle (#1+#3,#2+#3);
}

\newcommand{\sampletile}[4]{%
  \pgfmathsetmacro{\sz}{2*#3}%
  \begin{scope}
    \clip[rounded corners=1.5pt] (#1-#3,#2-#3) rectangle (#1+#3,#2+#3);
    \node[inner sep=0] at (#1,#2)
      {\includegraphics[width=\sz cm,height=\sz cm]{#4}};
  \end{scope}
  \draw[rounded corners=1.5pt, black!45, line width=0.5pt]
    (#1-#3,#2-#3) rectangle (#1+#3,#2+#3);
}

\newcommand{\neticon}[4]{%
  \def\lA{#1-0.42}\def\lB{#1}\def\lC{#1+0.42}%
  \ifnum#3=1
    \def\nodefill{#4!30}\def\nodedraw{#4!70!black}%
    \foreach \ay in {0.32,0,-0.32}{\foreach \by in {0.42,0.14,-0.14,-0.42}{%
      \draw[#4!55!black, line width=0.45pt] (\lA,#2+\ay) -- (\lB,#2+\by);}}
    \foreach \ay in {0.42,0.14,-0.14,-0.42}{\foreach \by in {0.32,0,-0.32}{%
      \draw[#4!55!black, line width=0.45pt] (\lB,#2+\ay) -- (\lC,#2+\by);}}
  \else
    \def\nodefill{white}\def\nodedraw{black!45}%
    \pgfmathsetseed{91}%
    \foreach \ay in {0.32,0,-0.32}{\foreach \by in {0.42,0.14,-0.14,-0.42}{%
      \pgfmathsetmacro{\op}{20+60*rnd}%
      \draw[black!\op, line width=0.4pt] (\lA,#2+\ay) -- (\lB,#2+\by);}}
    \foreach \ay in {0.42,0.14,-0.14,-0.42}{\foreach \by in {0.32,0,-0.32}{%
      \pgfmathsetmacro{\op}{20+60*rnd}%
      \draw[black!\op, line width=0.4pt] (\lB,#2+\ay) -- (\lC,#2+\by);}}
  \fi
  \foreach \y in {0.32,0,-0.32}{%
    \filldraw[fill=\nodefill, draw=\nodedraw, line width=0.5pt] (\lA,#2+\y) circle (0.062);}
  \foreach \y in {0.42,0.14,-0.14,-0.42}{%
    \filldraw[fill=\nodefill, draw=\nodedraw, line width=0.5pt] (\lB,#2+\y) circle (0.062);}
  \foreach \y in {0.32,0,-0.32}{%
    \filldraw[fill=\nodefill, draw=\nodedraw, line width=0.5pt] (\lC,#2+\y) circle (0.062);}
}

\newcommand{\stagebox}[5]{%
  \fill[black!11, rounded corners=3.5pt]
    (#1-#3+0.06,#2-#4-0.07) rectangle (#1+#3+0.06,#2+#4-0.07);
  \filldraw[fill=white, draw=cFrame, line width=0.8pt, rounded corners=3.5pt]
    (#1-#3,#2-#4) rectangle (#1+#3,#2+#4);
  \fill[#5!85!white, rounded corners=2pt]
    (#1-#3+0.08,#2+#4-0.13) rectangle (#1+#3-0.08,#2+#4-0.05);
}

\newcommand{\srcname}[4]{%
  \node[font=\flsf\bfseries\scriptsize, #3!72!black, anchor=base] at (#1,#2) {#4};}
\newcommand{\srcnum}[4]{%
  \filldraw[fill=#3!85!black, draw=white, line width=0.7pt] (#1,#2) circle (0.12);
  \node[font=\flsf\bfseries\scriptsize, white] at (#1,#2) {#4};}
\newcommand{\sidephrase}[3]{%
  \node[anchor=west, align=left, font=\flsf\scriptsize, cInkSoft] at (#1,#2) {#3};}

\newcommand{\flowarrow}[3]{
  \draw[-{Stealth[length=6.5pt,width=6.5pt]}, cRail, line width=1.8pt]
    (#1,#3) -- (#2,#3);}

\newcommand{\capnum}[2]{%
  \begin{tikzpicture}[baseline=-0.5ex, external/export=false]
    \node[circle, fill=#1!85!black, text=white, inner sep=0pt, minimum size=1.5ex,
          font=\fontsize{5}{5}\selectfont\bfseries] {#2};
  \end{tikzpicture}}

\begin{figure}[t]
  \centering
  \tikzsetnextfilename{output-pipeline}%
  \resizebox{\linewidth}{!}{%
  \begin{tikzpicture}
    \useasboundingbox (0,1.95) rectangle (17.95,6.55);

    \filldraw[fill=cTrainBg, draw=palSlate!45, line width=0.8pt, rounded corners=7pt]
      (0.20,2.10) rectangle (11.725,6.25);
    \filldraw[fill=cGenBg, draw=palCoral!55, line width=0.8pt, rounded corners=7pt]
      (13.70,2.10) rectangle (17.75,6.25);

    \node[anchor=west, font=\flsf\bfseries\footnotesize, white,
          fill=palSlateDark, rounded corners=2.5pt, inner xsep=6pt, inner ysep=3pt]
          at (0.40,5.96) {TRAINING LOTTERY};

    \flowarrow{2.72}{3.13}{3.25}
    \flowarrow{5.47}{5.88}{3.25}
    \flowarrow{8.22}{8.63}{3.25}

    \srcname{1.55}{5.35}{palMint}{random init}
    \sidephrase{0.40}{4.66}{starting\\weights}
    \dietwo{2.70}{4.50}{palMint}
    \srcnum{1.92}{4.98}{palMint}{1}
    \stagebox{1.55}{3.25}{1.15}{0.92}{palMint}
    \begin{scope}[shift={(1.55,3.25)}, scale=1.45, shift={(-1.55,-3.25)}]
      \neticon{1.55}{3.25}{0}{palMint}
    \end{scope}

    \srcname{4.30}{5.35}{palMauve}{data order}
    \sidephrase{3.15}{4.66}{reshuffled\\each epoch}
    \dietwo{5.45}{4.50}{palMauve}
    \srcnum{4.67}{4.98}{palMauve}{2}
    \stagebox{4.30}{3.25}{1.15}{0.92}{palMauve}
    \photoicon{3.98}{3.42}{0.32}{palSlate}
    \photoicon{4.30}{3.20}{0.32}{palSage}
    \photoicon{4.62}{3.42}{0.32}{palCoral}
    \draw[-{Stealth[length=2.8pt]}, cInkSoft, line width=0.7pt]
      (4.14,2.655) to[bend left=45] (4.46,2.655);
    \draw[-{Stealth[length=2.8pt]}, cInkSoft, line width=0.7pt]
      (4.46,2.555) to[bend left=45] (4.14,2.555);

    \srcname{7.05}{5.35}{palPeach}{training noise}
    \sidephrase{5.90}{4.66}{fresh noise\\every step}
    \dietwo{8.20}{4.50}{palPeach}
    \srcnum{7.42}{4.98}{palPeach}{3}
    \stagebox{7.05}{3.25}{1.15}{0.92}{palPeach}
    \photoicon{6.55}{3.34}{0.40}{palSage}
    \node[font=\scriptsize, palPeach!72!black] at (7.05,3.34) {$\boldsymbol{+}$};
    \noisetile{7.55}{3.34}{0.40}{77}
    \node[font=\flsf\scriptsize, cInkSoft] at (7.05,2.55) {$\circlearrowright$ every step};

    \srcname{9.80}{5.35}{palSteel}{hardware noise}
    \sidephrase{8.65}{4.66}{bitwise drift\\across GPUs}
    \dietwo{10.95}{4.50}{palSteel}
    \srcnum{10.17}{4.98}{palSteel}{4}
    \stagebox{9.80}{3.25}{1.15}{0.92}{palSteel}
    \begin{scope}[palSteelDark, line width=0.55pt, line cap=round]
      \draw (9.33,3.56) -- (10.27,3.56);
      \draw (9.33,2.94) -- (10.27,2.94);
      \draw (9.33,3.56) -- (9.33,2.94);
      \draw (10.27,3.56) -- (10.27,2.94);
      \draw (9.33,3.56) -- (10.27,2.94);
      \draw (10.27,3.56) -- (9.33,2.94);
    \end{scope}
    \gpucard{9.33}{3.56}{0.66}{palSteel}
    \gpucard{10.27}{3.56}{0.66}{palSteel}
    \gpucard{9.33}{2.94}{0.66}{palSteel}
    \gpucard{10.27}{2.94}{0.66}{palSteel}

    \node[font=\flsf\bfseries\footnotesize, palSageDark, align=center] at (12.7125,4.60) {trained\\network};
    \fill[black!11, rounded corners=3.5pt] (11.9725,2.26) rectangle (13.5725,4.10);
    \filldraw[fill=palSage!16, draw=palSageDark!70, line width=1pt, rounded corners=3.5pt]
      (11.9125,2.33) rectangle (13.5125,4.17);
    \begin{scope}[shift={(12.7125,3.25)}, scale=1.45, shift={(-12.7125,-3.25)}]
      \neticon{12.7125}{3.25}{1}{palSage}
    \end{scope}

    \flowarrow{10.97}{11.90}{3.25}
    \flowarrow{13.53}{14.46}{3.25}

    \node[anchor=west, font=\flsf\bfseries\footnotesize, white,
          fill=palCoralDark, rounded corners=2.5pt, inner xsep=6pt, inner ysep=3pt]
          at (13.90,5.96) {GENERATION LOTTERY};

    \srcname{15.725}{5.35}{palCoral}{initial noise}
    \sidephrase{14.475}{4.66}{a fresh draw\\for each image}
    \dietwo{16.975}{4.50}{palCoral}
    \srcnum{16.195}{4.98}{palCoral}{5}
    \stagebox{15.725}{3.25}{1.25}{0.92}{palCoral}
    \noisetile{15.20}{3.34}{0.40}{31}
    \draw[-{Stealth[length=5pt]}, cInkSoft, line width=1pt] (15.63,3.34) -- (15.82,3.34);
    \sampletile{16.25}{3.34}{0.40}{figures/fig_data/fid_hero/flamingo0}
    \node[font=\flsf\scriptsize, cInkSoft] at (15.725,2.55) {$\circlearrowright$ for every sample};
  \end{tikzpicture}}
  \vspace{-.2in}
  \caption[All sources of randomness behind a generative model]{%
  \textbf{All sources of randomness behind a generative model.}
  Training and then sampling a generative model is a chain of pseudo-random
  draws. 
  They fall into two lotteries. The
  \emph{training lottery} (left) is drawn from four sources: the random weight initialisation
  \capnum{palMint}{1}, the order in which examples are visited
  \capnum{palMauve}{2}, the fresh Gaussian noise the flow-matching loss injects
  at every gradient step \capnum{palPeach}{3}, and the bitwise non-determinism
  of multi-GPU execution \capnum{palSteel}{4}.  This draws a single trained network out of the many networks the same recipe could have produced.   The \emph{generation lottery} (right)  then draws a fresh initial noise $x_T\!\sim\!\mathcal{N}(0,I)$ for every
  sampled image \capnum{palCoral}{5}. Common practice accounts for only
  \capnum{palCoral}{5}, reporting error bars over sampling seeds on one fixed
  network. In this work we study all five. }
  \label{fig:lottery-machine}
\end{figure}
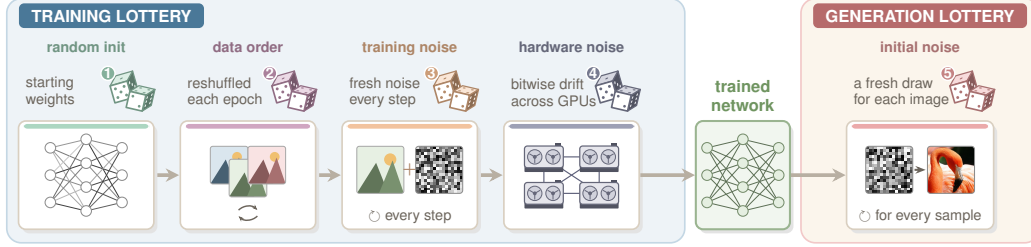

Open any recent image generation paper and the central claim usually
rests on a single number, the Fr\'echet Inception Distance (FID).
FID is the closest criterium image generation has to an arbiter: a half-unit shift
reorders the leaderboard. A decade of recipes have been justified
by single-FID-unit gains, and budgets in the low millions of
GPU-hours hinge on which architecture lands a few decimals lower.
But behind every reported FID number sits a chain of pseudo-random draws
(parameter initialisation, minibatch order, per-step Gaussian noise
injected by the training loss, hardware stochasticity and \nicolas{the initial noise drawn at
sampling time}), any of which could have produced a potentially
different score had the seed been different. Conventional wisdom considers this variance in FID to be negligible, especially for well-trained models. In this paper we show that the FID reproducibility gap is real and a serious concern.  

\nicolas{Each time one trains a generative model and reports its
FID, two lotteries are played (\autoref{fig:lottery-machine}). The \emph{training lottery} runs
once during training: it draws an initialisation, a data
ordering, and the per-step noise that the loss injects at every
gradient step, and what comes out is one trained network among
many that could have been produced by different seeds. The
\emph{generation lottery} runs on the trained network: one draws
an initial noise $x_T\!\sim\!\mathcal{N}(0,I)$ to seed the
sampler, generates a sample set, and scores it. Practitioners
have learned to mitigate the second lottery by redrawing the
initial noise across several seeds and reporting an error bar
or an averaged FID
score~\citep{chong2020effectively,parmar2022aliased}. However, no
amount of resampling on a single trained network says anything
about where a re-trained network run would have landed. The
training variability stays hidden behind the one ticket we actually
drew. Diffusion makes the problem worse: the
flow-matching~\citep{lipman2023flowmatching} or
score-matching~\citep{song2021scoresde} loss redraws a fresh
Gaussian $\varepsilon\!\sim\!\mathcal{N}(0,I)$ at every gradient
step, so the training noise never settles. It is a permanent
random injection that an independent training run would resolve
differently, not transient noise that longer training averages
out.} Scale
offers no automatic remedy. Neural scaling
laws~\citep{kaplan2020scaling,hoffmann2022chinchilla} characterise
how the mean loss falls with parameters and tokens, leaving the
seed-induced \emph{spread around} that mean
unspecified. \citet{zhang2024emergence} recently showed that
independently-trained diffusion networks converge to nearly the same
noise-to-image mapping. But does this also hold for the FID metric computed over a set of generated images?


\FloatBarrier
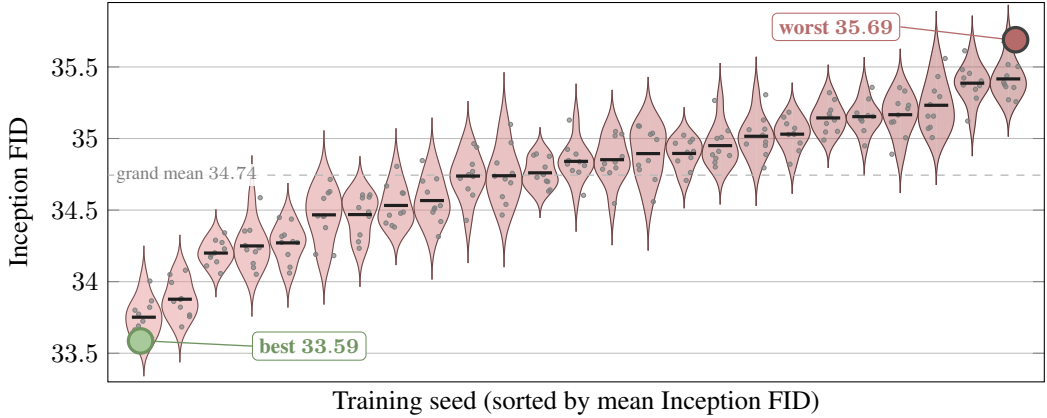
\begin{figure}[t]
  \centering
  \pgfplotsset{width=\linewidth, height=6.6cm}
  \begin{tikzpicture}
    \violinsetoptions[scaled]{
      xlabel={Training seed (sorted by mean Inception FID)},
      ylabel={Inception FID},
      xmin=0.0,
      xmax=26.0,
      ymin=33.30,
      ymax=35.95,
      xtick={1,5,10,15,20,25},
      ytick={33.5,34.0,34.5,35.0,35.5},
      ymajorgrids,
    }

    \violinplotwholefile[%
      col sep=tab,
      primary color=palCoralDark,
      secondary color=palCoral,
      indexes={s01,s02,s03,s04,s05,s06,s07,s08,s09,s10,s11,s12,s13,s14,s15,s16,s17,s18,s19,s20,s21,s22,s23,s24,s25},
      labels={,,,,,,,,,,,,,,,,,,,,,,,,},
      spacing=1.0,
    ]{figures/fig_data/fig1a_violins_wholefile.tsv}

    \begin{axis}[
      xmin=0.0, xmax=26.0,
      ymin=33.30, ymax=35.95,
      axis line style={draw=none},
      tick style={draw=none},
      xticklabels={,,}, yticklabels={,,},
      xmajorticks=false, ymajorticks=false,
      axis on top,
      clip=false,
    ]
      \draw[dashed, color=black!28, line width=0.5pt]
        (axis cs:0.0,34.744) -- (axis cs:26.0,34.744);
      \node[anchor=west, font=\scriptsize, color=black!50,
            fill=white, fill opacity=0.85, text opacity=1, inner sep=1pt]
            at (axis cs:0.15,34.744) {grand mean $34.74$};

      \addplot+[
        only marks, mark=*,
        mark size=0.8pt,
        mark options={fill=black!50, draw=black!50, line width=0pt, fill opacity=0.65},
      ] table[x=xj, y=fid_inc] {figures/fig_data/fig1a_strip.tsv};

      \addplot+[
        only marks, mark=-, mark size=4.5pt,
        mark options={draw=black!88, line width=1.2pt},
      ] table[x=x, y=mean_inc] {figures/fig_data/fig1a_means.tsv};

      \addplot+[
        only marks, mark=*,
        mark size=4.6pt,
        mark options={fill=palSage, draw=palSageDark, line width=1.2pt},
      ] coordinates {(0.908, 33.586)};
      \node[
        anchor=west,
        font=\footnotesize\bfseries, color=palSageDark,
        fill=white, fill opacity=0.92, text opacity=1,
        inner sep=2.5pt, draw=palSageDark, line width=0.4pt,
        rounded corners=2pt,
      ] (bestlbl) at (axis cs:4.0, 33.55) {best $\mathbf{33.59}$};
      \draw[->, >=stealth, palSageDark, line width=0.5pt]
        (bestlbl.west) -- (axis cs:1.05, 33.586);

      \addplot+[
        only marks, mark=*,
        mark size=4.6pt,
        mark options={fill=palCoralDark, draw=black!75, line width=1.2pt},
      ] coordinates {(25.203, 35.690)};
      \node[
        anchor=east,
        font=\footnotesize\bfseries, color=palCoralDark,
        fill=white, fill opacity=0.92, text opacity=1,
        inner sep=2.5pt, draw=palCoralDark, line width=0.4pt,
        rounded corners=2pt,
      ] (worstlbl) at (axis cs:22.0, 35.78) {worst $\mathbf{35.69}$};
      \draw[->, >=stealth, palCoralDark, line width=0.5pt]
        (worstlbl.east) -- (axis cs:25.05, 35.690);
    \end{axis}
  \end{tikzpicture}
    \vspace{-.1in}
  \caption{\textbf{The FID lottery in SiT-B/2 at 400k steps.}
  \nicolas{Each violin is one of $25$ independently trained SiT-B/2
  models, sorted by per-seed mean Inception FID. Small dots are the
  $250$ individual sampling-seed evaluations and short black ticks are
  per-seed means. The two highlighted markers pick out the single best
  ($\mathbf{33.59}$) and worst ($\mathbf{35.69}$) FID across the panel,
  a $\mathbf{2.10}$-point gap produced purely by changing seeds. Each
  violin is short ($\sigma_{\text{within}}\!\approx\!0.14$, evaluation
  lottery). The per-seed means stagger over a $3{\times}$ wider range
  ($\sigma_{\text{between}}\!=\!0.44$, training lottery).}}
  \label{fig:overview}
\end{figure}
\nicolas{The lotteries defines two axis of FID variance: a
\emph{training axis} of $N$ independent training runs and a
\emph{generation axis} of $K$ sampling seeds per run. To measure
both, we score the resulting $N{\times}K$ panel of FID
evaluations.} \autoref{fig:overview}
renders the panel for a converged SiT-B/2: every violin is one
trained network, every dot is one FID evaluation, and the spread
along the training axis already overshoots the spread along the
generation axis at a glance. \nicolas{On this panel} we first decompose
the training lottery into its independent random sources, then
sweep four practitioner-controlled axes (classifier-free
guidance~\citep{ho2022cfg}, compute, model size, and learning
rate, transferred across widths via $\mu$P~\citep{yang2022mup}, a
parameterisation that makes the optimal learning rate
width-invariant) to test whether any of them tightens it.
\nicolas{Across several hundred SiT networks from S through XL on
ImageNet $256{\times}256$, the training axis dominates the
evaluation axis at every scale we probe and none of the four knobs
closes the gap. A single-seed Inception FID therefore sits on a
noise floor that "SoTA" improvements regularly fall below.}

\paragraph{Paper's Contributions:}
\begin{itemize}
\item \textbf{A measurement of the FID lottery on modern
diffusion.} Treating FID as a random variable over training and
sampling seeds, we measure its spread across several hundred SiT
networks and find that retraining moves FID $3.2{\times}$ more than
resampling does. The relative noise floor stays remarkably stable
across model sizes and training budgets, which we turn into a
concrete calibration target for single-seed FID claims
(\autoref{sec:results-sampling}, \autoref{sec:results-scaling}).
\item \textbf{An examination of different sources of randomness.}
We separate the contributions of initialisation, data order, hardware noise, and
the per-step Gaussian noise of the flow-matching loss. (\autoref{sec:results-decomposition}).
\item \textbf{GS-FID, a per-cell golden-section classifier-free guidance (CFG) protocol.}
Tuning guidance individually for every (training, sampling) seed
pair tightens the noise floor, but at substantial evaluation cost
and reshuffling which seeds rank best
(\autoref{sec:results-guidance}).
\item \nicolas{\textbf{The luck of the draw.} A lucky training seed
reaches the same FID with up to $2{\times}$ less compute than an
unlucky one (\autoref{sec:results-speedup},
\autoref{fig:lucky-speedup}).}
\end{itemize}

\section{Related Work}
\label{sec:related}

\paragraph{FID and its limitations.}
The Fr\'echet Inception Distance \citep{heusel2017gans}, which compares
Gaussian moments of Inception-V3 \citep{szegedy2016inceptionv3} features, has
displaced the Inception Score \citep{salimans2016improved} despite well-known
fragilities. \citet{barratt2018note} flag systematic Inception-Score bias.
\citet{chong2020effectively} prove finite-sample FID is a model-dependent
biased estimator whose ranking can flip from sampling noise alone, and
propose FID$_\infty$. \citet{parmar2022aliased} show that aliased resizing
and JPEG compression shift FID by amounts comparable to claimed
state-of-the-art gains (Clean-FID). \citet{kynkaanniemi2023role} show that
class-balance matching against ImageNet \citep{deng2009imagenet} histograms
lowers FID without changing perceived quality. \citet{stein2023exposing}
report that FID systematically penalises diffusion models relative to human
raters. Replacements include KID \citep{binkowski2018demystifying},
precision/recall
\citep{sajjadi2018assessing,kynkaanniemi2019improved,naeem2020reliable},
self-supervised features \citep{morozov2021self,oquab2024dinov2}, and CMMD
\citep{jayasumana2024rethinking}, with \citet{stein2023exposing} recommending
DINOv2 and \citet{wu2025pragmatic} finding small FID gaps uncorrelated with
downstream utility. FID nonetheless dominates
\citep{borji2019pros,borji2022pros,lucic2018gans} on a decade of comparable
numbers and the rank-consistency argument of \citet{chong2020effectively}.

\myparagraph{Reproducibility and statistical methodology.}
Empirical machine learning has long been under scrutiny, prefigured by the
satire of \citet{laloudouana2003data}. \citet{henderson2018deep} show that
nominally identical RL agents diverge across seeds. NLP
\citep{reimers2017reporting,mosbach2021stability,dodge2020fine,
crane2018questionable,madhyastha2019stability} and language modelling
\citep{melis2018state} reach the same verdict. Systemic critiques
\citep{sculley2018winners,gundersen2018state,pineau2021improving,raff2019step},
echoed by the {Pascal VOC} retrospective \citep{everingham2014pascal} on a
decade of competition methodology, call for stronger reporting. \citet{dodge2019showyourwork} argue for
disclosed hyperparameter-search budgets.
\citet{bouthillier2019unreproducible,bouthillier2021accounting} decompose
total variance into algorithmic, data and implementation sources.
\citet{pham2020problems,summers2021nondeterminism,picard2021torch} document
fluctuations large enough to invert rankings even in nominally deterministic
pipelines. Finally, \citet{recht2019imagenet,damour2022underspecification} show
off-distribution disagreement among equivalent models. The classical
comparison toolkit is under-applied in generative modelling: paired tests
and $5\times2$ cross-validation \citep{dietterich1998tests}, Friedman/Nemenyi
\citep{demsar2006comparisons}, Bayesian variants \citep{benavoli2017bayesian},
resampling \citep{efron1979bootstrap,diciccio1996bootstrap}, Welch's $t$-test
\citep{welch1947ttest}, multiple-comparison corrections
\citep{dunn1961multiple}, and rank statistics
\citep{spearman1904rank,kendall1939w}. NLP imports these
\citep{dror2018hitchhiker,dodge2019showyourwork}. Closest to ours,
\citet{banerjee2024measuring} propose a Kolmogorov--Smirnov test for
seed-induced model variability, and \citet{bench2025quantifying} use Monte
Carlo dropout in the feature extractor to obtain a distribution over FID-like
scores.

\begin{figure*}[!tb]
  \centering
  \begin{tikzpicture}
    \begin{groupplot}[
        group style={
          group size=2 by 1,
          horizontal sep=1.7cm,
        },
        fidlottery,
        width=0.46\textwidth,
        height=5.6cm,
      ]

      \nextgroupplot[
        title={(a) Per-seed mean FID by condition},
        xlabel={Inception FID (per-seed mean)},
        xmin=33.7, xmax=35.55,
        ymin=-0.7, ymax=4.7,
        ytick={0,1,2,3,4},
        yticklabels={vary all, vary noise, vary init, vary data, same seed},
        y dir=reverse,
        yticklabel style={font=\footnotesize},
        boxplot/draw direction=x,
      ]
        \addplot+[
          boxplot prepared={
            draw position=0,
            lower whisker=33.7517, lower quartile=34.4687,
            median=34.8405, upper quartile=35.0304, upper whisker=35.4163,
            box extend=0.42, whisker extend=0.18,
          },
          fill=palGray!75, draw=black!55, line width=0.4pt,
        ] coordinates {};
        \addplot+[
          boxplot prepared={
            draw position=1,
            lower whisker=33.9927, lower quartile=34.3389,
            median=34.5087, upper quartile=34.7803, upper whisker=35.3210,
            box extend=0.42, whisker extend=0.18,
          },
          fill=palPeach!85, draw=palCoralDark, line width=0.4pt,
        ] coordinates {};
        \addplot+[
          boxplot prepared={
            draw position=2,
            lower whisker=34.2031, lower quartile=34.6320,
            median=34.7945, upper quartile=35.0078, upper whisker=35.2926,
            box extend=0.42, whisker extend=0.18,
          },
          fill=palMint!85, draw=palSageDark, line width=0.4pt,
        ] coordinates {};
        \addplot+[
          boxplot prepared={
            draw position=3,
            lower whisker=34.2181, lower quartile=34.2664,
            median=34.4228, upper quartile=34.5643, upper whisker=35.0351,
            box extend=0.42, whisker extend=0.18,
          },
          fill=palMauve!85, draw=palLavDark, line width=0.4pt,
        ] coordinates {};
        \addplot+[
          boxplot prepared={
            draw position=4,
            lower whisker=34.5485, lower quartile=34.6187,
            median=34.6354, upper quartile=34.6578, upper whisker=34.8086,
            box extend=0.42, whisker extend=0.18,
          },
          fill=palSlate!85, draw=palSlateDark, line width=0.4pt,
        ] coordinates {};

        \addplot+[
          only marks, mark=*, mark size=1.4pt,
          mark options={fill=black!75, draw=black!85, line width=0.2pt, fill opacity=0.7},
        ] table[col sep=tab, x=mean_inc, y=yj] {figures/fig_data/fig2a_strip_vary_all.tsv};
        \addplot+[
          only marks, mark=*, mark size=1.4pt,
          mark options={fill=palCoralDark, draw=palCoralDark, line width=0.2pt, fill opacity=0.7},
        ] table[col sep=tab, x=mean_inc, y=yj] {figures/fig_data/fig2a_strip_vary_noise.tsv};
        \addplot+[
          only marks, mark=*, mark size=1.4pt,
          mark options={fill=palSageDark, draw=palSageDark, line width=0.2pt, fill opacity=0.7},
        ] table[col sep=tab, x=mean_inc, y=yj] {figures/fig_data/fig2a_strip_vary_init.tsv};
        \addplot+[
          only marks, mark=*, mark size=1.4pt,
          mark options={fill=palLavDark, draw=palLavDark, line width=0.2pt, fill opacity=0.7},
          y filter/.expression={\pgfmathresult-1},
        ] table[col sep=tab, x=mean_inc, y=yj] {figures/fig_data/fig2a_strip_vary_data.tsv};
        \addplot+[
          only marks, mark=*, mark size=1.4pt,
          mark options={fill=palSlateDark, draw=palSlateDark, line width=0.2pt, fill opacity=0.7},
        ] table[col sep=tab, x=mean_inc, y=yj] {figures/fig_data/fig2a_strip_same_seed.tsv};

      \nextgroupplot[
        title={(b) Between-seed vs. within-seed $\boldsymbol{\sigma}$},
        ylabel={$\sigma$ (Inception FID)},
        ybar=1pt,
        bar width=6pt,
        ymin=0, ymax=0.62,
        symbolic x coords={vary all, vary noise, vary init, vary data, same seed},
        xtick=data,
        x tick label style={font=\footnotesize, rotate=20, anchor=north east},
        enlarge x limits=0.12,
        legend style={
          at={(0.99,0.97)}, anchor=north east,
          draw=black!15, fill=white, fill opacity=0.92, text opacity=1,
          font=\scriptsize, inner sep=2pt, rounded corners=1pt,
        },
        legend image post style={scale=0.7},
      ]
        \addplot+[
          fill=palCoral, draw=palCoralDark, line width=0.3pt,
        ] coordinates {
          (vary all,   0.4379)
          (vary noise, 0.3339)
          (vary init,  0.2936)
          (vary data,  0.2180)
          (same seed,  0.0473)
        };
        \addlegendentry{between-seed $\sigma$}

        \addplot+[
          fill=palSage, draw=palSageDark, line width=0.3pt,
        ] coordinates {
          (vary all,   0.1373)
          (vary noise, 0.1441)
          (vary init,  0.1503)
          (vary data,  0.1501)
          (same seed,  0.1193)
        };
        \addlegendentry{within-seed $\sigma$ (sampling)}

    \end{groupplot}
  \end{tikzpicture}
    \vspace{-.1in}
  \caption{\textbf{Variance decomposition of the training-seed lottery
  (SiT-B/2, 400k, no CFG).}
  \nicolas{(a)~Per-seed mean Inception FID under the three single-source
  conditions plus the fully-stochastic baseline (\emph{vary all}) and
  the \emph{same seed} control. Each dot is one training seed, boxes show
  the 25/50/75 percentiles. (b)~Between-seed $\sigma$ (coral) versus
  within-seed sampling $\sigma$ (sage) per condition. The four
  random-source conditions are ordered monotonically by between-seed $\sigma$,
  with sampling $\sigma$ flat across all four. The rightmost
  \emph{same seed} column is a control that fixes init, data order,
  \emph{and} training noise to identical values, leaving only the bitwise
  non-determinism of multi-GPU (DDP) execution: its between-seed $\sigma$
  collapses to $0.047$, below the within-seed sampling floor, even though
  the trained weights differ by ${\approx}5$--$6\%$ of their norm.
  Numerical non-determinism is thus not a meaningful source of FID
  variance.}}
  \label{fig:disentangle}
\end{figure*}
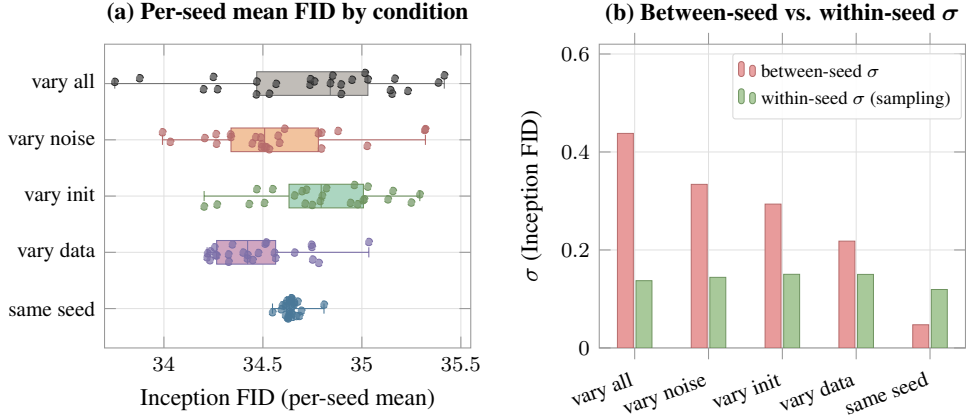

\myparagraph{Why runs differ.}
The lottery-ticket hypothesis \citep{frankle2019lottery,frankle2020linear}
and the loss-landscape view of \citet{fort2019deepensembles} argue that the stochastic gradient descent 
visits a discretely diverse set of basins, amplified by initialisation
\citep{glorot2010xavier,he2015kaiming}, batch order, and adaptive optimisers
\citep{kingma2015adam}. Architectures
\citep{he2016resnet,vaswani2017transformer,dosovitskiy2021vit} change basin
geometry but not multiplicity. \citet{nagarajan2019uniform} formalise why
these gaps are intrinsic, while
\citet{wenzel2020hyperensembles,jordan2023calibrated} exploit them for
uncertainty quantification. \citet{zhang2024emergence} show diffusion models
are unusually well-behaved at the function level. We add that near-identical
noise-to-image maps still yield percent-level FID fluctuations.

\myparagraph{Generative reproducibility, scaling, and seeds.}
Variance studies in generative modelling remain rare. \citet{lucic2018gans}
find equalised hyperparameter budgets erase most reported GAN gains.
\citet{degeorge2025imagenet} attack the data side with a redistributable
ImageNet-only text-to-image protocol. Scaling laws
\citep{hestness2017predictable,kaplan2020scaling,hoffmann2022chinchilla,
henighan2020scaling}, including for diffusion transformers
\citep{liang2024scalingdit,peebles2023dit}, encourage treating seed variance
as vanishing residual. We instead measure the variance that remains within a
scale. Diffusion models \citep{sohldickstein2015thermo,song2019ncsn,
ho2020ddpm,song2021ddim,song2021scoresde,nichol2021iddpm,dhariwal2021adm,
kingma2021vdm,karras2022edm,ho2022cfg,rombach2022ldm,lipman2023flowmatching,
liu2023rectified,karras2024edm2,ma2024sit,esser2024sd3,li2024mar} on ImageNet
\citep{deng2009imagenet}, FFHQ \citep{karras2019stylegan} and LAION-5B
\citep{schuhmann2022laion5b} dominate recent state of the art yet escape
systematic variance characterisation.
\citet{kadkhodaie2024generalization,zhang2024emergence} characterise
function-level reproducibility but not metric-level noise.
\citet{xu2024goodseed} demonstrate an inference-time ``seed lottery''
complementing our training-time variance. The closest precursors, \citet{chong2020effectively} on finite-sample FID bias and
\citet{bench2025quantifying} on feature-extractor uncertainty, vary
neither training seeds, architectures, nor checkpoints.

\section{Experimental Setup}

\label{sec:setup}

This section describes the experimental setting used for all experiments in 
\autoref{sec:results}: the $N\!\times\!K$ panel of FID evaluations,
the random number generators that populate it, and the 3 nested
statistics that sum it up.


\definecolor{fidGold} {HTML}{9C7C2F}   
\definecolor{fidInk}  {HTML}{2A2620}   
\definecolor{fidLo}   {rgb}{0.56,0.80,0.78}  
\definecolor{fidMid}  {rgb}{0.97,0.93,0.75}  
\definecolor{fidHi}   {rgb}{0.95,0.69,0.72}  

\newcommand{\fidtile}[3]{%
  \fill[black!22, rounded corners=2.5pt] (#1-0.80,#2-0.90) rectangle (#1+0.90,#2+0.80);
  \begin{scope}
    \clip[rounded corners=2.5pt] (#1-0.85,#2-0.85) rectangle (#1+0.85,#2+0.85);
    \node[inner sep=0] at (#1,#2)
      {\includegraphics[width=1.72cm,height=1.72cm]{figures/fig_data/fid_hero/#3.jpg}};
  \end{scope}
  \draw[rounded corners=2.5pt, fidGold, line width=0.8pt]
    (#1-0.85,#2-0.85) rectangle (#1+0.85,#2+0.85);
}

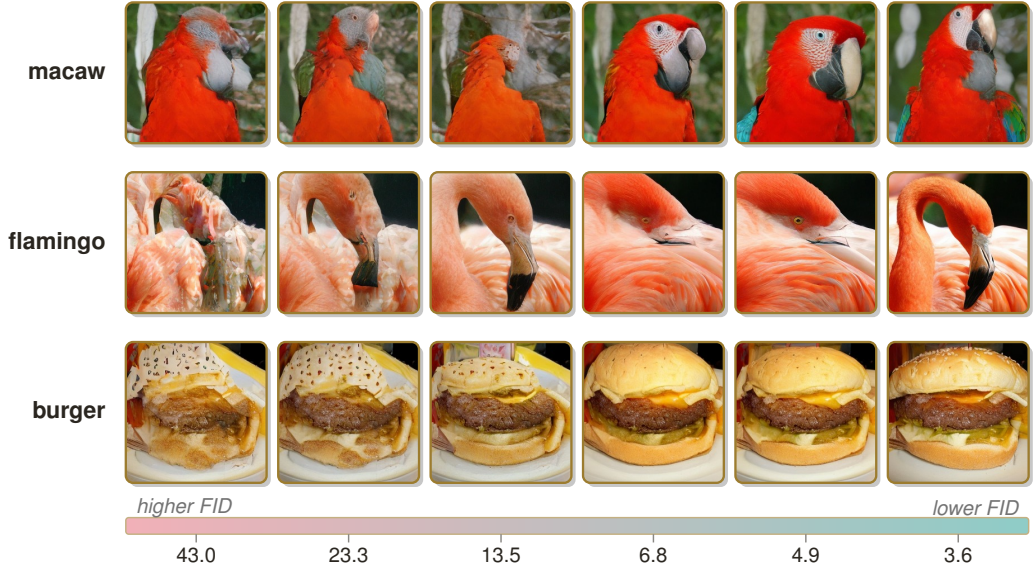
\begin{figure}[t]
  \centering
  \tikzsetnextfilename{output-howfid}%
  \resizebox{\linewidth}{!}{%
  \begin{tikzpicture}
    \useasboundingbox (-0.10,-0.90) rectangle (12.65,6.05);

    \foreach \cls/\ry in {macaw/5.05, flamingo/3.00, burger/0.95}{
      \foreach \cx/\j in {2.45/4, 4.29/3, 6.13/2, 7.97/1, 9.81/5, 11.65/0}{
        \fidtile{\cx}{\ry}{\cls\j}
      }
    }
    \foreach \name/\ry in {macaw/5.05, flamingo/3.00, burger/0.95}{
      \node[anchor=east, font=\flsf\footnotesize\bfseries, fidInk] at (1.48,\ry) {\name};
    }

    \shade[left color=fidHi, middle color=fidMid, right color=fidLo,
           rounded corners=1pt] (1.60,-0.52) rectangle (12.50,-0.32);
    \draw[fidGold!60, line width=0.5pt, rounded corners=1pt]
      (1.60,-0.52) rectangle (12.50,-0.32);
    \foreach \cx/\fid in {2.45/43.0, 4.29/23.3, 6.13/13.5, 7.97/6.8, 9.81/4.9, 11.65/3.6}{
      \draw[fidInk!55, line width=0.5pt] (\cx,-0.52) -- (\cx,-0.62);
      \node[font=\flsf\scriptsize, fidInk] at (\cx,-0.77) {\fid};
    }
    \node[anchor=west, font=\flsf\scriptsize\itshape, black!55] at (1.60,-0.20) {higher FID};
    \node[anchor=east, font=\flsf\scriptsize\itshape, black!55] at (12.50,-0.20) {lower FID};
  \end{tikzpicture}}
  \vspace{-.15in}
  \caption{\textbf{What does FID spectrum looks like?}
  Each row is one scene rendered by SiT-XL model whose Inception FID
  falls log-uniformly from $\mathbf{43}$ (left) to $\mathbf{3.6}$ (right),
  a $\mathbf{12\times}$ range: quality improves toward the right, as FID goes down.  
  FID is defined \emph{at the distribution level}. It's a Fr\'echet distance
  between Gaussians fit to the Inception features of the reference distribution (the ImageNet dataset) and
  $50{,}000$ generated images. It is a property of the whole
  \emph{set}, never of any single generated image. 
  For a fuller, set-level impression
  at each FID level, see the per-class galleries in
  \autoref{app:fid-visualization}.}
  \label{fig:how-fid}
\end{figure}

\myparagraph{Experimental Setting.}
\nicolas{All experiments train Scalable Interpolant Transformers
(SiT)~\citep{ma2024sit} at four widths (S/2, B/2, L/2, XL/2) on
class-conditional ImageNet
$256{\times}256$~\citep{deng2009imagenet} under conditional flow
matching~\citep{lipman2023flowmatching}. The loss redraws a Gaussian
$\varepsilon\!\sim\!\mathcal{N}(0,I)$ at every gradient step. In
contrast to the trained weights, this per-step noise never settles,
so we treat it as one of three training-time random number generators
(\autoref{sec:results-decomposition}). FID is computed in
Inception-V3 feature space~\citep{szegedy2016inceptionv3,heusel2017gans}
on one shared pipeline.
Because FID is a Fr\'echet distance between feature \emph{distributions}
rather than a per-image score, it can shift on differences too small to
perceive (\autoref{fig:how-fid}).
Sampling uses a fixed deterministic ODE solver and number of function
evaluations. Classifier-free guidance~\citep{ho2022cfg} is off by
default and only enabled in \autoref{sec:results-guidance}, where the
scale is selected per (training, sampling) seed pair by golden-section
search~\citep{kiefer1953sequential}. Per-experiment $N$, $K$, and
step counts are in \autoref{app:setup-detail}.}

\myparagraph{The two-axis panel.}
\nicolas{Every §4 experiment produces an $N\!\times\!K$ panel: $N$
independently trained models (\emph{training seeds}) and, for each,
$K$ generations under different \emph{sampling seeds}. A training
seed drives parameter initialisation, data-loader order, and the
per-step flow-matching noise. A sampling seed drives the initial
noise drawn at generation time. Panel sizes are stated per
subsection.}

\myparagraph{Notation.}
\nicolas{Three nested statistics summarise a panel.
$\sigma_{\text{within}}$, computed across the $K$ sampling evaluations
of one training seed and averaged over $N$ seeds, measures FID noise
for a fixed model (the \emph{generation lottery}).
$\sigma_{\text{between}}$, across the $N$ per-seed means measures
training-seed spread (the \emph{training lottery}). The coefficient
of variation $\mathrm{CoV}\!=\!\sigma/\mu$ (\%) is dimensionless and
comparable across panels whose absolute FID differs by an order of
magnitude (e.g.\ unguided vs.\ guided).}


\section{Experiments}
%
%

\label{sec:results}

\nicolas{Each subsection below answers one question about the FID
seed lottery on the panels of \autoref{sec:setup}. Two
further analyses on the converged SiT-B/2 panel appear in the
appendix: how the choice of summary statistic reshuffles
training-seed rankings (\autoref{sec:results-training}), and whether
``good'' init seeds transfer across (data, noise) pairings
(\autoref{sec:results-optimality}).}

\subsection{Training Variability Dominates Evaluation Variability}
\label{sec:results-sampling}

\begin{tldrbox}
The error bar you get from resampling a fixed model is the small
one. Retraining the same recipe moves FID $3.2{\times}$ more than
redrawing samples does, so most of the variance hides in the single
training run you happened to draw.
\end{tldrbox}

\myparagraph{The training lottery has a $3.2{\times}$ larger effect than the generation lottery.}
On the converged SiT-B/2 panel of \autoref{fig:overview} ($N\!=\!25$
training seeds, $K\!=\!10$ sampling seeds, $400$k steps, no CFG,
\autoref{app:setup-detail}) the asymmetry is visible at a glance.
Column-to-column spread (training-seed) overshoots within-column
spread (sampling-seed). The between-seed
$\sigma_{\text{between}}\!=\!0.438$ ($\mathrm{CoV}\!\approx\!1.3\%$)
is $3.2{\times}$ the within-seed
$\sigma_{\text{within}}\!=\!0.137$ ($\mathrm{CoV}\!\approx\!0.4\%$).
Per-seed mean Inception FIDs span $33.75\!\to\!35.42$ around a grand
mean of $34.74$. The variance a benchmark cares about lives in
\emph{which model} was trained, not in \emph{which samples} are drawn
from a fixed model.

\myparagraph{Sampling-seed \PP{confidence intervals} (CIs) misreport the spread.}
\nicolas{A $95\%$ Student's-$t$ interval of the grand mean from the
$25$ per-seed means is $34.74\!\pm\!0.18$ Inception FID, with a
one-$\sigma$ distance of $0.44$ FID. This is already larger than the
headline gain claimed in many recent papers. Multiplying the sampling
budget by ten on one run shrinks within-seed jitter by
$\sqrt{10}\!\approx\!3.2$ but leaves the $0.44$-wide between-seed
envelope untouched. The within-seed $\approx\!0.4\%$ CoV is
homoscedastic across the $25$ training seeds, so a single sampling-seed
FID carries $\approx\!0.14$ units of unrepeatable jitter even with the
model fixed. Only adding training seeds reduces the dominant source
of FID variance.}

\begin{figure*}[!tb]
  \centering
  \begin{minipage}[t]{0.66\textwidth}
    \pgfplotsset{width=\linewidth, height=4.6cm}
    \begin{tikzpicture}
      \violinsetoptions[scaled]{
        title={(a) Guided FID across 25 training seeds},
        xlabel={Training seed (sorted by per-seed mean guided FID)},
        ylabel={Guided Inception FID},
        xmin=0.0,
        xmax=26.0,
        ymin=7.18,
        ymax=7.66,
        xtick={1,5,10,15,20,25},
        ytick={7.2,7.3,7.4,7.5,7.6},
        ymajorgrids,
      }
      \violinplotwholefile[%
        col sep=tab,
        primary color=palLavDark,
        secondary color=palLavender,
        indexes={s01,s02,s03,s04,s05,s06,s07,s08,s09,s10,s11,s12,s13,s14,s15,s16,s17,s18,s19,s20,s21,s22,s23,s24,s25},
        labels={,,,,,,,,,,,,,,,,,,,,,,,,},
        spacing=1.0,
      ]{figures/fig_data/fig4_violins_wholefile.tsv}

      \begin{axis}[
        xmin=0.0, xmax=26.0,
        ymin=7.18, ymax=7.66,
        axis line style={draw=none},
        tick style={draw=none},
        xticklabels={,,}, yticklabels={,,},
        xmajorticks=false, ymajorticks=false,
        axis on top,
        clip=false,
      ]
        \draw[dashed, color=black!28, line width=0.5pt]
          (axis cs:0.0,7.418) -- (axis cs:26.0,7.418);
        \node[anchor=west, font=\scriptsize, color=black!50,
              fill=white, fill opacity=0.85, text opacity=1, inner sep=1pt]
            at (axis cs:0.15,7.418) {grand mean $7.42$};

        \addplot+[
          only marks, mark=*,
          mark size=0.8pt,
          mark options={fill=black!50, draw=black!50, line width=0pt, fill opacity=0.65},
        ] table[x=xj, y=fid] {figures/fig_data/fig4_strip.tsv};

        \addplot+[
          only marks, mark=-, mark size=4.5pt,
          mark options={draw=black!88, line width=1.2pt},
        ] table[x=x, y=mean_fid] {figures/fig_data/fig4_means.tsv};
      \end{axis}
    \end{tikzpicture}
  \end{minipage}\hfill
  \begin{minipage}[t]{0.31\textwidth}
    \begin{tikzpicture}
      \begin{axis}[
        title={(b) Unguided $\to$ guided rank},
        ylabel={Rank (1 = best)},
        ymin=0.4, ymax=25.6,
        y dir=reverse,
        ytick={1,5,10,15,20,25},
        xmin=0.7, xmax=2.3,
        xtick={1,2},
        xticklabels={Unguided, Guided ($\mathrm{cfg}^{\star}$)},
        x tick label style={font=\footnotesize},
        width=\linewidth, height=4.6cm,
      ]
        \addplot+[
          mark=*, mark size=1.4pt, line width=0.5pt,
          color=palLavender,
          mark options={fill=palLavender, draw=palLavDark, line width=0.3pt},
          opacity=0.75,
          forget plot,
        ] table[col sep=tab, x=x, y=rank]
            {figures/fig_data/fig4_bump_static.tsv};

        \addplot+[
          mark=*, mark size=2.4pt, line width=1.4pt,
          color=palCoralDark,
          mark options={fill=palCoral, draw=palCoralDark, line width=0.4pt},
          forget plot,
        ] table[col sep=tab, x=x, y=rank]
            {figures/fig_data/fig4_bump_moved.tsv};

        \node[anchor=north east, font=\scriptsize, fill=white, fill opacity=0.92,
              text opacity=1, inner sep=2.5pt, draw=black!18, rounded corners=1pt]
            at (axis cs:2.27, 1.2)
            {Spearman $\rho\!=\!0.73$};
      \end{axis}
    \end{tikzpicture}
  \end{minipage}
    \vspace{-.1in}
  \caption{\textbf{Per-cell guidance tuning halves the seed-induced FID
  spread, but reshuffles which seeds rank best.}
  \nicolas{Per-(training, sampling)-seed golden-section CFG search
  (GS-FID) across $25$ SiT-B/2 training seeds ($400$k steps, $10$
  sampling seeds per cell). (a)~Per-seed violins of guided Inception
  FID, sorted by per-seed mean. The relative spread tightens to
  $\mathrm{CoV}\!=\!0.67\%$, about half the $1.26\%$ measured unguided
  on the same panel. (b)~Tuning does not preserve the seed ranking:
  unguided and guided ranks correlate at only Spearman
  $\rho\!=\!0.73$. Lavender lines mark seeds that barely move
  ($|\Delta\!\operatorname{rank}|\!<\!5$), coral lines the $8/25$ seeds
  that shift by $|\Delta\!\operatorname{rank}|\!\geq\!5$.}}
  \label{fig:golden}
\end{figure*}
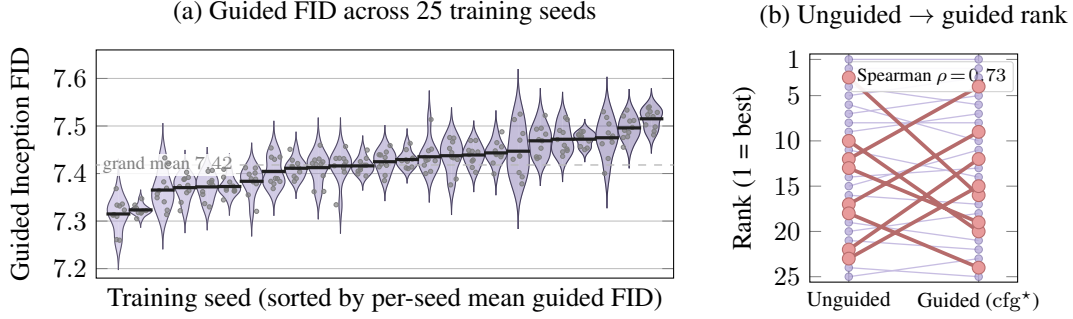

\subsection{Flow-Matching Noise Leads Init and Data Order}
\label{sec:results-decomposition}

\begin{tldrbox}
Training variance has three contributing sources. The per-step
Gaussian noise of the flow-matching loss is the largest, but
initialisation and data order add comparable amounts, and the three
overlap rather than sum.
\end{tldrbox}

\myparagraph{Three training-time sources.}
\nicolas{A SiT run draws from three independent generators:
initialisation, data-loader order, and the per-step Gaussian noise
$\varepsilon$ of the flow-matching loss (hereafter \emph{training
noise}, distinct from the sampling-time noise of
\autoref{sec:results-sampling}). Each single-source condition fixes
two and varies the third (\textsc{vary-noise/init/data}).
\textsc{vary-all} is the $25\!\times\!10$ panel of
\autoref{sec:results-sampling}. Per-condition $N$ and the SiT-B/2
$400$k protocol are in \autoref{app:setup-detail}.}

\myparagraph{Per-source contributions.}
The four conditions order monotonically by between-seed $\sigma$
(\autoref{fig:disentangle}b):
\textsc{vary-all} ($0.438$) $>$
\textsc{vary-noise} ($0.336$) $>$
\textsc{vary-init} ($0.294$) $>$
\textsc{vary-data} ($0.221$). Noise alone reproduces $77\%$ of the
baseline $\sigma$, init alone $67\%$, and data order $51\%$.
This contradicts the informal ``different seeds mean different inits'':
init matters, but as the second source. The within-seed $\sigma$ is
invariant across the four conditions ($0.137$--$0.150$, sage bars in
\autoref{fig:disentangle}b), so each per-source
$\sigma_{\text{between}}$ measures the trained model's spread, not the
scoring procedure's.

\myparagraph{The data lottery is shape-different, not just smaller.}
\textsc{Vary-data} has a tight bulk and a long upper tail (skewness
$+0.74$, IQR $0.30$, whiskers $-0.05$/$+0.47$), unlike the symmetric
\textsc{vary-init} ($-0.24$) and the broader, right-skewed
\textsc{vary-noise} ($+0.62$, \autoref{fig:disentangle}a).
Data-loader spread is therefore not continuous but a few outliers
resembling training failures: fixing init and noise yields mostly
near-identical FIDs punctuated by occasional bad runs.

\myparagraph{The sources combine sub-additively.}
The 
sum
$\sqrt{\sigma_{\text{noise}}^{2}\!+\!\sigma_{\text{init}}^{2}\!+\!\sigma_{\text{data}}^{2}}\!\approx\!0.50$ overshoots the observed
$\sigma_{\text{vary-all}}\!=\!0.44$ by $14\%$: noise, init and data
share variance through the trained weights, so each source counts
more in isolation than as a marginal increment on top of the others and
one-at-a-time ablations overestimate how much variance any individual fix recovers in a fully-random regime.

\myparagraph{Numerical non-determinism is not the source.}
\nicolas{A \textsc{same-seed} control fixes init, data order, \emph{and}
training noise across $N\!=\!24$ retrains, leaving only the bitwise
non-determinism of multi-GPU (DDP) execution, the floating-point
reduction-order effect known to perturb both training~\citep{summers2021nondeterminism}
and inference~\citep{he2025nondeterminism}. It compounds enough to
drive the EMA weights ${\approx}5$--$6\%$ of their norm apart (one run,
$33\%$): genuinely \emph{different} networks. Yet FID barely moves.
The between-seed $\sigma\!=\!0.047$ falls \emph{below} the within-seed
sampling floor ($0.119$), inverting the $3.2{\times}$ ratio of
\textsc{vary-all} to $0.4{\times}$ (\autoref{fig:disentangle}b,
rightmost). The lottery is thus driven by the \emph{intended} random
draws of init, data and noise, not by numerical noise.}

\begin{figure*}[!tb]
  \centering
  \begin{tikzpicture}
    \begin{groupplot}[
        group style={
          group size=3 by 1,
          horizontal sep=1.2cm,
        },
        fidlottery,
        width=0.36\textwidth,
        height=4.2cm,
        xmin=180000, xmax=2050000,
        scaled x ticks=false,
        xtick={200000,600000,1000000,1400000,1800000},
        xticklabels={200k,600k,1.0M,1.4M,1.8M},
        x tick label style={font=\scriptsize},
        xlabel={Training step},
      ]

      \nextgroupplot[
        title={(a) Per-seed FID},
        ylabel={Inception FID},
        ymin=13, ymax=54,
        ytick={15,20,25,30,40,50},
        legend to name=fig6legend,
        legend columns=-1,
        legend style={
          /tikz/every even column/.append style={column sep=0.6cm},
          font=\small, draw=black!15,
          inner sep=2pt, fill=white,
        },
        legend image post style={line width=1.2pt},
      ]
        \foreach \s in {3279,3479,3906,4166,11396,12281,13435,14593,28658,30496,32099,36049,71483,78908,83811,88697,96531,97081,97197} {
          \addplot+[mark=none, line width=0.3pt, color=palRose, opacity=0.45, forget plot]
            table[col sep=tab, x=step, y=s\s] {figures/fig_data/fig6a_wide_DiT-S.tsv};
        }
        \foreach \s in {3279,3479,4166,12281,13435,14593,18290,28658,29257,30496,32099,36049,55303,71483,73564,78908,83811,88697,97081,97197} {
          \addplot+[mark=none, line width=0.3pt, color=palLavender, opacity=0.45, forget plot]
            table[col sep=tab, x=step, y=s\s] {figures/fig_data/fig6a_wide_DiT-B.tsv};
        }
        \foreach \s in {3279,3479,3906,4166,11396,12281,13435,14593,18290,28658,29257,30496,32099,36049,55303,66238,71483,73564,77398,78908,83811,96531,97081,97197} {
          \addplot+[mark=none, line width=0.3pt, color=palTeal, opacity=0.45, forget plot]
            table[col sep=tab, x=step, y=s\s] {figures/fig_data/fig6a_wide_DiT-L.tsv};
        }
        \foreach \s in {3279,3479,3906,4166,11396,12281,13435,14593,18290,28658,29257,30496,32099,36049,55303,66238,71483,73564,77398,78908,83811,88697,96531,97081,97197} {
          \addplot+[mark=none, line width=0.3pt, color=palOchre, opacity=0.45, forget plot]
            table[col sep=tab, x=step, y=s\s] {figures/fig_data/fig6a_wide_DiT-XL.tsv};
        }
        \addplot+[
          mark=none, line width=1.4pt, color=palRoseDark, smooth,
        ] table[col sep=tab, x=step, y=mean] {figures/fig_data/fig6a_envelope_dits.tsv};
        \addlegendentry{SiT-S}
        \addplot+[
          mark=none, line width=1.4pt, color=palLavDark, smooth,
        ] table[col sep=tab, x=step, y=mean] {figures/fig_data/fig6a_envelope_ditb.tsv};
        \addlegendentry{SiT-B}
        \addplot+[
          mark=none, line width=1.4pt, color=palTealDark, smooth,
        ] table[col sep=tab, x=step, y=mean] {figures/fig_data/fig6a_envelope_ditl.tsv};
        \addlegendentry{SiT-L}
        \addplot+[
          mark=none, line width=1.4pt, color=palOchreDark, smooth,
        ] table[col sep=tab, x=step, y=mean] {figures/fig_data/fig6a_envelope_ditxl.tsv};
        \addlegendentry{SiT-XL}

      \nextgroupplot[
        title={(b) CoV $\boldsymbol{\sigma/\mu}$ over training},
        ylabel={CoV (\%)},
        ylabel shift=-8pt,
        ymin=0.7, ymax=2.5,
        ytick={1.0, 1.5, 2.0, 2.5},
      ]
        \addplot+[name path=bandlo, mark=none, draw=none, forget plot]
          coordinates {(180000,1.0) (2050000,1.0)};
        \addplot+[name path=bandhi, mark=none, draw=none, forget plot]
          coordinates {(180000,2.0) (2050000,2.0)};
        \addplot[palSage, opacity=0.30, forget plot] fill between[of=bandlo and bandhi];
        \node[anchor=west, font=\scriptsize, color=palSageDark, fill=white,
              fill opacity=0.85, text opacity=1, inner sep=1.5pt]
          at (axis cs:230000,2.35) {1.0--2.0\,\% band};

        \addplot+[
          mark=*, mark size=1.5pt, line width=1.0pt,
          color=palRoseDark, mark options={fill=palRose, draw=palRoseDark}, smooth, forget plot,
        ] table[col sep=tab, x=step, y=cov_pct]
            {figures/fig_data/fig6_summary_DiT-S.tsv};
        \addplot+[
          mark=*, mark size=1.5pt, line width=1.0pt,
          color=palLavDark, mark options={fill=palLavender, draw=palLavDark}, smooth, forget plot,
        ] table[col sep=tab, x=step, y=cov_pct]
            {figures/fig_data/fig6_summary_DiT-B.tsv};
        \addplot+[
          mark=*, mark size=1.5pt, line width=1.0pt,
          color=palTealDark, mark options={fill=palTeal, draw=palTealDark}, smooth, forget plot,
        ] table[col sep=tab, x=step, y=cov_pct]
            {figures/fig_data/fig6_summary_DiT-L.tsv};
        \addplot+[
          mark=*, mark size=1.5pt, line width=1.0pt,
          color=palOchreDark, mark options={fill=palOchre, draw=palOchreDark}, smooth, forget plot,
        ] table[col sep=tab, x=step, y=cov_pct]
            {figures/fig_data/fig6_summary_DiT-XL.tsv};

      \nextgroupplot[
        title={(c) Rank stability vs. 2M},
        ylabel={Spearman $\rho$ vs. 2M},
        ylabel shift=-4pt,
        ymin=0.3, ymax=1.05,
        ytick={0.4,0.6,0.8,1.0},
      ]
        \addplot+[mark=none, dashed, color=black!30, line width=0.5pt, forget plot]
          coordinates {(180000,0.8) (2050000,0.8)};
        \addplot+[mark=none, dashed, color=black!30, line width=0.5pt, forget plot]
          coordinates {(180000,1.0) (2050000,1.0)};

        \addplot+[
          mark=*, mark size=1.5pt, line width=1.0pt, color=palRoseDark,
          mark options={fill=palRose, draw=palRoseDark}, smooth, forget plot,
        ] table[col sep=tab, x=step, y=rho] {figures/fig_data/fig7a_rho_DiT-S.tsv};
        \addplot+[
          mark=*, mark size=1.5pt, line width=1.0pt, color=palLavDark,
          mark options={fill=palLavender, draw=palLavDark}, smooth, forget plot,
        ] table[col sep=tab, x=step, y=rho] {figures/fig_data/fig7a_rho_DiT-B.tsv};
        \addplot+[
          mark=*, mark size=1.5pt, line width=1.0pt, color=palTealDark,
          mark options={fill=palTeal, draw=palTealDark}, smooth, forget plot,
        ] table[col sep=tab, x=step, y=rho] {figures/fig_data/fig7a_rho_DiT-L.tsv};
        \addplot+[
          mark=*, mark size=1.5pt, line width=1.0pt, color=palOchreDark,
          mark options={fill=palOchre, draw=palOchreDark}, smooth, forget plot,
        ] table[col sep=tab, x=step, y=rho] {figures/fig_data/fig7a_rho_DiT-XL.tsv};

    \end{groupplot}
  \end{tikzpicture}

  \centerline{\pgfplotslegendfromname{fig6legend}}
  \vspace{-0.2em}
  \caption{\textbf{The seed lottery across compute and model size.}
  \nicolas{(a)~Inception FID over training: thin pastel lines are
  individual training-seed trajectories, bold lines are per-step means.
  The spread between seeds stays wide at every checkpoint and does not
  shrink as training converges. (b)~Coefficient of
  variation $\sigma/\mu$ over training: all four models stay near a
  $1$--$2\%$ band. Bigger models do not yield proportionally tighter FID.
  (c)~Spearman $\rho$ between the seed ranking at step $t$ and at $2$M:
  weak before $\sim\!1$M steps.}}
  \label{fig:variance-training}
\end{figure*}
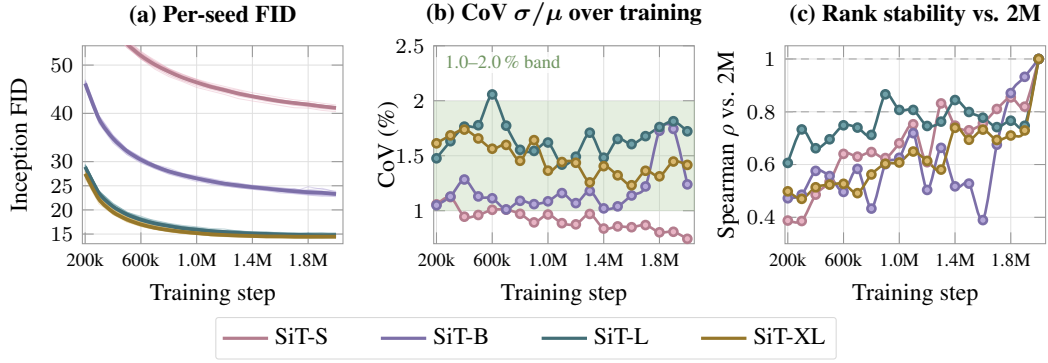

\subsection{GS-FID Halves the Floor but Reshuffles Rankings}
\label{sec:results-guidance}

\begin{tldrbox}
Tuning CFG separately for every seed makes FID more repeatable,
nearly halving the relative spread ($\mathrm{CoV}\!:1.26\%\!\to\!0.67\%$).
But it reshuffles which seeds come out best (Spearman
$\rho\!=\!0.73$), so a seed chosen by unguided FID is not reliably the
best one once guidance is tuned.
\end{tldrbox}

\myparagraph{Procedure.}
\nicolas{\textbf{GS-FID} (golden-section FID) runs golden-section
search~\citep{kiefer1953sequential} on the CFG scale for every
(training, sampling) seed pair over $[\omega_{\min},\omega_{\max}]\!=\![1,2]$
at tolerance $0.01$. Each iteration queries two interior probes and
discards the half-bracket with the larger FID, costing $\approx\!14$
evaluations per cell. Pseudocode (\autoref{alg:gss}), the one-step
illustration (\autoref{fig:gss-diagram}), and the unimodality and
convergence proofs are in \autoref{app:golden-section}.}

\myparagraph{GS-FID halves the relative noise floor.}
Per-seed mean GS-FID spans $[7.31, 7.52]$ around a grand mean of
$7.42$, with $\sigma_{\text{between}}\!=\!0.050$ and
$\sigma_{\text{within}}\!=\!0.027$ (\autoref{fig:golden}a). The CoV
drops to $0.67\%$, half the $1.26\%$ unguided. The grand mean also
falls from $\approx\!34.7$ to $\approx\!7.4$, so absolute FID ranges
overstate the gain and the dimensionless CoV is the comparable
quantity across CFG conditions.

\myparagraph{Sampling jitter takes a larger share under GS-FID.}
\nicolas{The GS-FID floor still sits above the $\approx\!0.4\%$
pure-sampling floor of \autoref{sec:results-sampling}, and the
between-to-within $\sigma$ ratio falls only from $3.2{\times}$ to
$1.87{\times}$. Per-cell tuning does not eliminate the seed lottery.
Sampling jitter takes a larger share of what remains, so
multi-sampling-seed reporting matters \emph{more} under GS-FID, not
less. The recovered optima concentrate tightly
($\sigma_{\omega}\!\approx\!0.045$), so a $\boldsymbol{\pm 0.05}$
miscalibration of the scale injects FID noise comparable to the
within-seed floor: any single CFG number must come with its search
tolerance.}

\myparagraph{GS-FID reshuffles the seed ranking.}
Across mean, min, and median criteria, Spearman $\rho$ between the
unguided and GS-FID rankings of the $25$ training seeds is $0.73$
($p\!<\!10^{-4}$). The bump chart in \autoref{fig:golden}(b) shows
$8/25$ seeds shift by at least five places, and the top of the
leaderboard mixes seeds the unguided ranking placed near the middle.
A model selected on unguided FID is not guaranteed to be best under
GS-FID, and the two protocols should not be compared across papers.
Within a coherent ablation GS-FID is the more precise estimator.
Benchmarks that report either should report both the noise floor and
the search tolerance.

\begin{figure*}[!tb]
  \centering
  \resizebox{\textwidth}{!}{%
  \begin{tikzpicture}
    \begin{groupplot}[
        group style={
          group size=4 by 1,
          horizontal sep=0.55cm,
        },
        fidlottery,
        width=0.275\textwidth,
        height=4.6cm,
        xmin=600000, xmax=2080000,
        scaled x ticks=false,
        xtick={800000,1400000,2000000},
        xticklabels={800k,1.4M,2M},
        title style={font=\small\bfseries, yshift=-3pt},
        xlabel={Training step},
      ]

      \nextgroupplot[
        title={(a) SiT-S},
        ylabel={Inception FID},
        ymin=40.5, ymax=53,
        ytick={42,44,46,48,50,52},
      ]
        \fill[palAmber, opacity=0.22]
          (axis cs:1600000,40.5) rectangle (axis cs:2000000,53);
        \addplot+[name path=Smin, mark=none, draw=none, forget plot]
          table[col sep=tab, x=step, y=min] {figures/fig_data/fig6_summary_DiT-S.tsv};
        \addplot+[name path=Smax, mark=none, draw=none, forget plot]
          table[col sep=tab, x=step, y=max] {figures/fig_data/fig6_summary_DiT-S.tsv};
        \addplot[palRose, opacity=0.40, forget plot] fill between[of=Smin and Smax];
        \addplot+[
          mark=none, line width=1.2pt,
          color=palRoseDark, smooth, forget plot,
        ] table[col sep=tab, x=step, y=mean] {figures/fig_data/fig6_summary_DiT-S.tsv};
        \addplot+[mark=none, dashed, color=black!60, line width=0.7pt, forget plot]
          coordinates {(600000,41.74) (2080000,41.74)};
        \node[anchor=west, font=\footnotesize, fill=white, fill opacity=0.8,
              text opacity=1, inner sep=1.5pt]
          at (axis cs:680000,42.7) {$T\!=\!41.74$};
        \node[circle, draw=black!70, fill=palSage, inner sep=0pt, minimum size=6pt]
          at (axis cs:1600000,41.74) {};
        \node[circle, draw=black!70, fill=palCoral, inner sep=0pt, minimum size=6pt]
          at (axis cs:2000000,41.74) {};
        \node[anchor=center, font=\footnotesize\bfseries, color=white,
              fill=palAmberDark, rounded corners=2pt, inner sep=2.5pt]
          at (axis cs:1800000,48.5)
          {$\boldsymbol{1.25{\times}}$};

      \nextgroupplot[
        title={(b) SiT-B},
        ymin=22.5, ymax=31.5,
        ytick={23,25,27,29,31},
      ]
        \fill[palAmber, opacity=0.22]
          (axis cs:1600000,22.5) rectangle (axis cs:2000000,31.5);
        \addplot+[name path=Bmin, mark=none, draw=none, forget plot]
          table[col sep=tab, x=step, y=min] {figures/fig_data/fig6_summary_DiT-B.tsv};
        \addplot+[name path=Bmax, mark=none, draw=none, forget plot]
          table[col sep=tab, x=step, y=max] {figures/fig_data/fig6_summary_DiT-B.tsv};
        \addplot[palLavender, opacity=0.40, forget plot] fill between[of=Bmin and Bmax];
        \addplot+[
          mark=none, line width=1.2pt,
          color=palLavDark, smooth, forget plot,
        ] table[col sep=tab, x=step, y=mean] {figures/fig_data/fig6_summary_DiT-B.tsv};
        \addplot+[mark=none, dashed, color=black!60, line width=0.7pt, forget plot]
          coordinates {(600000,23.86) (2080000,23.86)};
        \node[anchor=west, font=\footnotesize, fill=white, fill opacity=0.8,
              text opacity=1, inner sep=1.5pt]
          at (axis cs:680000,24.7) {$T\!=\!23.86$};
        \node[circle, draw=black!70, fill=palSage, inner sep=0pt, minimum size=6pt]
          at (axis cs:1600000,23.86) {};
        \node[circle, draw=black!70, fill=palCoral, inner sep=0pt, minimum size=6pt]
          at (axis cs:2000000,23.86) {};
        \node[anchor=center, font=\footnotesize\bfseries, color=white,
              fill=palAmberDark, rounded corners=2pt, inner sep=2.5pt]
          at (axis cs:1800000,29)
          {$\boldsymbol{1.25{\times}}$};

      \nextgroupplot[
        title={(c) SiT-L},
        ymin=14, ymax=19,
        ytick={14,15,16,17,18},
      ]
        \fill[palAmber, opacity=0.22]
          (axis cs:1100000,14) rectangle (axis cs:2000000,19);
        \addplot+[name path=Lmin, mark=none, draw=none, forget plot]
          table[col sep=tab, x=step, y=min] {figures/fig_data/fig6_summary_DiT-L.tsv};
        \addplot+[name path=Lmax, mark=none, draw=none, forget plot]
          table[col sep=tab, x=step, y=max] {figures/fig_data/fig6_summary_DiT-L.tsv};
        \addplot[palTeal, opacity=0.40, forget plot] fill between[of=Lmin and Lmax];
        \addplot+[
          mark=none, line width=1.2pt,
          color=palTealDark, smooth, forget plot,
        ] table[col sep=tab, x=step, y=mean] {figures/fig_data/fig6_summary_DiT-L.tsv};
        \addplot+[mark=none, dashed, color=black!60, line width=0.7pt, forget plot]
          coordinates {(600000,15.22) (2080000,15.22)};
        \node[anchor=west, font=\footnotesize, fill=white, fill opacity=0.8,
              text opacity=1, inner sep=1.5pt]
          at (axis cs:680000,15.85) {$T\!=\!15.22$};
        \node[circle, draw=black!70, fill=palSage, inner sep=0pt, minimum size=6pt]
          at (axis cs:1100000,15.22) {};
        \node[circle, draw=black!70, fill=palCoral, inner sep=0pt, minimum size=6pt]
          at (axis cs:2000000,15.22) {};
        \node[anchor=center, font=\footnotesize\bfseries, color=white,
              fill=palAmberDark, rounded corners=2pt, inner sep=2.5pt]
          at (axis cs:1550000,17.4)
          {$\boldsymbol{1.82{\times}}$};

      \nextgroupplot[
        title={(d) SiT-XL},
        ymin=13.9, ymax=17.8,
        ytick={14,15,16,17},
      ]
        \fill[palAmber, opacity=0.22]
          (axis cs:1000000,13.9) rectangle (axis cs:2000000,17.8);
        \addplot+[name path=XLmin, mark=none, draw=none, forget plot]
          table[col sep=tab, x=step, y=min] {figures/fig_data/fig6_summary_DiT-XL.tsv};
        \addplot+[name path=XLmax, mark=none, draw=none, forget plot]
          table[col sep=tab, x=step, y=max] {figures/fig_data/fig6_summary_DiT-XL.tsv};
        \addplot[palOchre, opacity=0.40, forget plot] fill between[of=XLmin and XLmax];
        \addplot+[
          mark=none, line width=1.2pt,
          color=palOchreDark, smooth, forget plot,
        ] table[col sep=tab, x=step, y=mean] {figures/fig_data/fig6_summary_DiT-XL.tsv};
        \addplot+[mark=none, dashed, color=black!60, line width=0.7pt, forget plot]
          coordinates {(600000,14.94) (2080000,14.94)};
        \node[anchor=west, font=\footnotesize, fill=white, fill opacity=0.8,
              text opacity=1, inner sep=1.5pt]
          at (axis cs:680000,15.55) {$T\!=\!14.94$};
        \node[circle, draw=black!70, fill=palSage, inner sep=0pt, minimum size=6pt]
          at (axis cs:1000000,14.94) {};
        \node[circle, draw=black!70, fill=palCoral, inner sep=0pt, minimum size=6pt]
          at (axis cs:2000000,14.94) {};
        \node[anchor=center, font=\footnotesize\bfseries, color=white,
              fill=palAmberDark, rounded corners=2pt, inner sep=2.5pt]
          at (axis cs:1500000,16.6)
          {$\boldsymbol{2.00{\times}}$};

    \end{groupplot}
  \end{tikzpicture}}
  \vspace{-.2in}
  \caption{\textbf{The luck of the draw: a $\boldsymbol{1.2}$--$\boldsymbol{2.0{\times}}$
  convergence gap.}
  \nicolas{For each model the dashed horizontal line marks the target
  $T$, the FID reached by the \emph{unluckiest} of $\sim\!20$ seeds at
  $2$M. The green dot is the step at which the \emph{luckiest} seed
  first crosses $T$. The coral dot sits at $2$M where the unlucky seed
  finally reaches it. The amber band between them is the training
  compute the unlucky seed wastes catching up. The per-step shaded band is
  the min--max envelope across seeds, and the bold line is the per-step mean.}}
  \label{fig:lucky-speedup}
\end{figure*}
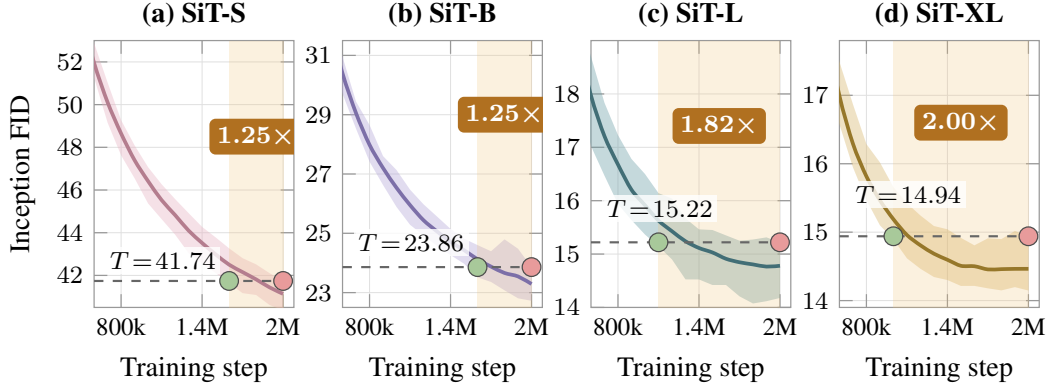

\subsection{The 1--2\,\% CoV Floor Survives Scale and Compute}
\label{sec:results-scaling}

\begin{tldrbox}
More compute and larger models do not reduce the variance: the
CoV stays in a $1$--$2\%$ band (median $1.30\%$) at every
checkpoint and every model size.
\end{tldrbox}

\myparagraph{The relative floor is scale-invariant.}
\nicolas{Across the SiT-S/B/L/XL panel ($N\!=\!25$, $K\!=\!10$ at
every $100$k-step checkpoint to $2$M, \autoref{app:setup-detail}),
mean Inception FID drops $\approx\!2{\times}$ from $200$k to $2$M
while $\sigma_{\text{between}}$ shrinks at most $2.4{\times}$. A fan of
seed-to-seed differences trails the mean down rather than collapsing
at convergence (\autoref{fig:variance-training}a). The CoV stays
inside $[0.74\%, 2.06\%]$ across all $76$ cells (median $1.30\%$,
\autoref{fig:variance-training}b): \emph{the FID noise floor is
$1$--$2\%$ of the mean FID the model has reached, at every compute
budget and scale on this family.} A gain below
$\approx\!2{\times}\mathrm{CoV}$ of the mean FID ($\approx\!3$--$4\%$
of the baseline) sits inside the floor and should not be reported as
real without multi-seed confirmation.}

\myparagraph{More parameters does not mean less variance.}
\nicolas{The CoV at $2$M is non-monotonic in scale: SiT-S ($0.74\%$)
and SiT-B ($1.24\%$) sit \emph{below} SiT-XL ($1.42\%$) and SiT-L
($1.72\%$). Bigger architectures do not automatically tighten the
floor. This matches the decomposition of
\autoref{sec:results-decomposition}: the per-step noise that
dominates the seed lottery is regenerated every batch, so it neither
fades with compute nor averages across width. Reproducibility is a
property of the metric and the loss, not of compute or scale.}

\myparagraph{Rank stability is weak before $\boldsymbol{\approx 1}$M
steps.}
\nicolas{Spearman $\rho$ between the seed ranking at step $t$ and at
$2$M is $0.39$--$0.61$ at $200$k, climbs to $0.65$--$0.81$ by $1.1$M,
and reaches $1.0$ at $2$M (\autoref{fig:variance-training}c):
selecting a seed on an early checkpoint and reusing it for a long
final run amounts to picking near-randomly through the first half of
training. The spread is Gaussian, not heavy-tailed:
$(\max\!-\!\min)/\sigma\!\in\![3.1, 5.0]$ at every step, bracketing
the $3.5$--$3.9$ predicted for a Gaussian sample of size
$n\!\in\![19, 25]$.}

\input{figures/fig9_mup_lr_sweep.tex}

\subsection{The Luck of the Draw: A \texorpdfstring{1.2--2.0$\boldsymbol{\times}$}{1.2--2.0x} Convergence Gap}
\label{sec:results-speedup}

\begin{tldrbox}
Drawing many training seeds and cherry-picking the luckiest gives
absurdly impressive headlines. Without changing the recipe, training
runs up to $2.0{\times}$ faster, from $1.25{\times}$ on SiT-S/B to
$2.0{\times}$ on SiT-XL.
\end{tldrbox}

\myparagraph{An extra 100k training steps and a seed swap move FID by
similar amounts.}
\nicolas{Past $\approx\!1.5$M steps, the smallest single-seed
improvement clearing $2\sigma$ ($\approx\!0.5$--$0.8$ FID at $2$M) is
indistinguishable from $200$--$500$k of extra training on the same
architecture. With $N$ seeds the resolvable gap scales as
$2\sigma/\sqrt{N}$: $N\!=\!5$ cuts the threshold to $\approx\!0.25$
FID, restoring detectability for $100$k-step increments.}

\myparagraph{The lucky--unlucky gap is 400k--1M steps depending on size.}
\nicolas{We anchor the gap to a familiar target (the FID the
unluckiest seed reaches at $2$M) and ask when the luckiest first
hits it (\autoref{fig:lucky-speedup}): $1.25{\times}$ on SiT-S/B,
$1.82{\times}$ on SiT-L, $2.0{\times}$ on SiT-XL, where the lucky seed
reaches at $1$M what the unlucky seed only attains at $2$M. An
unlucky seed therefore costs between a fifth and a half of the
training budget compared to a lucky one. Any single-seed paper
claiming $\approx\!1.3{\times}$ training speedup on this architecture
is implicitly competing with what the seed lottery already delivers
without changing a line of code.}

\subsection{\texorpdfstring{$\boldsymbol{\mu}$P}{muP} Transfers a \texorpdfstring{1.7$\boldsymbol{\times}$}{1.7x} LR Window, Not a Point}
\label{sec:results-mup}

\begin{tldrbox}
With CFG, several nearby LRs all give similar FID, a
$1.7{\times}$-wide window. Without CFG, the best LR moves to the
edge of training stability.
\end{tldrbox}

\myparagraph{$\boldsymbol{\mu}$P transfers, but to a window, not a point.}
\nicolas{The sweep covers $10$ $\mu$P-coordinated LRs log-spaced in
$[5\!\times\!10^{-5}, 5\!\times\!10^{-4}]$ across SiT-S/B/L/XL with
$N\!=\!10$ seeds per cell, $100$k steps, and both unguided FID and
GS-FID ($\approx\!400$ networks, \autoref{app:setup-detail}). The
GS-FID panel of \autoref{fig:mup-lr} shows flat-bottomed valleys near
$\omega^{\star}\!\approx\!2$--$3\!\times\!10^{-4}$ at every size, with
highlighted minima within one log-step. The two LRs flanking each
minimum sit inside its seed envelope, so three adjacent LRs share the
per-size best FID, a $1.7{\times}$ LR window per size
(\autoref{app:mup-extras}). $\mu$P does transfer the optimum across
widths, but the object it transfers is a window, not a single number.
A recipe-comparison study sweeping a few LRs and reporting the
per-recipe best is out-resolved by the seed lottery, exactly like the
single-LR practitioner of \autoref{sec:results-scaling}.}

\myparagraph{The seed CoV does not dip at the optimal LR.}
\nicolas{A flatter loss-landscape region should also be lower-variance,
but \autoref{fig:mup-lr}c refutes this: every size's argmin CoV lies
$1$--$3$ brackets \emph{below} its argmin mean, in the under-trained
regime where seeds have not moved far from initialisation. At the
actual optimum LR (open rings, panel c) the CoV is $1.7$--$2.3\%$,
inside the $1$--$2\%$ floor of \autoref{sec:results-scaling}. The
floor is therefore the variance at the LR a practitioner would
actually pick: under-training squashes seeds toward initialisation,
but the practitioner gives that regime up when they tune for FID.}

\myparagraph{Without CFG, the sweep points to the unstable edge.}
\nicolas{The unguided panel is monotone in LR at $100$k, so the
lowest mean sits at the right edge ($5\!\times\!10^{-4}$), the
same LR where $3/10$ SiT-S and $1/10$ SiT-XL seeds diverge.
Single-seed unguided LR selection therefore lands on an unstable
point with a confident-looking number, while GS-FID returns the
$1.7{\times}$ window around $3\!\times\!10^{-4}$.}

{\tikzexternaldisable\begin{advicebox}
\textbf{1. More evaluations can't substitute for more training runs.}
Resampling on a fixed network shrinks evaluation noise but leaves the
dominant training variance untouched
(\autoref{sec:results-sampling}). Only multi-seed training reaches
below the $\approx\!1.3\%$ CoV floor.

\medskip
\textbf{2. Treat any FID gap below $\approx\!2\%$ as inconclusive.}
FID variance stays inside a $1$--$2\%$ band across SiT-S/B/L/XL
from $200$k to $2$M training steps (\autoref{sec:results-scaling}).
Gaps below the band may just be seed noise. Use this as a cheap
first-pass check before running multiple seeds. Our
\href{https://kyutai.org/fid-lottery/\#calculator}{online error-bar
calculator} returns the seed-only $95\%$ CI for any reported FID.

\medskip
\textbf{3. Guided and unguided FID disagree on best seeds and
hyperparameters.}
GS-FID is more reliable, but its best seeds differ from unguided's
(Spearman $\rho\!\approx\!0.73$, \autoref{sec:results-guidance}). For
LRs, GS-FID returns a stable optimum while unguided picks the unstable
edge of training (\autoref{sec:results-mup}). Evaluate and tune with
the same FID you plan to report.

\medskip
\textbf{4. Under guided FID, the best LR is a flat region, not a
single value.}
On the $\mu$P sweep, GS-FID returns a $1.7{\times}$-wide window of
adjacent LRs that all give similar FID (\autoref{sec:results-mup}).
Seed variance blurs the optimum into a flat region.

\medskip
\textbf{5. Use golden-section search to pick the CFG scale.}
GS-FID finds the per-cell optimal CFG scale in a logarithmic number
of evaluations (\autoref{sec:results-guidance}) and gives the most
reliable comparisons under CFG.
\end{advicebox}
}

\section{Discussion}
\label{sec:conclusion}

\nicolas{\myparagraph{Limitations.}
Our measurements cover one combination: SiT, flow matching,
class-conditional ImageNet $256{\times}256$, Inception-V3 FID. The
$\approx\!1.3\%$ CoV is a calibration target for that combination,
not a universal constant. Other backbones, objectives, latent vs.\
pixel diffusion, text-to-image, or other Fr\'echet variants (FDD,
CMMD, KID) may sit on different floors. \autoref{app:supp-metrics}
replicates the analyses on DINOv2 FID and Inception
precision/recall/density/coverage: fidelity metrics track Inception
FID closely. Recall is the outlier. The panel is finite ($20$--$25$
training seeds, $10$ sampling seeds), and we do not push past SiT-XL
or $2$M steps, so production-scale behaviour is an extrapolation.}

\nicolas{\myparagraph{Outlook.}
We read these findings less as a verdict on FID than as a proposal for how
to measure it: a two-axis panel of training and sampling seeds, summarised
by a per-source variance decomposition. We establish this for a single
combination (SiT, flow matching, Inception-V3 FID). Whether the same panel
and the $1$--$2\%$ floor extend to other training methods, model families,
and Fr\'echet-style metrics is left to future work, as is whether
seed-noise floors can be predicted from proxies cheaper than a multi-seed
retrain.}

\nicolas{\myparagraph{Broader impact.}
The findings are dual-use in a benign sense: knowing the
$\approx\!1.3\%$ floor saves compute by avoiding sub-floor retraining,
but the same numbers reveal the size of headline a few extra seeds can
manufacture. We make this temptation explicit
(\autoref{sec:results-speedup}) so it can be priced into peer review.}

\section*{Acknowledgements}
We thank David Picard for the inspiration behind this project. We
thank Amil Dravid, Richard Zhang, A. Sophia Koepke, and David Picard
for proofreading the manuscript, and Adrien Ramanana-Rahary for the
interesting discussions. AE is supported, in part, by NSF IIS-2403305
and ONR MURI.

\newpage
\bibliographystyle{plainnat}
\bibliography{references}

\begin{thebibliography}{99}
\providecommand{\natexlab}[1]{#1}
\providecommand{\url}[1]{\texttt{#1}}
\expandafter\ifx\csname urlstyle\endcsname\relax
  \providecommand{\doi}[1]{doi: #1}\else
  \providecommand{\doi}{doi: \begingroup \urlstyle{rm}\Url}\fi

\bibitem[Banerjee et~al.(2024)Banerjee, Marrinan, Cannon, Chiang, and Sarwate]{banerjee2024measuring}
Sinjini Banerjee, Tim Marrinan, Reilly Cannon, Tony Chiang, and Anand~D. Sarwate.
\newblock Measuring training variability from stochastic optimization using robust nonparametric testing.
\newblock \emph{arXiv preprint arXiv:2406.08307}, 2024.

\bibitem[Barratt and Sharma(2018)]{barratt2018note}
Shane Barratt and Rishi Sharma.
\newblock A note on the inception score.
\newblock \emph{arXiv preprint arXiv:1801.01973}, 2018.

\bibitem[Benavoli et~al.(2017)Benavoli, Corani, Dem{\v{s}}ar, and Zaffalon]{benavoli2017bayesian}
Alessio Benavoli, Giorgio Corani, Janez Dem{\v{s}}ar, and Marco Zaffalon.
\newblock Time for a change: a tutorial for comparing multiple classifiers through {B}ayesian analysis.
\newblock \emph{Journal of Machine Learning Research}, 2017.

\bibitem[Bench and Thomas(2025)]{bench2025quantifying}
Ciaran Bench and Spencer~Angus Thomas.
\newblock Quantifying the uncertainty of model-based synthetic image quality metrics.
\newblock \emph{arXiv preprint arXiv:2504.03623}, 2025.

\bibitem[Bhatia(2007)]{bhatia2007positive}
Rajendra Bhatia.
\newblock \emph{Positive Definite Matrices}.
\newblock Princeton University Press, 2007.

\bibitem[Bi{\'n}kowski et~al.(2018)Bi{\'n}kowski, Sutherland, Arbel, and Gretton]{binkowski2018demystifying}
Miko{\l}aj Bi{\'n}kowski, Danica~J. Sutherland, Michael Arbel, and Arthur Gretton.
\newblock Demystifying {MMD} {GAN}s.
\newblock In \emph{6th International Conference on Learning Representations (ICLR)}, 2018.

\bibitem[Borji(2019)]{borji2019pros}
Ali Borji.
\newblock Pros and cons of {GAN} evaluation measures.
\newblock \emph{Computer Vision and Image Understanding}, 2019.

\bibitem[Borji(2022)]{borji2022pros}
Ali Borji.
\newblock Pros and cons of {GAN} evaluation measures: New developments.
\newblock \emph{Computer Vision and Image Understanding}, 2022.

\bibitem[Bouthillier et~al.(2019)Bouthillier, Laurent, and Vincent]{bouthillier2019unreproducible}
Xavier Bouthillier, C{\'e}sar Laurent, and Pascal Vincent.
\newblock Unreproducible research is reproducible.
\newblock In \emph{Proceedings of the 36th International Conference on Machine Learning (ICML)}, 2019.

\bibitem[Bouthillier et~al.(2021)Bouthillier, Delaunay, Bronzi, Trofimov, Nichyporuk, Szeto, Sepah, Raff, Madan, Voleti, Ebrahimi~Kahou, Michalski, Serdyuk, Arbel, Pal, Varoquaux, and Vincent]{bouthillier2021accounting}
Xavier Bouthillier, Pierre Delaunay, Mirko Bronzi, Assya Trofimov, Brennan Nichyporuk, Justin Szeto, Naz Sepah, Edward Raff, Kanika Madan, Vikram Voleti, Samira Ebrahimi~Kahou, Vincent Michalski, Dmitriy Serdyuk, Tal Arbel, Chris Pal, Ga{\"e}l Varoquaux, and Pascal Vincent.
\newblock Accounting for variance in machine learning benchmarks.
\newblock In \emph{Proceedings of Machine Learning and Systems (MLSys)}, 2021.

\bibitem[Brent(1973)]{brent1973algorithms}
Richard~P. Brent.
\newblock \emph{Algorithms for Minimization without Derivatives}.
\newblock Prentice-Hall, 1973.

\bibitem[Chong and Forsyth(2020)]{chong2020effectively}
Min~Jin Chong and David Forsyth.
\newblock Effectively unbiased {FID} and inception score and where to find them.
\newblock In \emph{Proceedings of the IEEE/CVF Conference on Computer Vision and Pattern Recognition (CVPR)}, 2020.

\bibitem[Crane(2018)]{crane2018questionable}
Matt Crane.
\newblock Questionable answers in question answering research: Reproducibility and variability of published results.
\newblock \emph{Transactions of the Association for Computational Linguistics}, 2018.

\bibitem[D'Amour et~al.(2022)D'Amour, Heller, Moldovan, Adlam, Alipanahi, Beutel, Chen, Deaton, Eisenstein, Hoffman, Hormozdiari, Houlsby, Hou, Jerfel, Karthikesalingam, Lucic, Ma, McLean, Mincu, Mitani, Montanari, Nado, Natarajan, Nielson, Osborne, Raman, Ramasamy, Sayres, Schrouff, Seneviratne, Sequeira, Suresh, Veitch, Vladymyrov, Wang, Webster, Yadlowsky, Yun, Zhai, and Sculley]{damour2022underspecification}
Alexander D'Amour, Katherine Heller, Dan Moldovan, Ben Adlam, Babak Alipanahi, Alex Beutel, Christina Chen, Jonathan Deaton, Jacob Eisenstein, Matthew~D. Hoffman, Farhad Hormozdiari, Neil Houlsby, Shaobo Hou, Ghassen Jerfel, Alan Karthikesalingam, Mario Lucic, Yian Ma, Cory McLean, Diana Mincu, Akinori Mitani, Andrea Montanari, Zachary Nado, Vivek Natarajan, Christopher Nielson, Thomas~F. Osborne, Rajiv Raman, Kim Ramasamy, Rory Sayres, Jessica Schrouff, Martin Seneviratne, Shannon Sequeira, Harini Suresh, Victor Veitch, Max Vladymyrov, Xuezhi Wang, Kellie Webster, Steve Yadlowsky, Taedong Yun, Xiaohua Zhai, and D.~Sculley.
\newblock Underspecification presents challenges for credibility in modern machine learning.
\newblock \emph{Journal of Machine Learning Research}, 2022.

\bibitem[Degeorge et~al.(2025)Degeorge, Ghosh, Dufour, Picard, and Kalogeiton]{degeorge2025imagenet}
Lucas Degeorge, Arijit Ghosh, Nicolas Dufour, David Picard, and Vicky Kalogeiton.
\newblock How far can we go with {ImageNet} for text-to-image generation?
\newblock In \emph{Advances in Neural Information Processing Systems 38 (NeurIPS)}, 2025.

\bibitem[Dem{\v{s}}ar(2006)]{demsar2006comparisons}
Janez Dem{\v{s}}ar.
\newblock Statistical comparisons of classifiers over multiple data sets.
\newblock \emph{Journal of Machine Learning Research}, 2006.

\bibitem[Deng et~al.(2009)Deng, Dong, Socher, Li, Li, and Fei-Fei]{deng2009imagenet}
Jia Deng, Wei Dong, Richard Socher, Li-Jia Li, Kai Li, and Li~Fei-Fei.
\newblock {ImageNet}: A large-scale hierarchical image database.
\newblock In \emph{Proceedings of the IEEE Conference on Computer Vision and Pattern Recognition (CVPR)}, 2009.

\bibitem[Dhariwal and Nichol(2021)]{dhariwal2021adm}
Prafulla Dhariwal and Alexander~Quinn Nichol.
\newblock Diffusion models beat {GAN}s on image synthesis.
\newblock In \emph{Advances in Neural Information Processing Systems 34 (NeurIPS)}, 2021.

\bibitem[DiCiccio and Efron(1996)]{diciccio1996bootstrap}
Thomas~J. DiCiccio and Bradley Efron.
\newblock Bootstrap confidence intervals.
\newblock \emph{Statistical Science}, 1996.

\bibitem[Dietterich(1998)]{dietterich1998tests}
Thomas~G. Dietterich.
\newblock Approximate statistical tests for comparing supervised classification learning algorithms.
\newblock \emph{Neural Computation}, 1998.

\bibitem[Dodge et~al.(2019)Dodge, Gururangan, Card, Schwartz, and Smith]{dodge2019showyourwork}
Jesse Dodge, Suchin Gururangan, Dallas Card, Roy Schwartz, and Noah~A. Smith.
\newblock Show your work: Improved reporting of experimental results.
\newblock In \emph{Proceedings of EMNLP-IJCNLP}, 2019.

\bibitem[Dodge et~al.(2020)Dodge, Ilharco, Schwartz, Farhadi, Hajishirzi, and Smith]{dodge2020fine}
Jesse Dodge, Gabriel Ilharco, Roy Schwartz, Ali Farhadi, Hannaneh Hajishirzi, and Noah~A. Smith.
\newblock Fine-tuning pretrained language models: Weight initializations, data orders, and early stopping.
\newblock \emph{arXiv preprint arXiv:2002.06305}, 2020.

\bibitem[Dosovitskiy et~al.(2021)Dosovitskiy, Beyer, Kolesnikov, Weissenborn, Zhai, Unterthiner, Dehghani, Minderer, Heigold, Gelly, Uszkoreit, and Houlsby]{dosovitskiy2021vit}
Alexey Dosovitskiy, Lucas Beyer, Alexander Kolesnikov, Dirk Weissenborn, Xiaohua Zhai, Thomas Unterthiner, Mostafa Dehghani, Matthias Minderer, Georg Heigold, Sylvain Gelly, Jakob Uszkoreit, and Neil Houlsby.
\newblock An image is worth 16x16 words: Transformers for image recognition at scale.
\newblock In \emph{International Conference on Learning Representations (ICLR)}, 2021.

\bibitem[Dror et~al.(2018)Dror, Baumer, Shlomov, and Reichart]{dror2018hitchhiker}
Rotem Dror, Gili Baumer, Segev Shlomov, and Roi Reichart.
\newblock The hitchhiker's guide to testing statistical significance in natural language processing.
\newblock In \emph{Proceedings of the 56th Annual Meeting of the Association for Computational Linguistics (ACL), Volume 1: Long Papers}, 2018.

\bibitem[Dunn(1961)]{dunn1961multiple}
Olive~Jean Dunn.
\newblock Multiple comparisons among means.
\newblock \emph{Journal of the American Statistical Association}, 1961.

\bibitem[Efron(1979)]{efron1979bootstrap}
B.~Efron.
\newblock Bootstrap methods: Another look at the jackknife.
\newblock \emph{The Annals of Statistics}, 1979.

\bibitem[Esser et~al.(2024)Esser, Kulal, Blattmann, Entezari, M{\"u}ller, Saini, Levi, Lorenz, Sauer, Boesel, Podell, Dockhorn, English, and Rombach]{esser2024sd3}
Patrick Esser, Sumith Kulal, Andreas Blattmann, Rahim Entezari, Jonas M{\"u}ller, Harry Saini, Yam Levi, Dominik Lorenz, Axel Sauer, Frederic Boesel, Dustin Podell, Tim Dockhorn, Zion English, and Robin Rombach.
\newblock Scaling rectified flow transformers for high-resolution image synthesis.
\newblock In \emph{Proceedings of the 41st International Conference on Machine Learning (ICML)}, 2024.

\bibitem[Everingham et~al.(2014)Everingham, Eslami, Van~Gool, Williams, Winn, and Zisserman]{everingham2014pascal}
Mark Everingham, S.~M.~Ali Eslami, Luc Van~Gool, Christopher K.~I. Williams, John Winn, and Andrew Zisserman.
\newblock The {Pascal} visual object classes challenge: A retrospective.
\newblock \emph{International Journal of Computer Vision}, 2014.

\bibitem[Fort et~al.(2019)Fort, Hu, and Lakshminarayanan]{fort2019deepensembles}
Stanislav Fort, Huiyi Hu, and Balaji Lakshminarayanan.
\newblock Deep ensembles: A loss landscape perspective.
\newblock \emph{arXiv preprint arXiv:1912.02757}, 2019.

\bibitem[Frankle and Carbin(2019)]{frankle2019lottery}
Jonathan Frankle and Michael Carbin.
\newblock The lottery ticket hypothesis: Finding sparse, trainable neural networks.
\newblock In \emph{International Conference on Learning Representations (ICLR)}, 2019.

\bibitem[Frankle et~al.(2020)Frankle, Dziugaite, Roy, and Carbin]{frankle2020linear}
Jonathan Frankle, Gintare~Karolina Dziugaite, Daniel~M. Roy, and Michael Carbin.
\newblock Linear mode connectivity and the lottery ticket hypothesis.
\newblock In \emph{Proceedings of the 37th International Conference on Machine Learning (ICML)}, 2020.

\bibitem[Glorot and Bengio(2010)]{glorot2010xavier}
Xavier Glorot and Yoshua Bengio.
\newblock Understanding the difficulty of training deep feedforward neural networks.
\newblock In \emph{Proceedings of the Thirteenth International Conference on Artificial Intelligence and Statistics (AISTATS)}, 2010.

\bibitem[Gundersen and Kjensmo(2018)]{gundersen2018state}
Odd~Erik Gundersen and Sigbj{\o}rn Kjensmo.
\newblock State of the art: Reproducibility in artificial intelligence.
\newblock In \emph{Proceedings of the AAAI Conference on Artificial Intelligence}, 2018.

\bibitem[He and {Thinking Machines Lab}(2025)]{he2025nondeterminism}
Horace He and {Thinking Machines Lab}.
\newblock Defeating nondeterminism in {LLM} inference.
\newblock \url{https://thinkingmachines.ai/blog/defeating-nondeterminism-in-llm-inference/}, 2025.
\newblock Thinking Machines Lab blog; accessed 2026.

\bibitem[He et~al.(2015)He, Zhang, Ren, and Sun]{he2015kaiming}
Kaiming He, Xiangyu Zhang, Shaoqing Ren, and Jian Sun.
\newblock Delving deep into rectifiers: Surpassing human-level performance on {ImageNet} classification.
\newblock In \emph{Proceedings of the IEEE International Conference on Computer Vision (ICCV)}, 2015.

\bibitem[He et~al.(2016)He, Zhang, Ren, and Sun]{he2016resnet}
Kaiming He, Xiangyu Zhang, Shaoqing Ren, and Jian Sun.
\newblock Deep residual learning for image recognition.
\newblock In \emph{Proceedings of the IEEE Conference on Computer Vision and Pattern Recognition (CVPR)}, 2016.

\bibitem[Henderson et~al.(2018)Henderson, Islam, Bachman, Pineau, Precup, and Meger]{henderson2018deep}
Peter Henderson, Riashat Islam, Philip Bachman, Joelle Pineau, Doina Precup, and David Meger.
\newblock Deep reinforcement learning that matters.
\newblock In \emph{Proceedings of the AAAI Conference on Artificial Intelligence}, 2018.

\bibitem[Henighan et~al.(2020)Henighan, Kaplan, Katz, Chen, Hesse, Jackson, Jun, Brown, Dhariwal, Gray, Hallacy, Mann, Radford, Ramesh, Ryder, Ziegler, Schulman, Amodei, and McCandlish]{henighan2020scaling}
Tom Henighan, Jared Kaplan, Mor Katz, Mark Chen, Christopher Hesse, Jacob Jackson, Heewoo Jun, Tom~B. Brown, Prafulla Dhariwal, Scott Gray, Chris Hallacy, Benjamin Mann, Alec Radford, Aditya Ramesh, Nick Ryder, Daniel~M. Ziegler, John Schulman, Dario Amodei, and Sam McCandlish.
\newblock Scaling laws for autoregressive generative modeling.
\newblock \emph{arXiv preprint arXiv:2010.14701}, 2020.

\bibitem[Hestness et~al.(2017)Hestness, Narang, Ardalani, Diamos, Jun, Kianinejad, Patwary, Yang, and Zhou]{hestness2017predictable}
Joel Hestness, Sharan Narang, Newsha Ardalani, Gregory Diamos, Heewoo Jun, Hassan Kianinejad, Md. Mostofa~Ali Patwary, Yang Yang, and Yanqi Zhou.
\newblock Deep learning scaling is predictable, empirically.
\newblock \emph{arXiv preprint arXiv:1712.00409}, 2017.

\bibitem[Heusel et~al.(2017)Heusel, Ramsauer, Unterthiner, Nessler, and Hochreiter]{heusel2017gans}
Martin Heusel, Hubert Ramsauer, Thomas Unterthiner, Bernhard Nessler, and Sepp Hochreiter.
\newblock {GANs} trained by a two time-scale update rule converge to a local {N}ash equilibrium.
\newblock In \emph{Advances in Neural Information Processing Systems 30 (NIPS 2017)}, 2017.

\bibitem[Ho and Salimans(2022)]{ho2022cfg}
Jonathan Ho and Tim Salimans.
\newblock Classifier-free diffusion guidance.
\newblock \emph{arXiv preprint arXiv:2207.12598}, 2022.

\bibitem[Ho et~al.(2020)Ho, Jain, and Abbeel]{ho2020ddpm}
Jonathan Ho, Ajay Jain, and Pieter Abbeel.
\newblock Denoising diffusion probabilistic models.
\newblock In \emph{Advances in Neural Information Processing Systems 33 (NeurIPS)}, 2020.

\bibitem[Hoffmann et~al.(2022)Hoffmann, Borgeaud, Mensch, Buchatskaya, Cai, Rutherford, de~Las~Casas, Hendricks, Welbl, Clark, Hennigan, Noland, Millican, van~den Driessche, Damoc, Guy, Osindero, Simonyan, Elsen, Rae, Vinyals, and Sifre]{hoffmann2022chinchilla}
Jordan Hoffmann, Sebastian Borgeaud, Arthur Mensch, Elena Buchatskaya, Trevor Cai, Eliza Rutherford, Diego de~Las~Casas, Lisa~Anne Hendricks, Johannes Welbl, Aidan Clark, Tom Hennigan, Eric Noland, Katie Millican, George van~den Driessche, Bogdan Damoc, Aurelia Guy, Simon Osindero, Karen Simonyan, Erich Elsen, Jack~W. Rae, Oriol Vinyals, and Laurent Sifre.
\newblock Training compute-optimal large language models.
\newblock In \emph{Advances in Neural Information Processing Systems 35 (NeurIPS)}, 2022.

\bibitem[Jayasumana et~al.(2024)Jayasumana, Ramalingam, Veit, Glasner, Chakrabarti, and Kumar]{jayasumana2024rethinking}
Sadeep Jayasumana, Srikumar Ramalingam, Andreas Veit, Daniel Glasner, Ayan Chakrabarti, and Sanjiv Kumar.
\newblock Rethinking {FID}: Towards a better evaluation metric for image generation.
\newblock In \emph{Proceedings of the IEEE/CVF Conference on Computer Vision and Pattern Recognition (CVPR)}, 2024.

\bibitem[Jordan(2023)]{jordan2023calibrated}
Keller Jordan.
\newblock Calibrated chaos: Variance between runs of neural network training is harmless and inevitable.
\newblock \emph{arXiv preprint arXiv:2304.01910}, 2023.

\bibitem[Kadkhodaie et~al.(2024)Kadkhodaie, Guth, Simoncelli, and Mallat]{kadkhodaie2024generalization}
Zahra Kadkhodaie, Florentin Guth, Eero~P. Simoncelli, and St{\'e}phane Mallat.
\newblock Generalization in diffusion models arises from geometry-adaptive harmonic representations.
\newblock In \emph{The Twelfth International Conference on Learning Representations (ICLR)}, 2024.

\bibitem[Kaplan et~al.(2020)Kaplan, McCandlish, Henighan, Brown, Chess, Child, Gray, Radford, Wu, and Amodei]{kaplan2020scaling}
Jared Kaplan, Sam McCandlish, Tom Henighan, Tom~B. Brown, Benjamin Chess, Rewon Child, Scott Gray, Alec Radford, Jeffrey Wu, and Dario Amodei.
\newblock Scaling laws for neural language models.
\newblock \emph{arXiv preprint arXiv:2001.08361}, 2020.

\bibitem[Karras et~al.(2019)Karras, Laine, and Aila]{karras2019stylegan}
Tero Karras, Samuli Laine, and Timo Aila.
\newblock A style-based generator architecture for generative adversarial networks.
\newblock In \emph{Proceedings of the IEEE/CVF Conference on Computer Vision and Pattern Recognition (CVPR)}, 2019.

\bibitem[Karras et~al.(2022)Karras, Aittala, Aila, and Laine]{karras2022edm}
Tero Karras, Miika Aittala, Timo Aila, and Samuli Laine.
\newblock Elucidating the design space of diffusion-based generative models.
\newblock In \emph{Advances in Neural Information Processing Systems 35 (NeurIPS)}, 2022.

\bibitem[Karras et~al.(2024)Karras, Aittala, Lehtinen, Hellsten, Aila, and Laine]{karras2024edm2}
Tero Karras, Miika Aittala, Jaakko Lehtinen, Janne Hellsten, Timo Aila, and Samuli Laine.
\newblock Analyzing and improving the training dynamics of diffusion models.
\newblock In \emph{Proceedings of the IEEE/CVF Conference on Computer Vision and Pattern Recognition (CVPR)}, 2024.

\bibitem[Kendall and Babington~Smith(1939)]{kendall1939w}
M.~G. Kendall and B.~Babington~Smith.
\newblock The problem of $m$ rankings.
\newblock \emph{The Annals of Mathematical Statistics}, 1939.

\bibitem[Kiefer(1953)]{kiefer1953sequential}
J.~Kiefer.
\newblock Sequential minimax search for a maximum.
\newblock \emph{Proceedings of the American Mathematical Society}, 1953.

\bibitem[Kingma and Ba(2015)]{kingma2015adam}
Diederik~P. Kingma and Jimmy Ba.
\newblock {Adam}: A method for stochastic optimization.
\newblock In \emph{3rd International Conference on Learning Representations (ICLR)}, 2015.

\bibitem[Kingma et~al.(2021)Kingma, Salimans, Poole, and Ho]{kingma2021vdm}
Diederik~P. Kingma, Tim Salimans, Ben Poole, and Jonathan Ho.
\newblock Variational diffusion models.
\newblock In \emph{Advances in Neural Information Processing Systems 34 (NeurIPS)}, 2021.

\bibitem[Kynk{\"a}{\"a}nniemi et~al.(2019)Kynk{\"a}{\"a}nniemi, Karras, Laine, Lehtinen, and Aila]{kynkaanniemi2019improved}
Tuomas Kynk{\"a}{\"a}nniemi, Tero Karras, Samuli Laine, Jaakko Lehtinen, and Timo Aila.
\newblock Improved precision and recall metric for assessing generative models.
\newblock In \emph{Advances in Neural Information Processing Systems 32 (NeurIPS)}, 2019.

\bibitem[Kynk{\"a}{\"a}nniemi et~al.(2023)Kynk{\"a}{\"a}nniemi, Karras, Aittala, Aila, and Lehtinen]{kynkaanniemi2023role}
Tuomas Kynk{\"a}{\"a}nniemi, Tero Karras, Miika Aittala, Timo Aila, and Jaakko Lehtinen.
\newblock The role of {I}mage{N}et classes in {F}r{\'e}chet inception distance.
\newblock In \emph{The Eleventh International Conference on Learning Representations (ICLR)}, 2023.

\bibitem[LaLoudouana et~al.(2003)LaLoudouana, Tarare, Center, and Selacie]{laloudouana2003data}
Doudou LaLoudouana, Mambobo~Bonouliqui Tarare, Lupano~Tecallonou Center, and GUANA Selacie.
\newblock Data set selection.
\newblock \emph{Journal of Machine Learning Gossip}, 2003.

\bibitem[Li et~al.(2024)Li, Tian, Li, Deng, and He]{li2024mar}
Tianhong Li, Yonglong Tian, He~Li, Mingyang Deng, and Kaiming He.
\newblock Autoregressive image generation without vector quantization.
\newblock In \emph{Advances in Neural Information Processing Systems 37 (NeurIPS)}, 2024.

\bibitem[Liang et~al.(2024)Liang, He, Yang, and Dai]{liang2024scalingdit}
Zhengyang Liang, Hao He, Ceyuan Yang, and Bo~Dai.
\newblock Scaling laws for diffusion transformers.
\newblock \emph{arXiv preprint arXiv:2410.08184}, 2024.

\bibitem[Lipman et~al.(2023)Lipman, Chen, Ben-Hamu, Nickel, and Le]{lipman2023flowmatching}
Yaron Lipman, Ricky T.~Q. Chen, Heli Ben-Hamu, Maximilian Nickel, and Matt Le.
\newblock Flow matching for generative modeling.
\newblock In \emph{The Eleventh International Conference on Learning Representations (ICLR)}, 2023.

\bibitem[Liu et~al.(2023)Liu, Gong, and Liu]{liu2023rectified}
Xingchao Liu, Chengyue Gong, and Qiang Liu.
\newblock Flow straight and fast: Learning to generate and transfer data with rectified flow.
\newblock In \emph{The Eleventh International Conference on Learning Representations (ICLR)}, 2023.

\bibitem[Lucic et~al.(2018)Lucic, Kurach, Michalski, Gelly, and Bousquet]{lucic2018gans}
Mario Lucic, Karol Kurach, Marcin Michalski, Sylvain Gelly, and Olivier Bousquet.
\newblock Are {GAN}s created equal? {A} large-scale study.
\newblock In \emph{Advances in Neural Information Processing Systems 31 (NeurIPS)}, 2018.

\bibitem[Ma et~al.(2024)Ma, Goldstein, Albergo, Boffi, Vanden-Eijnden, and Xie]{ma2024sit}
Nanye Ma, Mark Goldstein, Michael~S. Albergo, Nicholas~M. Boffi, Eric Vanden-Eijnden, and Saining Xie.
\newblock {SiT}: Exploring flow and diffusion-based generative models with scalable interpolant transformers.
\newblock In \emph{Computer Vision -- ECCV 2024}, 2024.

\bibitem[Madhyastha and Jain(2019)]{madhyastha2019stability}
Pranava Madhyastha and Rishabh Jain.
\newblock On model stability as a function of random seed.
\newblock In \emph{Proceedings of the 23rd Conference on Computational Natural Language Learning (CoNLL)}, 2019.

\bibitem[Melis et~al.(2018)Melis, Dyer, and Blunsom]{melis2018state}
G{\'a}bor Melis, Chris Dyer, and Phil Blunsom.
\newblock On the state of the art of evaluation in neural language models.
\newblock In \emph{International Conference on Learning Representations (ICLR)}, 2018.

\bibitem[Morozov et~al.(2021)Morozov, Voynov, and Babenko]{morozov2021self}
Stanislav Morozov, Andrey Voynov, and Artem Babenko.
\newblock On self-supervised image representations for {GAN} evaluation.
\newblock In \emph{9th International Conference on Learning Representations (ICLR)}, 2021.

\bibitem[Mosbach et~al.(2021)Mosbach, Andriushchenko, and Klakow]{mosbach2021stability}
Marius Mosbach, Maksym Andriushchenko, and Dietrich Klakow.
\newblock On the stability of fine-tuning {BERT}: Misconceptions, explanations, and strong baselines.
\newblock In \emph{International Conference on Learning Representations (ICLR)}, 2021.

\bibitem[Naeem et~al.(2020)Naeem, Oh, Uh, Choi, and Yoo]{naeem2020reliable}
Muhammad~Ferjad Naeem, Seong~Joon Oh, Youngjung Uh, Yunjey Choi, and Jaejun Yoo.
\newblock Reliable fidelity and diversity metrics for generative models.
\newblock In \emph{Proceedings of the 37th International Conference on Machine Learning}, 2020.

\bibitem[Nagarajan and Kolter(2019)]{nagarajan2019uniform}
Vaishnavh Nagarajan and J.~Zico Kolter.
\newblock Uniform convergence may be unable to explain generalization in deep learning.
\newblock In \emph{Advances in Neural Information Processing Systems 32 (NeurIPS)}, 2019.

\bibitem[Nichol and Dhariwal(2021)]{nichol2021iddpm}
Alexander~Quinn Nichol and Prafulla Dhariwal.
\newblock Improved denoising diffusion probabilistic models.
\newblock In \emph{Proceedings of the 38th International Conference on Machine Learning (ICML)}, 2021.

\bibitem[Oquab et~al.(2024)Oquab, Darcet, Moutakanni, Vo, Szafraniec, Khalidov, Fernandez, Haziza, Massa, El-Nouby, Assran, Ballas, Galuba, Howes, Huang, Li, Misra, Rabbat, Sharma, Synnaeve, Xu, J{\'e}gou, Mairal, Labatut, Joulin, and Bojanowski]{oquab2024dinov2}
Maxime Oquab, Timoth{\'e}e Darcet, Th{\'e}o Moutakanni, Huy~V. Vo, Marc Szafraniec, Vasil Khalidov, Pierre Fernandez, Daniel Haziza, Francisco Massa, Alaaeldin El-Nouby, Mahmoud Assran, Nicolas Ballas, Wojciech Galuba, Russell Howes, Po-Yao Huang, Shang-Wen Li, Ishan Misra, Michael Rabbat, Vasu Sharma, Gabriel Synnaeve, Hu~Xu, Herv{\'e} J{\'e}gou, Julien Mairal, Patrick Labatut, Armand Joulin, and Piotr Bojanowski.
\newblock {DINOv2}: Learning robust visual features without supervision.
\newblock \emph{Transactions on Machine Learning Research}, 2024.

\bibitem[Parmar et~al.(2022)Parmar, Zhang, and Zhu]{parmar2022aliased}
Gaurav Parmar, Richard Zhang, and Jun-Yan Zhu.
\newblock On aliased resizing and surprising subtleties in {GAN} evaluation.
\newblock In \emph{Proceedings of the IEEE/CVF Conference on Computer Vision and Pattern Recognition (CVPR)}, 2022.

\bibitem[Peebles and Xie(2023)]{peebles2023dit}
William Peebles and Saining Xie.
\newblock Scalable diffusion models with transformers.
\newblock In \emph{Proceedings of the IEEE/CVF International Conference on Computer Vision (ICCV)}, 2023.

\bibitem[Pham et~al.(2020)Pham, Qian, Wang, Lutellier, Rosenthal, Tan, Yu, and Nagappan]{pham2020problems}
Hung~Viet Pham, Shangshu Qian, Jiannan Wang, Thibaud Lutellier, Jonathan Rosenthal, Lin Tan, Yaoliang Yu, and Nachiappan Nagappan.
\newblock Problems and opportunities in training deep learning software systems: An analysis of variance.
\newblock In \emph{Proceedings of the 35th IEEE/ACM International Conference on Automated Software Engineering (ASE)}, 2020.

\bibitem[Picard(2021)]{picard2021torch}
David Picard.
\newblock Torch.manual\_seed(3407) is all you need: On the influence of random seeds in deep learning architectures for computer vision.
\newblock \emph{arXiv preprint arXiv:2109.08203}, 2021.

\bibitem[Pineau et~al.(2021)Pineau, Vincent-Lamarre, Sinha, Larivi{\`e}re, Beygelzimer, d'Alch{\'e} Buc, Fox, and Larochelle]{pineau2021improving}
Joelle Pineau, Philippe Vincent-Lamarre, Koustuv Sinha, Vincent Larivi{\`e}re, Alina Beygelzimer, Florence d'Alch{\'e} Buc, Emily Fox, and Hugo Larochelle.
\newblock Improving reproducibility in machine learning research (a report from the {NeurIPS} 2019 reproducibility program).
\newblock \emph{Journal of Machine Learning Research}, 2021.

\bibitem[Raff(2019)]{raff2019step}
Edward Raff.
\newblock A step toward quantifying independently reproducible machine learning research.
\newblock In \emph{Advances in Neural Information Processing Systems 32 (NeurIPS)}, 2019.

\bibitem[Recht et~al.(2019)Recht, Roelofs, Schmidt, and Shankar]{recht2019imagenet}
Benjamin Recht, Rebecca Roelofs, Ludwig Schmidt, and Vaishaal Shankar.
\newblock Do {ImageNet} classifiers generalize to {ImageNet}?
\newblock In \emph{Proceedings of the 36th International Conference on Machine Learning (ICML)}, 2019.

\bibitem[Reimers and Gurevych(2017)]{reimers2017reporting}
Nils Reimers and Iryna Gurevych.
\newblock Reporting score distributions makes a difference: Performance study of {LSTM}-networks for sequence tagging.
\newblock In \emph{Proceedings of the 2017 Conference on Empirical Methods in Natural Language Processing (EMNLP)}, 2017.

\bibitem[Rombach et~al.(2022)Rombach, Blattmann, Lorenz, Esser, and Ommer]{rombach2022ldm}
Robin Rombach, Andreas Blattmann, Dominik Lorenz, Patrick Esser, and Bj{\"o}rn Ommer.
\newblock High-resolution image synthesis with latent diffusion models.
\newblock In \emph{Proceedings of the IEEE/CVF Conference on Computer Vision and Pattern Recognition (CVPR)}, 2022.

\bibitem[Sajjadi et~al.(2018)Sajjadi, Bachem, Lucic, Bousquet, and Gelly]{sajjadi2018assessing}
Mehdi S.~M. Sajjadi, Olivier Bachem, Mario Lucic, Olivier Bousquet, and Sylvain Gelly.
\newblock Assessing generative models via precision and recall.
\newblock In \emph{Advances in Neural Information Processing Systems 31 (NeurIPS)}, 2018.

\bibitem[Salimans et~al.(2016)Salimans, Goodfellow, Zaremba, Cheung, Radford, and Chen]{salimans2016improved}
Tim Salimans, Ian Goodfellow, Wojciech Zaremba, Vicki Cheung, Alec Radford, and Xi~Chen.
\newblock Improved techniques for training {GANs}.
\newblock In \emph{Advances in Neural Information Processing Systems 29 (NIPS 2016)}, 2016.

\bibitem[Schuhmann et~al.(2022)Schuhmann, Beaumont, Vencu, Gordon, Wightman, Cherti, Coombes, Katta, Mullis, Wortsman, Schramowski, Kundurthy, Crowson, Schmidt, Kaczmarczyk, and Jitsev]{schuhmann2022laion5b}
Christoph Schuhmann, Romain Beaumont, Richard Vencu, Cade Gordon, Ross Wightman, Mehdi Cherti, Theo Coombes, Aarush Katta, Clayton Mullis, Mitchell Wortsman, Patrick Schramowski, Srivatsa Kundurthy, Katherine Crowson, Ludwig Schmidt, Robert Kaczmarczyk, and Jenia Jitsev.
\newblock {LAION-5B}: An open large-scale dataset for training next generation image-text models.
\newblock In \emph{Advances in Neural Information Processing Systems 35 (NeurIPS), Datasets and Benchmarks Track}, 2022.

\bibitem[Sculley et~al.(2018)Sculley, Snoek, Wiltschko, and Rahimi]{sculley2018winners}
D.~Sculley, Jasper Snoek, Alexander~B. Wiltschko, and Ali Rahimi.
\newblock Winner's curse? on pace, progress, and empirical rigor.
\newblock In \emph{6th International Conference on Learning Representations (ICLR), Workshop Track}, 2018.

\bibitem[Sohl-Dickstein et~al.(2015)Sohl-Dickstein, Weiss, Maheswaranathan, and Ganguli]{sohldickstein2015thermo}
Jascha Sohl-Dickstein, Eric~A. Weiss, Niru Maheswaranathan, and Surya Ganguli.
\newblock Deep unsupervised learning using nonequilibrium thermodynamics.
\newblock In \emph{Proceedings of the 32nd International Conference on Machine Learning (ICML)}, 2015.

\bibitem[Song et~al.(2021{\natexlab{a}})Song, Meng, and Ermon]{song2021ddim}
Jiaming Song, Chenlin Meng, and Stefano Ermon.
\newblock Denoising diffusion implicit models.
\newblock In \emph{9th International Conference on Learning Representations (ICLR)}, 2021{\natexlab{a}}.

\bibitem[Song and Ermon(2019)]{song2019ncsn}
Yang Song and Stefano Ermon.
\newblock Generative modeling by estimating gradients of the data distribution.
\newblock In \emph{Advances in Neural Information Processing Systems 32 (NeurIPS)}, 2019.

\bibitem[Song et~al.(2021{\natexlab{b}})Song, Sohl-Dickstein, Kingma, Kumar, Ermon, and Poole]{song2021scoresde}
Yang Song, Jascha Sohl-Dickstein, Diederik~P. Kingma, Abhishek Kumar, Stefano Ermon, and Ben Poole.
\newblock Score-based generative modeling through stochastic differential equations.
\newblock In \emph{9th International Conference on Learning Representations (ICLR)}, 2021{\natexlab{b}}.

\bibitem[Spearman(1904)]{spearman1904rank}
C.~Spearman.
\newblock The proof and measurement of association between two things.
\newblock \emph{The American Journal of Psychology}, 1904.

\bibitem[Stein et~al.(2023)Stein, Cresswell, Hosseinzadeh, Sui, Ross, Villecroze, Liu, Caterini, Taylor, and Loaiza-Ganem]{stein2023exposing}
George Stein, Jesse~C. Cresswell, Rasa Hosseinzadeh, Yi~Sui, Brendan~Leigh Ross, Valentin Villecroze, Zhaoyan Liu, Anthony~L. Caterini, J.~Eric~T. Taylor, and Gabriel Loaiza-Ganem.
\newblock Exposing flaws of generative model evaluation metrics and their unfair treatment of diffusion models.
\newblock In \emph{Advances in Neural Information Processing Systems 36 (NeurIPS)}, 2023.

\bibitem[Summers and Dinneen(2021)]{summers2021nondeterminism}
Cecilia Summers and Michael~J. Dinneen.
\newblock Nondeterminism and instability in neural network optimization.
\newblock In \emph{Proceedings of the 38th International Conference on Machine Learning (ICML)}, 2021.

\bibitem[Szegedy et~al.(2016)Szegedy, Vanhoucke, Ioffe, Shlens, and Wojna]{szegedy2016inceptionv3}
Christian Szegedy, Vincent Vanhoucke, Sergey Ioffe, Jon Shlens, and Zbigniew Wojna.
\newblock Rethinking the {I}nception architecture for computer vision.
\newblock In \emph{Proceedings of the IEEE Conference on Computer Vision and Pattern Recognition (CVPR)}, 2016.

\bibitem[Vaswani et~al.(2017)Vaswani, Shazeer, Parmar, Uszkoreit, Jones, Gomez, Kaiser, and Polosukhin]{vaswani2017transformer}
Ashish Vaswani, Noam Shazeer, Niki Parmar, Jakob Uszkoreit, Llion Jones, Aidan~N. Gomez, {\L}ukasz Kaiser, and Illia Polosukhin.
\newblock Attention is all you need.
\newblock In \emph{Advances in Neural Information Processing Systems 30 (NeurIPS)}, 2017.

\bibitem[Welch(1947)]{welch1947ttest}
B.~L. Welch.
\newblock The generalization of `{S}tudent's' problem when several different population variances are involved.
\newblock \emph{Biometrika}, 1947.

\bibitem[Wenzel et~al.(2020)Wenzel, Snoek, Tran, and Jenatton]{wenzel2020hyperensembles}
Florian Wenzel, Jasper Snoek, Dustin Tran, and Rodolphe Jenatton.
\newblock Hyperparameter ensembles for robustness and uncertainty quantification.
\newblock In \emph{Advances in Neural Information Processing Systems 33 (NeurIPS)}, 2020.

\bibitem[Wu et~al.(2025)Wu, Liu, Yilmaz, Konermann, Walter, and Stegmaier]{wu2025pragmatic}
Yuli Wu, Fucheng Liu, R{\"u}veyda Yilmaz, Henning Konermann, Peter Walter, and Johannes Stegmaier.
\newblock A pragmatic note on evaluating generative models with {Fr\'echet} inception distance for retinal image synthesis.
\newblock \emph{arXiv preprint arXiv:2502.17160}, 2025.

\bibitem[Xu et~al.(2024)Xu, Zhang, and Shi]{xu2024goodseed}
Katherine Xu, Lingzhi Zhang, and Jianbo Shi.
\newblock Good seed makes a good crop: Discovering secret seeds in text-to-image diffusion models.
\newblock \emph{arXiv preprint arXiv:2405.14828}, 2024.

\bibitem[Yang et~al.(2022)Yang, Hu, Babuschkin, Sidor, Liu, Farhi, Ryder, Pachocki, Chen, and Gao]{yang2022mup}
Greg Yang, Edward~J. Hu, Igor Babuschkin, Szymon Sidor, Xiaodong Liu, David Farhi, Nick Ryder, Jakub Pachocki, Weizhu Chen, and Jianfeng Gao.
\newblock Tensor programs {V}: Tuning large neural networks via zero-shot hyperparameter transfer.
\newblock In \emph{Advances in Neural Information Processing Systems}, 2022.

\bibitem[Zhang et~al.(2024)Zhang, Zhou, Lu, Guo, Wang, Shen, and Qu]{zhang2024emergence}
Huijie Zhang, Jinfan Zhou, Yifu Lu, Minzhe Guo, Peng Wang, Liyue Shen, and Qing Qu.
\newblock The emergence of reproducibility and consistency in diffusion models.
\newblock In \emph{Proceedings of the 41st International Conference on Machine Learning (ICML)}, 2024.

\end{thebibliography}


\newpage
\appendix

\begin{center}
  {\LARGE\bfseries Supplementary Material}
\end{center}
\vspace{1em}

\section{Overview}
\label{app:overview}

\noindent\emph{The supplementary material extends the main paper along
three axes: (i) further empirical analyses of the seed lottery on
Inception FID, (ii) a theoretical justification of the golden-section
search, and (iii) a metric-robustness stress test that replicates every
main-paper claim under DINOv2 FID and the four Inception PRDC metrics.}

\myparagraph{What the stress test changes (and what it does not).}
The cross-metric replication of \autoref{app:supp-metrics} sharpens
rather than softens the main paper. The training lottery still dominates
the generation lottery on every fidelity-axis metric, the
\emph{noise > init > data} hierarchy of \autoref{sec:results-decomposition}
holds verbatim, and the lucky-seed speedup of \autoref{sec:results-speedup}
\emph{grows} to $2$--$3{\times}$ on SiT-L/XL once the benchmark moves to
precision, density or coverage. Inception recall is the lone metric that
inverts the asymmetry: \autoref{app:supp-overview} traces the inversion to
the only PRDC metric whose $k$-NN balls live on the per-evaluation
generated set, and is therefore a property of recall's estimator rather
than a counterexample to the seed-lottery story.

\myparagraph{Roadmap.}
\begin{itemize}
\item \autoref{app:setup-detail} -- per-experiment configurations
  ($N$, $K$, training-step budget, deviations from the shared protocol)
  for every subsection of \autoref{sec:results}.
\item \autoref{app:summary-tables} -- companion tables consolidating the
  headline numbers behind \autoref{sec:results-decomposition},
  \autoref{sec:results-scaling} and \autoref{sec:results-mup} for
  one-stop lookup.
\item \autoref{app:extra-analyses} -- five additional empirical analyses
  on Inception FID: panel-by-panel violins for the source decomposition
  (\autoref{app:disentangle-violins}), rank instability across summary
  statistics (\autoref{sec:results-training}), a $10\!\times\!15$
  factorial test of init-seed universality (\autoref{sec:results-optimality}),
  per-bracket numbers for the $\mu$P sweep (\autoref{app:mup-extras}),
  and a practitioner-facing $\mathrm{FID}\!\to\!\mathrm{CI}$ lookup
  (\autoref{app:fid-to-ci}).
\item \autoref{app:golden-section} -- convergence of \autoref{alg:gss},
  unimodality of $\mathrm{FID}(\mathrm{CFG})$ under a Gaussian feature
  model, and noise sensitivity of the returned optimum.
\item \autoref{app:supp-metrics} -- DINOv2 FID and Inception PRDC
  replication of every \autoref{sec:results} subsection, row-for-row
  with the main-paper tables.
\item \autoref{app:fid-visualization} -- per-class galleries of
  generated samples ordered by FID, both with and without
  classifier-free guidance, that visualise how perceptual quality
  tracks (or fails to track) the FID gradient.
\end{itemize}

\myparagraph{The two lotteries as a slot machine.}
\autoref{fig:slot-machine} recasts the randomness pipeline of
\autoref{fig:lottery-machine} as a pair of slot machines and pairs it with
the measured SiT-B/2 panel, making the $3.2{\times}$ training-vs-generation
asymmetry literal: the left machine spins three reels (initialisation, data
order, per-step noise) to yield one network, the right machine spins ten
sampling seeds to score ten FIDs of that network.

\definecolor{cFelt}    {HTML}{1F6147}  
\definecolor{cFeltDk}  {HTML}{124B36}  
\definecolor{cGold}    {HTML}{D9B65C}  
\definecolor{cGoldDk}  {HTML}{9C7C2F}  
\definecolor{cGoldHi}  {HTML}{F0DC9A}  
\definecolor{cCream}   {HTML}{FBF4DD}  
\definecolor{cTan}     {HTML}{E7D3A2}  
\definecolor{cInk}     {HTML}{241F17}  
\definecolor{cClaret}  {HTML}{9A2B22}  

\newcommand{\fldie}[4]{
  \begin{scope}[opacity=#4]
  \filldraw[fill=cClaret, draw=cClaret!65, line width=0.7pt, rounded corners=1.8pt]
    (#1-0.31,#2-0.31) rectangle (#1+0.31,#2+0.31);
  \ifcase#3\relax
    \def\pips{0/0}%
  \or \def\pips{0/0}%
  \or \def\pips{-0.165/0.165,0.165/-0.165}%
  \or \def\pips{-0.165/0.165,0/0,0.165/-0.165}%
  \or \def\pips{-0.165/0.165,0.165/0.165,-0.165/-0.165,0.165/-0.165}%
  \or \def\pips{-0.165/0.165,0.165/0.165,0/0,-0.165/-0.165,0.165/-0.165}%
  \or \def\pips{-0.165/0.165,0.165/0.165,-0.165/0,0.165/0,-0.165/-0.165,0.165/-0.165}%
  \fi
  \foreach \px/\py in \pips {\fill[cCream] (#1+\px,#2+\py) circle (0.046);}
  \end{scope}
}
\newcommand{\flspade}{\textcolor{cInk}{$\spadesuit$}}
\newcommand{\flheart}{\textcolor{cClaret}{$\heartsuit$}}
\newcommand{\fldiamond}{\textcolor{cClaret}{$\diamondsuit$}}
\newcommand{\flcard}[4]{
  \begin{scope}[rotate around={#3:(#1,#2-0.22)}]
    \filldraw[fill=cCream, draw=cInk, line width=0.5pt, rounded corners=1.3pt]
      (#1-0.14,#2-0.22) rectangle (#1+0.14,#2+0.24);
    \node[font=\scriptsize] at (#1,#2+0.105) {#4};
  \end{scope}}
\newcommand{\fldeck}[6]{
  \begin{scope}[opacity=#3]
    \flcard{#1}{#2}{-23}{#4}\flcard{#1}{#2}{0}{#5}\flcard{#1}{#2}{23}{#6}
  \end{scope}}
\newcommand{\flnoise}[4]{
  \begin{scope}[opacity=#4]
  \ifcase#3\relax
  \or 
    \draw[cClaret, line width=0.9pt, line cap=round, line join=round]
    (#1-0.45,#2)--(#1-0.386,#2+0.19)--(#1-0.321,#2-0.14)--(#1-0.257,#2+0.23)--(#1-0.193,#2-0.09)--(#1-0.129,#2+0.12)--(#1-0.064,#2-0.21)--(#1,#2+0.07)--(#1+0.064,#2+0.18)--(#1+0.129,#2-0.24)--(#1+0.193,#2+0.10)--(#1+0.257,#2-0.16)--(#1+0.321,#2+0.20)--(#1+0.386,#2-0.11)--(#1+0.45,#2);
  \or 
    \draw[cClaret, line width=0.9pt, line cap=round, line join=round]
    (#1-0.45,#2)--(#1-0.40,#2+0.10)--(#1-0.35,#2-0.13)--(#1-0.30,#2+0.06)--(#1-0.25,#2-0.09)--(#1-0.20,#2+0.14)--(#1-0.15,#2-0.05)--(#1-0.10,#2+0.11)--(#1-0.05,#2-0.15)--(#1,#2+0.04)--(#1+0.05,#2+0.12)--(#1+0.10,#2-0.08)--(#1+0.15,#2+0.15)--(#1+0.20,#2-0.11)--(#1+0.25,#2+0.07)--(#1+0.30,#2-0.14)--(#1+0.35,#2+0.09)--(#1+0.40,#2-0.06)--(#1+0.45,#2);
  \or 
    \draw[cClaret, line width=0.9pt, line cap=round, line join=round]
    (#1-0.45,#2)--(#1-0.36,#2+0.25)--(#1-0.27,#2-0.20)--(#1-0.18,#2+0.28)--(#1-0.09,#2-0.27)--(#1,#2+0.15)--(#1+0.09,#2-0.24)--(#1+0.18,#2+0.22)--(#1+0.27,#2-0.18)--(#1+0.36,#2+0.12)--(#1+0.45,#2);
  \fi
  \end{scope}
}

\newcommand{\fldrum}[1]{
  \shade[top color=cTan, middle color=cCream, bottom color=cTan]
    (#1-0.55,3.25) rectangle (#1+0.55,5.45);
  \foreach \dy in {0.66,0.90}{
    \draw[cInk!10,line width=0.35pt] (#1-0.5,4.35+\dy)--(#1+0.5,4.35+\dy);
    \draw[cInk!10,line width=0.35pt] (#1-0.5,4.35-\dy)--(#1+0.5,4.35-\dy);}
  \draw[cClaret!55,line width=0.8pt] (#1-0.5,4.35)--(#1+0.5,4.35);
}
\newcommand{\flreelframe}[1]{
  \draw[rounded corners=2pt, cGoldDk, line width=0.9pt]
    (#1-0.55,3.25) rectangle (#1+0.55,5.45);
}

\newcommand{\flgenreel}[4]{
  \ifnum#4=1
    \fill[cGoldHi!55, rounded corners=4pt] (#1-0.60,#2-0.56) rectangle (#1+0.60,#2+0.56);
  \fi
  \begin{scope}
    \clip[rounded corners=1.5pt] (#1-0.49,#2-0.46) rectangle (#1+0.49,#2+0.46);
    \ifnum#4=1
      \shade[top color=cGold, middle color=cGoldHi, bottom color=cGold]
        (#1-0.49,#2-0.46) rectangle (#1+0.49,#2+0.46);
    \else
      \shade[top color=cTan, middle color=cCream, bottom color=cTan]
        (#1-0.49,#2-0.46) rectangle (#1+0.49,#2+0.46);
    \fi
    \draw[cInk!12,line width=0.35pt] (#1-0.47,#2+0.28)--(#1+0.47,#2+0.28);
    \draw[cInk!12,line width=0.35pt] (#1-0.47,#2-0.28)--(#1+0.47,#2-0.28);
    \ifnum#4=1
      \draw[cClaret!75,line width=0.8pt] (#1-0.47,#2)--(#1+0.47,#2);
      \node[font=\scriptsize\ttfamily\bfseries, cInk] at (#1,#2) {#3};
    \else
      \draw[palSlateDark!85,line width=0.7pt] (#1-0.47,#2)--(#1+0.47,#2);
      \node[font=\scriptsize\ttfamily, cInk!88] at (#1,#2) {#3};
    \fi
  \end{scope}
  \ifnum#4=1
    \draw[rounded corners=1.5pt, cGoldDk, line width=1pt]
      (#1-0.49,#2-0.46) rectangle (#1+0.49,#2+0.46);
    \node[font=\fltitle\tiny\bfseries, cGoldHi, fill=cClaret, rounded corners=1.5pt,
          inner sep=2pt, anchor=south] at (#1,#2+0.52) {$\bigstar$ JACKPOT};
  \else
    \draw[rounded corners=1.5pt, cGoldDk!75, line width=0.6pt]
      (#1-0.49,#2-0.46) rectangle (#1+0.49,#2+0.46);
  \fi
}

\newcommand{\flcabinet}[4]{%
  \draw[rounded corners=5pt, fill=cFelt, draw=cGoldDk, line width=1.8pt] (#1,#2) rectangle (#3,#4);
  \draw[rounded corners=4pt, draw=cGold, line width=0.6pt] (#1+0.11,#2+0.11) rectangle (#3-0.11,#4-0.11);
}


\begin{figure}[t]
  \centering
  \tikzsetnextfilename{output-slotmachine}%
  \resizebox{0.99\linewidth}{!}{%
  \begin{tikzpicture}
    \useasboundingbox (0.05,2.12) rectangle (13.65,6.66);

    \flcabinet{0.18}{2.22}{4.55}{6.58}
    \draw[rounded corners=2.5pt, fill=cFeltDk, draw=cGold, line width=0.8pt]
      (0.34,5.90) rectangle (4.42,6.46);
    \node[font=\fltitle\small\bfseries, cGoldHi] at (2.38,6.18) {TRAINING LOTTERY};
    \node[font=\flsf\scriptsize, cCream] at (2.38,5.66)
      {pull once $\rightarrow$ train a network};
    \begin{scope}
      \clip[rounded corners=2pt] (0.475,3.25) rectangle (1.575,5.45);
      \fldrum{1.025}
      \fldie{1.025}{5.30}{2}{0.7} \fldie{1.025}{4.35}{5}{1} \fldie{1.025}{3.40}{3}{0.7}
    \end{scope}\flreelframe{1.025}
    \begin{scope}
      \clip[rounded corners=2pt] (1.725,3.25) rectangle (2.825,5.45);
      \fldrum{2.275}
      \fldeck{2.275}{5.30}{0.78}{\flheart}{\fldiamond}{\flspade}
      \fldeck{2.275}{4.35}{1}{\flspade}{\flheart}{\fldiamond}
      \fldeck{2.275}{3.40}{0.78}{\fldiamond}{\flspade}{\flheart}
    \end{scope}\flreelframe{2.275}
    \begin{scope}
      \clip[rounded corners=2pt] (2.975,3.25) rectangle (4.075,5.45);
      \fldrum{3.525}
      \flnoise{3.525}{5.30}{2}{0.7} \flnoise{3.525}{4.35}{1}{1} \flnoise{3.525}{3.40}{3}{0.7}
    \end{scope}\flreelframe{3.525}
    \node[font=\scriptsize, cGoldHi] at (0.34,4.35) {$\blacktriangleright$};
    \node[font=\scriptsize, cGoldHi] at (4.14,4.35) {$\blacktriangleleft$};
    \foreach \cx/\a/\b in {1.025/INIT/weights, 2.275/DATA/order, 3.525/NOISE/per-step}{
      \node[font=\flsf\scriptsize\bfseries, cCream] at (\cx,3.00) {\a};
      \node[font=\flsf\tiny\itshape, cGoldHi] at (\cx,2.74) {\b};
    }
    \node[font=\flsf\tiny, cCream!85] at (2.38,2.50)
      {three coupled sources of randomness};
    \draw[cGoldDk, line width=2.6pt, line cap=round] (4.32,3.55) -- (4.88,4.82);
    \draw[cGoldHi, line width=0.9pt, line cap=round] (4.32,3.55) -- (4.88,4.82);
    \fill[cGoldDk] (4.32,3.55) circle (0.10);
    \shade[ball color=cClaret] (4.91,4.89) circle (0.17);

    \draw[-{Stealth[length=5pt]}, cGoldDk, line width=1.1pt] (4.58,3.74) -- (5.14,3.74);
    \fill[black!10, rounded corners=3pt] (5.22,2.94) rectangle (6.96,4.66);
    \filldraw[fill=white, draw=cGoldDk, line width=1pt, rounded corners=3pt]
      (5.18,2.98) rectangle (6.92,4.70);
    \node[font=\flsf\tiny\bfseries, cGoldDk] at (6.05,4.55) {SiT-B/2};
    \foreach \a in {3.80,4.08,4.36}{\foreach \b in {3.70,3.92,4.14,4.36}{
      \draw[cGoldDk!40, line width=0.3pt] (5.52,\a) -- (6.05,\b);}}
    \foreach \a in {3.70,3.92,4.14,4.36}{\foreach \b in {3.80,4.08,4.36}{
      \draw[cGoldDk!40, line width=0.3pt] (6.05,\a) -- (6.58,\b);}}
    \foreach \y in {3.80,4.08,4.36}{
      \filldraw[fill=cCream, draw=cGoldDk, line width=0.5pt] (5.52,\y) circle (0.072);}
    \foreach \y in {3.70,3.92,4.14,4.36}{
      \filldraw[fill=cCream, draw=cGoldDk, line width=0.5pt] (6.05,\y) circle (0.072);}
    \foreach \y in {3.80,4.08,4.36}{
      \filldraw[fill=cCream, draw=cGoldDk, line width=0.5pt] (6.58,\y) circle (0.072);}
    \node[font=\flsf\scriptsize\bfseries, white, fill=cClaret, rounded corners=2.5pt,
          inner xsep=7pt, inner ysep=4pt] at (6.05,3.24) {seed 11396};
    \draw[-{Stealth[length=5pt]}, cGoldDk, line width=1.1pt] (6.96,3.84) -- (7.52,3.84);

    \flcabinet{7.55}{2.22}{13.55}{6.58}
    \draw[rounded corners=2.5pt, fill=cFeltDk, draw=cGold, line width=0.8pt]
      (7.71,5.90) rectangle (13.39,6.46);
    \node[font=\fltitle\small\bfseries, cGoldHi] at (10.55,6.18) {GENERATION LOTTERY};
    \node[font=\flsf\scriptsize, cCream] at (10.55,5.66)
      {10 sampling seeds $\rightarrow$ 10 FIDs};
    \foreach \cx/\cy/\v/\h in {%
      8.31/4.85/33.66/0, 9.43/4.85/34.00/0, 10.55/4.85/33.72/0, 11.67/4.85/33.59/0, 12.79/4.85/33.69/0,
      8.31/3.55/33.59/1, 9.43/3.55/33.80/0, 10.55/3.55/33.87/0, 11.67/3.55/33.82/0, 12.79/3.55/33.77/0}{
        \flgenreel{\cx}{\cy}{\v}{\h}
    }
  \end{tikzpicture}}%
  \caption{\textbf{The FID lottery, drawn as two slot machines.}
  A casino-themed rendering of the same two lotteries diagrammed from first
  principles in \autoref{fig:lottery-machine}. Reporting an FID means
  pulling two levers. The \emph{training lottery} (left) draws one network
  from three coupled sources of randomness. These are the random weight
  initialisation (a die), the shuffled data order (a deck), and the
  per-step flow-matching noise (a trace), and they yield one trained network
  (\emph{centre}, tagged by its seed). The \emph{generation lottery}
  (right) then draws $10$ sampling seeds and scores $10$ FIDs of
  \emph{that} network. On the measured SiT-B/2 panel
  (\autoref{fig:overview}) resampling a fixed network barely moves the
  score ($\sigma_{\mathrm{within}}\!\approx\!0.14$), while retraining moves
  it $\mathbf{3.2\times}$ farther ($\sigma_{\mathrm{between}}\!=\!0.44$): a
  $\mathbf{2.10}$-point gap between the best ($\mathbf{33.59}$) and worst
  ($\mathbf{35.69}$) evaluation, from seeds alone.}
  \label{fig:slot-machine}
\end{figure}
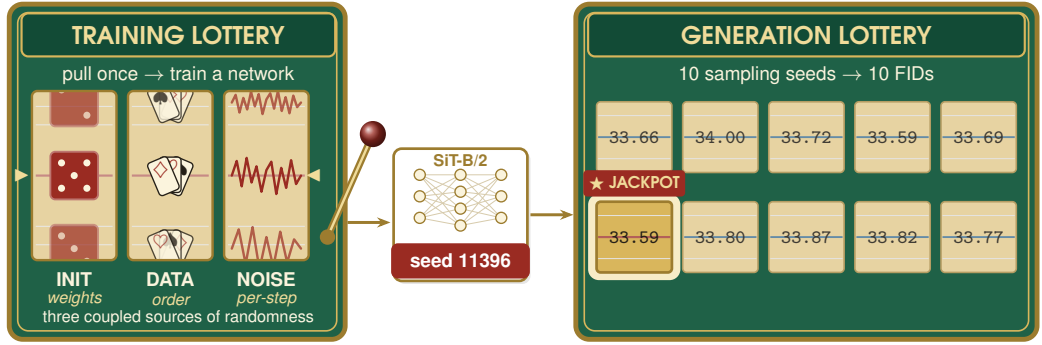

\section{Per-experiment configurations}
\label{app:setup-detail}

This appendix gathers the per-experiment $N$, $K$, training-step
budget, and any deviations from the shared protocol of
\autoref{sec:setup}. The default panel uses $K\!=\!10$ sampling seeds
per trained model and the shared FID evaluation pipeline (Inception-V3
features, $50\,000$ generated samples, ImageNet train-set reference
statistics, deterministic ODE sampler). Each entry below is referenced
from the corresponding subsection of \autoref{sec:results}.

\myparagraph{Implementation details.}
Training follows the SiT recipe of~\citet{ma2024sit} on
class-conditional ImageNet $256{\times}256$. The only deliberate
deviation is that we train $N\!\approx\!25$ independent runs per cell
to populate the two-axis panel rather than a single run per
configuration. The aggregate compute footprint of the experiments
reported in \autoref{sec:results} (including preliminary runs that
did not make it into the final paper) is approximately
$100\,000$ H100 GPU-hours.

\myparagraph{\autoref{sec:results-sampling} (training- vs.\ sampling-seed
asymmetry).}
A converged SiT-B/2 (class-conditional ImageNet $256{\times}256$,
$400$k training steps, no classifier-free guidance) generates $50$k
samples and we compute FID. Repeating with $K\!=\!10$ different
sampling seeds on each of $N\!=\!25$ independently trained SiT-B/2
networks yields the $250$ FID evaluations of \autoref{fig:overview}:
every dot is one evaluation, every violin is one trained model.

\myparagraph{\autoref{sec:results-decomposition} (variance decomposition
across the three training-time RNGs).}
Each single-source ablation fixes two of the three training-time random
number generators (parameter initialisation, data-loader order,
per-step flow-matching noise) and varies the third across $24$--$25$
SiT-B/2 runs trained to $400$k steps with $K\!=\!10$ sampling seeds per
run, yielding the conditions \textsc{vary-noise}, \textsc{vary-init},
and \textsc{vary-data}. The fully-random baseline (\textsc{vary-all})
is the $25\!\times\!10$ panel of \autoref{sec:results-sampling}.

\myparagraph{\autoref{sec:results-guidance} (golden-section CFG search).}
The CFG sweep operates on the same $25\!\times\!10$ converged SiT-B/2
panel. For every (training, sampling) seed pair, golden-section search
runs over the bracket $[\omega_{\min}, \omega_{\max}]\!=\![1, 2]$ at
tolerance $\varepsilon\!=\!0.01$, costing $\approx\!14$ FID evaluations
per cell, and we report the FID at the recovered optimum
(\autoref{alg:gss}, convergence and unimodality in
\autoref{app:golden-section}).

\myparagraph{\autoref{sec:results-scaling}~/~\ref{sec:results-speedup}
(compute and scale).}
We train $N\!=\!25$ networks per model size for $2$M steps and compute
$K\!=\!10$ Inception FID evaluations per checkpoint every $100$k
steps, yielding $76$ (model, step) cells across SiT-S/B/L/XL. The
lucky-seed speedup of \autoref{sec:results-speedup} is read off the
same panel without retraining.

\myparagraph{\autoref{sec:results-mup} ($\boldsymbol{\mu}$P learning-rate
sweep).}
We sweep $10$ $\mu$P-coordinated learning rates log-spaced in
$[5\!\times\!10^{-5}, 5\!\times\!10^{-4}]$ with $N\!=\!10$ training
seeds per cell across the four widths SiT-S/B/L/XL, train each cell to
$100$k steps, and score with both unguided FID and GS-FID
($\approx\!400$ trained networks total, \autoref{fig:mup-lr}).

\section{Companion summary tables}
\label{app:summary-tables}

The figures of the main paper carry the qualitative argument. The
three tables below consolidate the headline numbers behind them so
they can be looked up in one place.

\begin{table}[t]
  \centering
  \caption{\textbf{Variance decomposition of the training-seed
  lottery (SiT-B/2, $400$k, no CFG).} Companion to
  \autoref{fig:disentangle}. Each row fixes two of the three
  training-time random sources and varies the third, and \textsc{vary-all}
  is the fully-stochastic baseline. $\sigma_{\text{between}}$ is the
  standard deviation across the $N$ per-seed means.
  $\sigma_{\text{within}}$ is the within-seed sampling-noise $\sigma$,
  averaged across the $N$ training seeds. $\mathrm{CoV}_{\text{between}}$
  is $\sigma_{\text{between}}/\mu$. The rightmost column reports the
  share of the baseline $\sigma_{\text{between}}$ reproduced by each
  single-source ablation. The naive sum
  $\sqrt{\sigma_{\text{noise}}^{2}+\sigma_{\text{init}}^{2}+\sigma_{\text{data}}^{2}}\!\approx\!0.50$
  overshoots the observed $\sigma_{\text{vary-all}}\!=\!0.44$ by
  $14\%$: the sources are not independent.}
  \label{tab:variance-decomposition}
  \small
  \begin{tabular}{lrrrrrr}
    \toprule
    Condition & $N$ & $\sigma_{\text{between}}$ & $\sigma_{\text{within}}$ &
    $\mathrm{CoV}_{\text{between}}$ (\%) & range & \% of vary-all\\
    \midrule
    \textsc{vary-all}   & 25 & 0.438 & 0.137 & 1.26 & 1.66 & 100.0\\
    \textsc{vary-noise} & 25 & 0.336 & 0.144 & 0.97 & 1.33 & \phantom{0}76.7\\
    \textsc{vary-init}  & 25 & 0.294 & 0.150 & 0.85 & 1.09 & \phantom{0}67.1\\
    \textsc{vary-data}  & 24 & 0.221 & 0.150 & 0.64 & 0.82 & \phantom{0}50.5\\
    \bottomrule
  \end{tabular}
\end{table}

\begin{table}[t]
  \centering
  \caption{\textbf{The seed lottery across compute and model size.}
  Companion to \autoref{fig:variance-training}. Numbers at the start
  ($200$k) and the end ($2$M) of the scaling sweep. Mean FID drops
  by $\approx\!1.7$--$2{\times}$ over $1.8$M extra steps, while
  $\sigma_{\text{between}}$ shrinks $1.7$--$2.4{\times}$, so
  $\mathrm{CoV}_{\text{between}}$ stays near a $1$--$2\%$ band at
  every checkpoint.}
  \label{tab:scaling-sweep}
  \small
  \begin{tabular}{llrrrrr}
    \toprule
    Model & Step & $N$ & mean FID & $\sigma_{\text{between}}$ &
    $\mathrm{CoV}_{\text{between}}$ (\%) & $\sigma$ shrink \\
    \midrule
    SiT-S/2  & 200k & 19 & 71.2 & 0.75 & 1.06 & ---  \\
    SiT-S/2  & 2M   & 19 & 41.1 & 0.31 & 0.74 & $2.4{\times}$ \\
    \midrule
    SiT-B/2  & 200k & 20 & 46.2 & 0.48 & 1.05 & ---  \\
    SiT-B/2  & 2M   & 20 & 23.3 & 0.29 & 1.24 & $1.7{\times}$ \\
    \midrule
    SiT-L/2  & 200k & 24 & 28.9 & 0.43 & 1.48 & ---  \\
    SiT-L/2  & 2M   & 24 & 14.8 & 0.25 & 1.72 & $1.7{\times}$ \\
    \midrule
    SiT-XL/2 & 200k & 25 & 27.5 & 0.44 & 1.61 & ---  \\
    SiT-XL/2 & 2M   & 25 & 14.5 & 0.20 & 1.42 & $2.2{\times}$ \\
    \bottomrule
  \end{tabular}
\end{table}

\begin{table}[t]
  \centering
  \caption{\textbf{$\boldsymbol{\mu}$P sweep at $\boldsymbol{100}$k:
  per-size noise floor.} Companion to \autoref{fig:mup-lr} and
  \autoref{app:mup-extras}. Range of between-seed coefficient of
  variation across the central seven well-conditioned LRs
  $[8\!\times\!10^{-5}, 4\!\times\!10^{-4}]$, broken out by protocol
  (GS-FID versus unguided FID). The argmin LR per size is the modal
  bootstrap optimum from \autoref{app:mup-extras}. CFG-tuning halves
  the absolute FID but the relative spread sits in the same
  $1$--$3\%$ band as the long-run scaling sweep
  (\autoref{tab:scaling-sweep}).}
  \label{tab:mup-cov}
  \small
  \begin{tabular}{lrrrr}
    \toprule
    Size & argmin LR (GS) & GS $\mathrm{CoV}$ range (\%) &
    Unguided $\mathrm{CoV}$ range (\%) & GS \% at argmin\\
    \midrule
    SiT-S/2  & $3.0\!\times\!10^{-4}$ & 1.18--2.76 & 0.82--1.49 & 1.7 \\
    SiT-B/2  & $2.3\!\times\!10^{-4}$ & 1.48--2.61 & 1.06--2.05 & 1.9 \\
    SiT-L/2  & $3.0\!\times\!10^{-4}$ & 1.37--3.53 & 1.58--3.46 & 2.0 \\
    SiT-XL/2 & $3.0\!\times\!10^{-4}$ & 1.05--3.34 & 1.55--3.84 & 2.3 \\
    \bottomrule
  \end{tabular}
\end{table}

\section{Additional analyses on Inception FID}
\label{app:extra-analyses}

This appendix collects five supplementary analyses of the seed lottery
on Inception FID, each tied to a specific subsection of
\autoref{sec:results}. \autoref{app:disentangle-violins} reproduces the
panel layout of \autoref{fig:overview} for the four single-source
conditions of \autoref{sec:results-decomposition}.
\autoref{sec:results-training} asks whether the training-seed ranking
is stable when the within-seed dimension is collapsed by the mean, the
min, or the max. \autoref{sec:results-optimality} runs a
$10\!\times\!15$ factorial of init seeds against (data, noise) pairings
to test whether ``good init'' is a transferable property of an init
seed. \autoref{app:mup-extras} reports the per-bracket numbers behind
the $\mu$P sweep of \autoref{sec:results-mup}.
\autoref{app:fid-to-ci} distils the scale-invariant CoV findings into
a practitioner-facing lookup: given a reported FID, what
seed-induced $95\%$ confidence interval should one expect?

\subsection{Per-train-seed violin panels by variability source}
\label{app:disentangle-violins}

\autoref{fig:disentangle-panel-all} reproduces the panel layout of
\autoref{fig:overview} once for each of the four single-source
conditions of \autoref{sec:results-decomposition}, so the within-seed
generation lottery and the between-seed training lottery can be inspected
side by side under a common y-scale. Within each panel, the vertical
extent of a violin tracks the within-seed $\sigma$ of that training
run, and the vertical spread of the black per-seed mean ticks tracks
the between-seed $\sigma$ of the condition.


\begin{figure}[t]
  \centering
  \pgfplotsset{width=\linewidth, height=4.0cm}
  \begin{tikzpicture}
    \violinsetoptions[scaled]{
      title={(a) vary all},
      ylabel={Inception FID},
      xmin=0.0, xmax=26.0,
      ymin=33.30, ymax=35.95,
      xtick={1,5,10,15,20,25},
      ytick={33.5,34.5,35.5},
      ymajorgrids,
    }
    \violinplotwholefile[%
      col sep=tab,
      primary color=palSlateDark,
      secondary color=palGray,
      indexes={s01,s02,s03,s04,s05,s06,s07,s08,s09,s10,s11,s12,s13,s14,s15,s16,s17,s18,s19,s20,s21,s22,s23,s24,s25},
      labels={,,,,,,,,,,,,,,,,,,,,,,,,},
      spacing=1.0,
    ]{figures/fig_data/fig8_vary_all_violins_wholefile.tsv}
    \begin{axis}[
      xmin=0.0, xmax=26.0, ymin=33.30, ymax=35.95,
      axis line style={draw=none}, tick style={draw=none},
      xticklabels={,,}, yticklabels={,,},
      xmajorticks=false, ymajorticks=false,
      axis on top, clip=false,
      xlabel={Training seed (sorted by mean Inception FID)},
    ]
      \addplot+[only marks, mark=*, mark size=0.6pt,
        mark options={fill=black!55, draw=black!55, fill opacity=0.65, line width=0pt}]
        table[col sep=tab, x=xj, y=fid_inc] {figures/fig_data/fig8_vary_all_strip.tsv};
      \addplot+[only marks, mark=-, mark size=3.8pt,
        mark options={draw=black!88, line width=1.0pt}]
        table[col sep=tab, x=x, y=mean_inc] {figures/fig_data/fig8_vary_all_means.tsv};
    \end{axis}
  \end{tikzpicture}
  \caption{\textbf{(a) Vary all} ($\sigma_{\text{between}}\!=\!0.438$).
  All three randomness sources are free. The panel is the same data as
  \autoref{fig:overview}, repeated here as the reference baseline for
  the three single-source panels that follow.}
  \label{fig:disentangle-panel-all}
\end{figure}
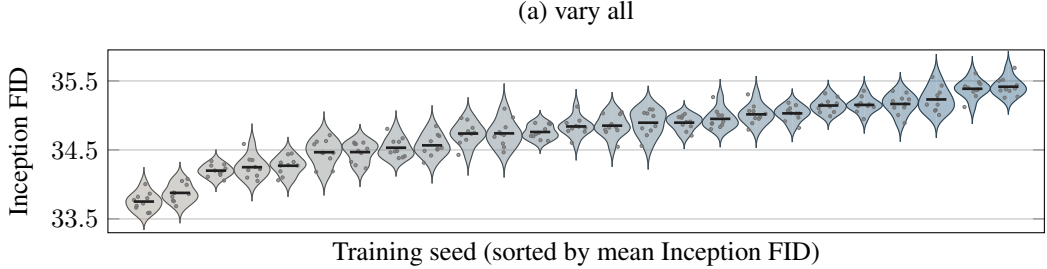

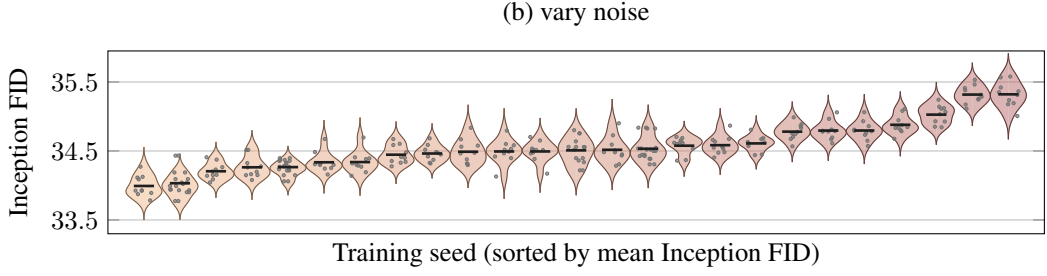
\begin{figure}[t]
  \centering
  \pgfplotsset{width=\linewidth, height=4.0cm}
  \begin{tikzpicture}
    \violinsetoptions[scaled]{
      title={(b) vary noise},
      ylabel={Inception FID},
      xmin=0.0, xmax=26.0,
      ymin=33.30, ymax=35.95,
      xtick={1,5,10,15,20,25},
      ytick={33.5,34.5,35.5},
      ymajorgrids,
    }
    \violinplotwholefile[%
      col sep=tab,
      primary color=palCoralDark,
      secondary color=palPeach,
      indexes={s01,s02,s03,s04,s05,s06,s07,s08,s09,s10,s11,s12,s13,s14,s15,s16,s17,s18,s19,s20,s21,s22,s23,s24,s25},
      labels={,,,,,,,,,,,,,,,,,,,,,,,,},
      spacing=1.0,
    ]{figures/fig_data/fig8_vary_noise_violins_wholefile.tsv}
    \begin{axis}[
      xmin=0.0, xmax=26.0, ymin=33.30, ymax=35.95,
      axis line style={draw=none}, tick style={draw=none},
      xticklabels={,,}, yticklabels={,,},
      xmajorticks=false, ymajorticks=false,
      axis on top, clip=false,
      xlabel={Training seed (sorted by mean Inception FID)},
    ]
      \addplot+[only marks, mark=*, mark size=0.6pt,
        mark options={fill=black!55, draw=black!55, fill opacity=0.65, line width=0pt}]
        table[col sep=tab, x=xj, y=fid_inc] {figures/fig_data/fig8_vary_noise_strip.tsv};
      \addplot+[only marks, mark=-, mark size=3.8pt,
        mark options={draw=black!88, line width=1.0pt}]
        table[col sep=tab, x=x, y=mean_inc] {figures/fig_data/fig8_vary_noise_means.tsv};
    \end{axis}
  \end{tikzpicture}
  \caption{\textbf{(b) Vary noise} ($\sigma_{\text{between}}\!=\!0.336$).
  Init and data order are fixed. Only the per-step Gaussian noise of
  the flow-matching loss varies between training runs. Noise alone
  reproduces ${\approx}77\%$ of the baseline between-seed
  $\sigma$ of (a).}
  \label{fig:disentangle-panel-noise}
\end{figure}

\begin{figure}[t]
  \centering
  \pgfplotsset{width=\linewidth, height=4.0cm}
  \begin{tikzpicture}
    \violinsetoptions[scaled]{
      title={(c) vary init},
      ylabel={Inception FID},
      xmin=0.0, xmax=25.0,
      ymin=33.30, ymax=35.95,
      xtick={1,5,10,15,20,24},
      ytick={33.5,34.5,35.5},
      ymajorgrids,
    }
    \violinplotwholefile[%
      col sep=tab,
      primary color=palSageDark,
      secondary color=palMint,
      indexes={s01,s02,s03,s04,s05,s06,s07,s08,s09,s10,s11,s12,s13,s14,s15,s16,s17,s18,s19,s20,s21,s22,s23,s24},
      labels={,,,,,,,,,,,,,,,,,,,,,,,},
      spacing=1.0,
    ]{figures/fig_data/fig8_vary_init_violins_wholefile.tsv}
    \begin{axis}[
      xmin=0.0, xmax=25.0, ymin=33.30, ymax=35.95,
      axis line style={draw=none}, tick style={draw=none},
      xticklabels={,,}, yticklabels={,,},
      xmajorticks=false, ymajorticks=false,
      axis on top, clip=false,
      xlabel={Training seed (sorted by mean Inception FID)},
    ]
      \addplot+[only marks, mark=*, mark size=0.6pt,
        mark options={fill=black!55, draw=black!55, fill opacity=0.65, line width=0pt}]
        table[col sep=tab, x=xj, y=fid_inc] {figures/fig_data/fig8_vary_init_strip.tsv};
      \addplot+[only marks, mark=-, mark size=3.8pt,
        mark options={draw=black!88, line width=1.0pt}]
        table[col sep=tab, x=x, y=mean_inc] {figures/fig_data/fig8_vary_init_means.tsv};
    \end{axis}
  \end{tikzpicture}
  \caption{\textbf{(c) Vary init} ($\sigma_{\text{between}}\!=\!0.294$).
  Data order and noise are fixed. Only the parameter initialisation
  varies. Init alone reproduces ${\approx}67\%$ of the baseline
  between-seed $\sigma$ of (a).}
  \label{fig:disentangle-panel-init}
\end{figure}
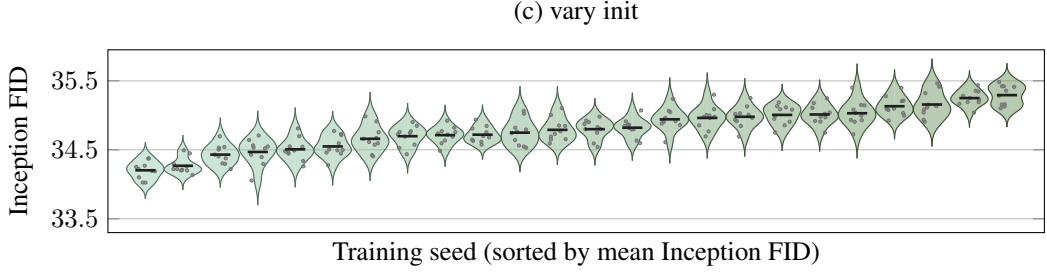

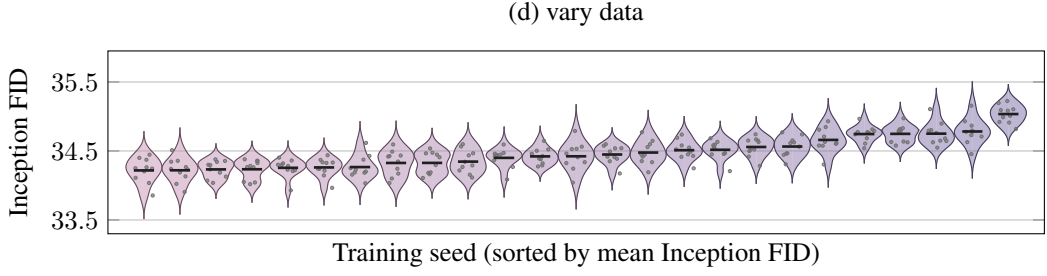
\begin{figure}[t]
  \centering
  \pgfplotsset{width=\linewidth, height=4.0cm}
  \begin{tikzpicture}
    \violinsetoptions[scaled]{
      title={(d) vary data},
      ylabel={Inception FID},
      xmin=0.0, xmax=26.0,
      ymin=33.30, ymax=35.95,
      xtick={1,5,10,15,20,25},
      ytick={33.5,34.5,35.5},
      ymajorgrids,
    }
    \violinplotwholefile[%
      col sep=tab,
      primary color=palLavDark,
      secondary color=palMauve,
      indexes={s01,s02,s03,s04,s05,s06,s07,s08,s09,s10,s11,s12,s13,s14,s15,s16,s17,s18,s19,s20,s21,s22,s23,s24,s25},
      labels={,,,,,,,,,,,,,,,,,,,,,,,,},
      spacing=1.0,
    ]{figures/fig_data/fig8_vary_data_violins_wholefile.tsv}
    \begin{axis}[
      xmin=0.0, xmax=26.0, ymin=33.30, ymax=35.95,
      axis line style={draw=none}, tick style={draw=none},
      xticklabels={,,}, yticklabels={,,},
      xmajorticks=false, ymajorticks=false,
      axis on top, clip=false,
      xlabel={Training seed (sorted by mean Inception FID)},
    ]
      \addplot+[only marks, mark=*, mark size=0.6pt,
        mark options={fill=black!55, draw=black!55, fill opacity=0.65, line width=0pt}]
        table[col sep=tab, x=xj, y=fid_inc] {figures/fig_data/fig8_vary_data_strip.tsv};
      \addplot+[only marks, mark=-, mark size=3.8pt,
        mark options={draw=black!88, line width=1.0pt}]
        table[col sep=tab, x=x, y=mean_inc] {figures/fig_data/fig8_vary_data_means.tsv};
    \end{axis}
  \end{tikzpicture}
  \caption{\textbf{(d) Vary data} ($\sigma_{\text{between}}\!=\!0.221$).
  Init and noise are fixed. Only the data-loader order varies. Data
  order alone reproduces ${\approx}51\%$ of the baseline between-seed
  $\sigma$ of (a). All four panels share the same y-range and a
  consistent layout: every violin is one training seed showing the
  Gaussian KDE of its $10$ sampling-seed evaluations, the small dots
  are individual evaluations, and the black tick is the per-seed mean.}
  \label{fig:disentangle-panel-data}
\end{figure}

The within-seed jitter is invariant across conditions. Every panel
holds violins of similar height: the per-condition within-seed
$\sigma_{\text{within}}$ stays in $[0.137, 0.152]$
(\autoref{sec:results-decomposition}), and the figure makes this
constancy visible. No panel inflates or deflates relative to the
others. What we vary during training does not leak into the
evaluation jitter, so the generation lottery is set by the converged
SiT-B/2 weights and not by which randomness sources produced them.

The between-seed spread shrinks monotonically with $\sigma$. The
per-seed mean ticks span a wide vertical band in panel (a)
(\emph{vary all}, $0.438$), narrow steadily through (b)
(\emph{vary noise}, $0.336$) and (c) (\emph{vary init}, $0.294$), and
collapse to the tight stripe of (d) (\emph{vary data}, $0.221$). The
top-to-bottom range of the mean ticks halves from $\approx\!1.7$ FID
in (a) to under $\approx\!0.9$ FID in (d), while the violin heights
above and below each tick stay essentially fixed. The lottery
contracts in the training axis without touching the evaluation axis.

The four panels also confirm that no condition contains a heavy-tailed
minority of training seeds: the mean ticks within each panel are
spread evenly along the sorted axis, not bunched at the centre with a
few outliers, which is consistent with the well-behaved
$(\max\!-\!\min)/\sigma$ ratios reported in
\autoref{sec:setup}.

\subsection{Summary statistics reshuffle the training-seed ranking}
\label{sec:results-training}

Even within a single feature space, the way the ten sampling seeds per
training run are summarised determines which seed is declared best.
Reusing the $25\!\times\!10$ SiT-B/2 panel of
\autoref{sec:results-sampling}, we collapse the within-seed dimension
three ways: per-seed mean, per-seed minimum, and per-seed maximum FID.
We then ask whether the resulting orderings of the $25$ training
seeds agree.

The three rankings disagree (\autoref{fig:rank-stability}a). Rank
crossings are dense, and the seed that wins under ``best-case
sampling'' (per-seed minimum) is not the seed that wins under
``average-case sampling'' (per-seed mean). A benchmark that reports
the best out of $k$ sampling seeds and a benchmark that reports the
average over the same $k$ are not measuring the same thing, even when
both are computed from the identical $25\!\times\!10$ panel.

The instability is specific to the summary statistic, not the feature
extractor: Inception and DINOv2 agree on which training seed is good
under each of the three criteria (Spearman $\rho\!\geq\!0.94$,
appendix). The training lottery is a property of the trained
model, but the choice of how to summarise the sampling seeds
determines which seed appears to win.

\begin{figure*}[t]
  \centering
  \begin{tikzpicture}
    \begin{groupplot}[
        group style={
          group size=2 by 1,
          horizontal sep=1.7cm,
        },
        fidlottery,
        width=0.46\textwidth,
        height=5.6cm,
      ]

      \nextgroupplot[
        title={(a) Per-seed rank by summary statistic},
        ylabel={Rank (1 = best Inception FID)},
        ymin=0.4, ymax=25.6,
        y dir=reverse,
        ytick={1,5,10,15,20,25},
        xmin=0.5, xmax=3.5,
        xtick={1,2,3},
        xticklabels={mean, min, max},
        x tick label style={font=\footnotesize},
        legend style={
          at={(0.5,-0.18)}, anchor=north,
          legend columns=2,
          font=\scriptsize, draw=black!15,
          fill=white, fill opacity=0.92, text opacity=1,
          inner sep=2pt, rounded corners=1pt,
        },
        legend image post style={scale=0.8},
      ]
        \addplot+[
          mark=*, mark size=1.4pt, line width=0.5pt,
          color=palLavender,
          mark options={fill=palLavender, draw=palLavDark, line width=0.3pt},
          opacity=0.8,
          forget plot,
        ] table[col sep=tab, x=x, y=rank] {figures/fig_data/fig3a_bump.tsv};

        \addplot+[
          mark=*, mark size=2.4pt, line width=1.4pt,
          color=palCoralDark,
          mark options={fill=palCoral, draw=palCoralDark, line width=0.4pt},
        ] table[col sep=tab, x=x, y=rank] {figures/fig_data/fig3a_best.tsv};
        \addlegendentry{best-by-mean}

        \addplot+[
          mark=*, mark size=2.4pt, line width=1.4pt,
          color=palTealDark,
          mark options={fill=palTeal, draw=palTealDark, line width=0.4pt},
        ] table[col sep=tab, x=x, y=rank] {figures/fig_data/fig3a_worst.tsv};
        \addlegendentry{worst-by-mean}

      \nextgroupplot[
        title={(b) Spearman $\rho$ across criteria},
        xmin=-0.6, xmax=5.6,
        ymin=-0.6, ymax=5.6,
        y dir=reverse,
        xtick={0,1,2,3,4,5},
        ytick={0,1,2,3,4,5},
        xticklabels={Inc-Mean, Inc-Min, Inc-Max, DINO-Mean, DINO-Min, DINO-Max},
        yticklabels={Inc-Mean, Inc-Min, Inc-Max, DINO-Mean, DINO-Min, DINO-Max},
        x tick label style={font=\scriptsize, rotate=35, anchor=north east},
        y tick label style={font=\scriptsize},
        colormap name=rhomap,
        point meta min=0.85, point meta max=1.0,
        colorbar,
        colorbar style={
          width=4pt, font=\scriptsize,
          ytick={0.85,0.9,0.95,1.0},
          title={$\rho$}, title style={yshift=-2pt},
        },
        axis equal image,
        scatter/use mapped color={draw=black!20, fill=mapped color, line width=0.2pt},
      ]
        \addplot+[
          scatter, only marks,
          point meta=explicit,
          mark=square*, mark size=12pt,
          nodes near coords*={\pgfmathprintnumber[precision=2,fixed]{\pgfplotspointmeta}},
          nodes near coords style={font=\scriptsize, color=black!75, anchor=center},
          every node near coord/.append style={yshift=0pt},
        ] table[col sep=tab, x=i, y=j, meta=rho]
          {figures/fig_data/fig3b_spearman.tsv};

    \end{groupplot}
  \end{tikzpicture}
  \caption{\textbf{Rank stability of the 25 training seeds (SiT-B/2, 400k).}
  (a)~Bump chart: each line is one training seed traced through three
  ranking criteria — mean, min, and max FID across its 10 sampling seeds.
  Crossings dominate the picture: the best seed by \emph{mean} is rarely
  the best by \emph{min} or \emph{max}. Coral and teal highlight the
  seeds that are best- and worst-by-mean to make their rank trajectories
  visible.
  (b)~Spearman $\rho$ between rankings under all six combinations of
  $\{$Inception, DINOv2$\}\times\{$mean, min, max$\}$. Inception and
  DINOv2 agree strongly on the \emph{mean} ranking ($\rho\!=\!0.99$) but
  the agreement weakens for the \emph{max} criterion ($\rho\!=\!0.94$),
  i.e.\ the two metrics disagree more often on which seed had the
  worst-case sampling.}
  \label{fig:rank-stability}
\end{figure*}
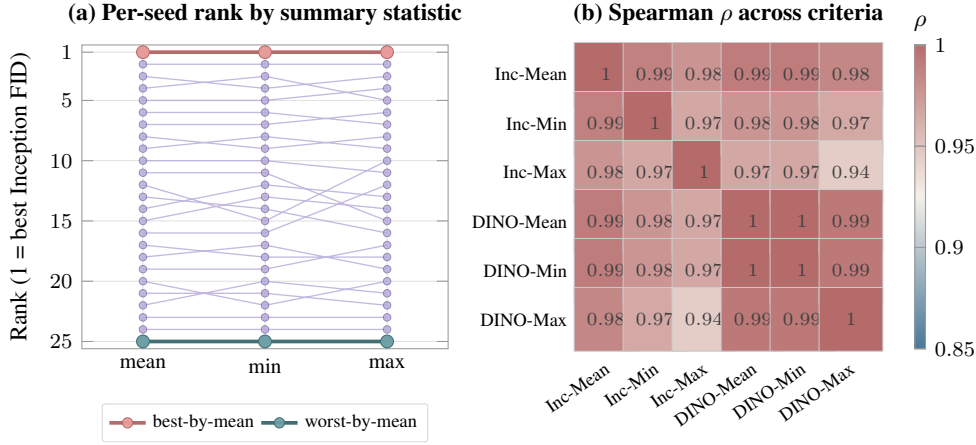

\subsection{Are ``good'' init seeds universal?}
\label{sec:results-optimality}

\autoref{sec:results-decomposition} showed that initialisation
contributes $\approx\!67\%$ of the baseline between-seed $\sigma$ on its
own, which makes ``find a good init'' the cheapest and most tempting
shortcut a practitioner can take to a low FID. We test whether such an
init exists by running a full factorial of $10$ init seeds against $15$
independent (data, noise) pairings on SiT-B/2 at $400$k, with $\approx\!11$
sampling seeds per cell, totalling $\approx\!1\,600$ FID evaluations
(two of the $150$ cells failed to train and are dropped).
\autoref{fig:optimality} shows the resulting $10\!\times\!15$ FID
heatmap and the bump chart of init ranks across pairs.

If we average across the $15$ (data, noise) pairs, the $10$ init grand
means span only $[34.49, 35.27]$ Inception FID, a $0.78$-FID range
slightly tighter than the $1.66$ between-seed range of
\autoref{sec:results-sampling}. After enough marginalisation, init
seeds are nearly interchangeable. The within-cell spread, on the other
hand, lands almost exactly on the generation lottery of
\autoref{sec:results-sampling}: the median per-cell range is $0.51$
and the median per-cell $\sigma$ is $0.149$, against $0.137$ in the
$25\!\times\!10$ baseline panel. The cross-experiment match is a
useful sanity check (two completely separate sweeps recover the same
within-seed sampling-noise scale), and it tells us the generation lottery
is set by the converged SiT-B/2 weights, not by which init or which
data pairing produced them. In aggregate, an init looks like a
constant offset.

The aggregate picture is misleading. Computing the $10$-init ranking
\emph{independently} for each of the $15$ pairs and asking how
concordant the rankings are gives a Kendall's $W$ of $0.41$ under the
mean criterion, $0.45$ under the min, and $0.36$ under the max. The
average pairwise Spearman $\rho$ between pair-rankings is $0.36$.
These values are significant ($p\!<\!0.01$), so ``init quality'' is
not pure noise, but they are also far from concordant: the same init
can be top-3 under one (data, noise) pair and bottom-3 under another.
The complementary $15$-pair ranking computed per init has
Kendall's $W$ in $[0.11, 0.13]$, statistically indistinguishable from
random, so knowing which (data, noise) pair was good for one init says
almost nothing about another init.

The bump chart in \autoref{fig:optimality}(b) makes this visceral: most
init seeds visit both the top three and the bottom three across the
$15$ pairings. The best init by grand mean ($90682$) takes top-$1$ in
only $5$ of $15$ pairings. The worst init ($96273$) is bottom-$1$ in
$8$ of $15$. Some inits are weakly worse than the rest, but no init is
universally best, and no practitioner who reports a ``best of $10$
inits'' run is reporting a transferable artefact. This is also why
fixing the init alone in \autoref{sec:results-decomposition} only
removes $33\%$ of the baseline variance: the init lottery is not a
constant per-init offset to the FID, but interacts strongly with the
(data, noise) draw the init is paired with. The init lottery and the
data lottery are entangled, and disentangling them by varying one at a
time understates the seed lottery the practitioner actually sees.

\begin{figure*}[t]
  \centering
  \begin{tikzpicture}
    \begin{groupplot}[
        group style={
          group size=2 by 1,
          horizontal sep=0.55cm,
        },
        fidlottery,
        width=0.56\textwidth,
        height=5.0cm,
      ]

      \nextgroupplot[
        title={(a) Per-cell mean Inception FID},
        xlabel={(data, noise) pair index},
        ylabel={Init seed index},
        xmin=-0.55, xmax=14.55,
        ymin=-0.55, ymax=9.55,
        xtick={0,2,4,6,8,10,12,14},
        ytick={0,2,4,6,8},
        x tick label style={font=\scriptsize},
        y tick label style={font=\scriptsize},
        colormap name=fidmap,
        point meta min=34.0, point meta max=35.5,
        colorbar,
        colorbar style={
          width=5pt, font=\scriptsize,
          ytick={34, 34.5, 35, 35.5},
          title={FID}, title style={yshift=-2pt, font=\scriptsize},
        },
        enlargelimits=false,
      ]
        \addplot+[
          scatter, only marks,
          point meta=explicit,
          mark=square*, mark size=8.2pt,
          scatter/use mapped color={draw=mapped color, fill=mapped color, line width=0pt},
        ] table[col sep=tab, x=x, y=y, meta=mean_inc]
          {figures/fig_data/fig5a_heatmap.tsv};

      \nextgroupplot[
        title={(b) Init rank by pair (Kendall's $W\!=\!0.41$)},
        xlabel={(data, noise) pair index},
        ylabel={Rank within pair (1 = best)},
        xmin=-0.4, xmax=14.4,
        ymin=0.4, ymax=10.6,
        y dir=reverse,
        xtick={0,2,4,6,8,10,12,14},
        ytick={1,3,5,7,9},
        x tick label style={font=\scriptsize},
        y tick label style={font=\scriptsize},
      ]
        \pgfplotsset{cycle list={
          {palCoral, mark=*}, {palSlate, mark=*}, {palSage, mark=*}, {palLavender, mark=*},
          {palRose, mark=*}, {palTeal, mark=*}, {palPeach, mark=*}, {palMint, mark=*},
          {palMauve, mark=*}, {palButter, mark=*},
        }}
        \addplot+[mark size=1.4pt, line width=0.7pt] table[col sep=tab, x=x, y=rank] {figures/fig_data/fig5b_ranks_init0.tsv};
        \addplot+[mark size=1.4pt, line width=0.7pt] table[col sep=tab, x=x, y=rank] {figures/fig_data/fig5b_ranks_init1.tsv};
        \addplot+[mark size=1.4pt, line width=0.7pt] table[col sep=tab, x=x, y=rank] {figures/fig_data/fig5b_ranks_init2.tsv};
        \addplot+[mark size=1.4pt, line width=0.7pt] table[col sep=tab, x=x, y=rank] {figures/fig_data/fig5b_ranks_init3.tsv};
        \addplot+[mark size=1.4pt, line width=0.7pt] table[col sep=tab, x=x, y=rank] {figures/fig_data/fig5b_ranks_init4.tsv};
        \addplot+[mark size=1.4pt, line width=0.7pt] table[col sep=tab, x=x, y=rank] {figures/fig_data/fig5b_ranks_init5.tsv};
        \addplot+[mark size=1.4pt, line width=0.7pt] table[col sep=tab, x=x, y=rank] {figures/fig_data/fig5b_ranks_init6.tsv};
        \addplot+[mark size=1.4pt, line width=0.7pt] table[col sep=tab, x=x, y=rank] {figures/fig_data/fig5b_ranks_init7.tsv};
        \addplot+[mark size=1.4pt, line width=0.7pt] table[col sep=tab, x=x, y=rank] {figures/fig_data/fig5b_ranks_init8.tsv};
        \addplot+[mark size=1.4pt, line width=0.7pt] table[col sep=tab, x=x, y=rank] {figures/fig_data/fig5b_ranks_init9.tsv};

    \end{groupplot}
  \end{tikzpicture}
  \caption{\textbf{Seed optimality: are ``good'' init seeds universal?
  (SiT-B/2, 400k, no CFG.)}
  (a)~Heatmap of mean Inception FID over a $10\!\times\!15$ grid of
  init seeds (rows) and (data, noise) pairings (columns). Cells span
  ${\sim}33.9$–$35.9$. A single init does not consistently shade greenest
  across rows.
  (b)~Bump chart of the same data: each line is one init seed traced
  through its rank within each (data, noise) pair. Heavy crossings make
  the rank instability visceral. Kendall's $W\!=\!0.41$ on the init
  rankings (mean criterion) shows significant agreement, but far from
  concordance. The ``best'' init by grand mean wins only $5/15$ pairs.}
  \label{fig:optimality}
\end{figure*}
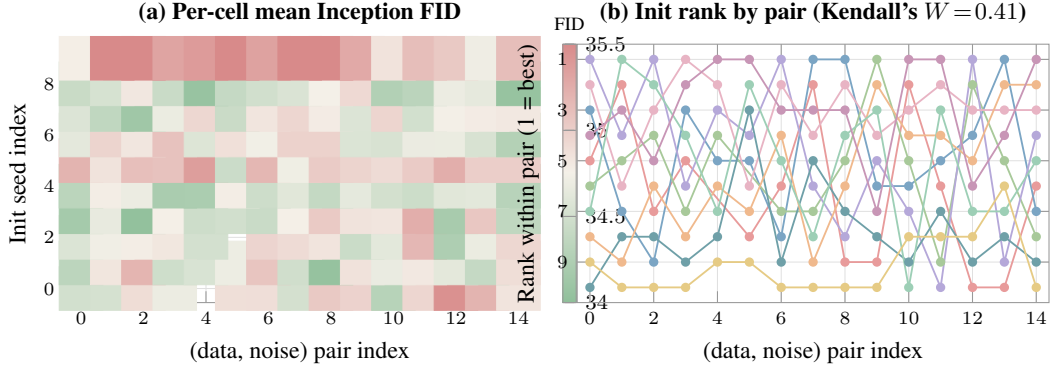

\subsection{Detailed numbers for the \texorpdfstring{$\mu$}{mu}P sweep}
\label{app:mup-extras}

This appendix collects the per-bracket numbers behind the main-paper
discussion in \autoref{sec:results-mup}. The full panel covers SiT-S,
SiT-B, SiT-L and SiT-XL on ImageNet $256{\times}256$, with $10$
$\mu$P-coordinated learning rates log-spaced from $5\!\times\!10^{-5}$
to $5\!\times\!10^{-4}$ and $10$ training seeds per (size, LR),
evaluated at $100$k steps under both the unguided and the GS-FID
protocol, $\approx\!400$ trained networks and $\approx\!4\,000$
FID evaluations. Source: \texttt{paper\_data/06\_mup\_sweep/}.

\myparagraph{Per-LR seed envelope at every well-conditioned LR.}
Across the central seven well-conditioned LRs, the between-seed
coefficient of variation stays inside the same $1$--$3\%$ band on
every model size, under both GS-FID and unguided FID. Per-size ranges
appear in \autoref{tab:mup-cov}. The band coincides with the
$1.0$--$2.0\%$ band of the $200$k$\to\!2$M scaling sweep
(\autoref{sec:results-scaling}), so $\mu$P transfers a stable noise
floor across widths, not just a stable mean.

\myparagraph{Stability collapses only at the rightmost LR.}
At $5\!\times\!10^{-4}$ the GS-FID coefficient of variation jumps an
order of magnitude above its central-LR value on SiT-B (one seed
reaches FID $31.3$ against a cluster around $17$), and rises by
$3$--$4{\times}$ on SiT-L and SiT-XL. SiT-S diverges outright on three
of ten seeds (FID $>\!300$, excluded from the curve). The unguided
protocol blunts the symptom rather than removing it: the same
seed-to-seed differences sit on top of much larger absolute FIDs, so
the rightmost unguided CoV stays below $7\%$. A single point at this
edge can therefore mix a successful run and a divergent one, and
reporting it without the divergence count produces a misleading
number.

\myparagraph{Per-seed argmin disagreement under unguided vs.\ GS-FID.}
The two protocols rank LRs differently \emph{per seed}, not only on
average. The per-seed argmin LR disagrees on a majority of seeds
across SiT-S, SiT-B and SiT-L (specifically $8/10$, $6/8$ and
$5/10$). At the population level the modal unguided argmin is the
rightmost LR ($5\!\times\!10^{-4}$) for every size, while the modal GS
argmin sits one or two notches to the left at
$2.3$--$3\!\times\!10^{-4}$. The cost of trusting the unguided
argmin therefore stacks across SiT-B/L: it is $7.7$--$11.5\%$ worse
in GS-FID than the GS optimum, carries a $5$--$12{\times}$ wider
between-seed envelope, and on SiT-S coincides with the LR at which
three seeds diverge.

\myparagraph{Per-seed bootstrap of the optimal LR.}
The bootstrap mass on each LR being optimal (the credible set used
for the strips in \autoref{fig:mup-lr}) concentrates almost
entirely on a single LR for SiT-S/B/L: the modal LR carries
$83$--$97\%$ of the mass on each of the three sizes. SiT-XL is the
exception: its mass spreads across three adjacent LRs at $51\%$,
$28\%$ and $21\%$ on $\{3, 3.9, 5\}\!\times\!10^{-4}$, with one of
the three plausible optima ($5\!\times\!10^{-4}$) sitting at the edge
of stability. The XL credible set is therefore the flattest, and a
practitioner who reads off ``the'' $\mu$P-transferred LR at this size
should expect three near-equally plausible answers rather than one.

\subsection{What CI should I expect at my reported FID?}
\label{app:fid-to-ci}

\noindent\emph{A practitioner-facing lookup distilled from
\autoref{tab:scaling-sweep} and \autoref{fig:supp-cov-bands}: given
the FID one just measured, the seed-induced $95\%$ CI on that number
is a fixed fraction of the FID, set by the scale-invariant CoV floor
of the panel.}

\myparagraph{Setup.}
The between-seed coefficient of variation
$\mathrm{CoV}\!=\!\sigma_{\text{between}}/\mu$ on Inception FID stays
in a tight band across the $76$ $(\text{model}, \text{step})$ cells of
the scaling sweep: median $1.30\%$, $p_{10}$--$p_{90}$
$0.88$--$1.73\%$ (\autoref{fig:supp-cov-bands},
\autoref{tab:scaling-sweep}). For a mean FID computed from $N$
independently trained models with $K\!=\!10$ sampling seeds each, the
normal-approximation $95\%$ CI half-width on the seed-mean is
\begin{equation}
  \mathrm{CI}_{95}(F, N) \;=\; z_{0.975}\,\frac{\mathrm{CoV}\,F}{\sqrt{N}}
  \;\approx\; \frac{0.0254\,F}{\sqrt{N}},
  \qquad z_{0.975}\!\approx\!1.96,
  \label{eq:fid-ci}
\end{equation}
under \mbox{$\mathrm{CoV}\!=\!1.30\%$}. The within-seed contribution
$\sigma_{\text{within}}^{2}/K$ is absorbed by the displayed CoV band:
at $K\!=\!10$ the within-CoV is $\!\approx\!0.4\%$
(\autoref{tab:supp-overview}) and adds less than $5\%$ to the variance
of the seed-mean.

%
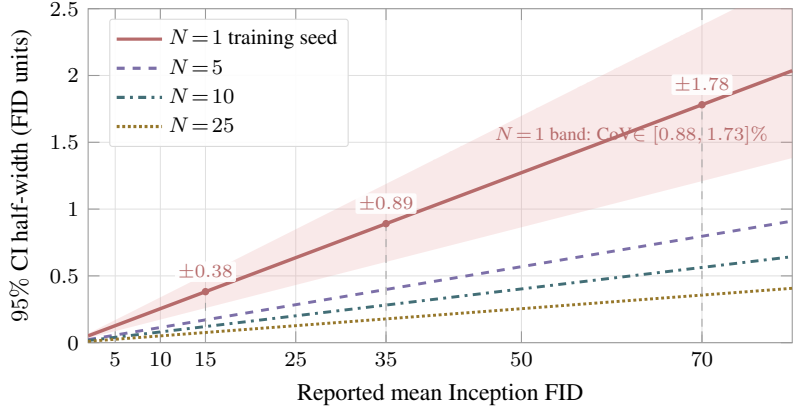
\begin{figure}[t]
  \centering
  \begin{tikzpicture}
    \begin{axis}[
      fidlottery,
      width=0.78\linewidth,
      height=6.0cm,
      xlabel={Reported mean Inception FID},
      ylabel={$95\%$ CI half-width (FID units)},
      xmin=2, xmax=80,
      ymin=0, ymax=2.5,
      xtick={5, 10, 15, 25, 35, 50, 70},
      ytick={0, 0.5, 1.0, 1.5, 2.0, 2.5},
      legend pos=north west,
      legend cell align=left,
      legend style={font=\footnotesize, draw=black!15, fill=white,
                    fill opacity=0.92, text opacity=1, inner sep=2pt,
                    rounded corners=1pt},
      legend image post style={line width=1.2pt},
      clip mode=individual,
    ]
      \addplot+[name path=ciP10, mark=none, draw=none, forget plot]
        table[col sep=tab, x=fid, y=n1_p10]
        {figures/fig_data/figS_fid_to_ci.tsv};
      \addplot+[name path=ciP90, mark=none, draw=none, forget plot]
        table[col sep=tab, x=fid, y=n1_p90]
        {figures/fig_data/figS_fid_to_ci.tsv};
      \addplot[palCoral, opacity=0.22, forget plot]
        fill between[of=ciP10 and ciP90];

      \addplot+[mark=none, line width=1.3pt, color=palCoralDark]
        table[col sep=tab, x=fid, y=n1_med]
        {figures/fig_data/figS_fid_to_ci.tsv};
      \addlegendentry{$N\!=\!1$ training seed}
      \addplot+[mark=none, line width=1.1pt, color=palLavDark, dashed]
        table[col sep=tab, x=fid, y=n5_med]
        {figures/fig_data/figS_fid_to_ci.tsv};
      \addlegendentry{$N\!=\!5$}
      \addplot+[mark=none, line width=1.1pt, color=palTealDark, dashdotted]
        table[col sep=tab, x=fid, y=n10_med]
        {figures/fig_data/figS_fid_to_ci.tsv};
      \addlegendentry{$N\!=\!10$}
      \addplot+[mark=none, line width=1.1pt, color=palOchreDark, densely dotted]
        table[col sep=tab, x=fid, y=n25_med]
        {figures/fig_data/figS_fid_to_ci.tsv};
      \addlegendentry{$N\!=\!25$}

      \draw[dashed, black!35, line width=0.4pt]
        (axis cs:15, 0) -- (axis cs:15, 0.382);
      \fill[palCoralDark] (axis cs:15, 0.382) circle [radius=1.4pt];
      \node[font=\scriptsize, color=palCoralDark, anchor=south,
            fill=white, fill opacity=0.90, text opacity=1,
            inner sep=1.2pt, rounded corners=0.8pt]
        at (axis cs:15, 0.45) {$\pm0.38$};

      \draw[dashed, black!35, line width=0.4pt]
        (axis cs:35, 0) -- (axis cs:35, 0.891);
      \fill[palCoralDark] (axis cs:35, 0.891) circle [radius=1.4pt];
      \node[font=\scriptsize, color=palCoralDark, anchor=south,
            fill=white, fill opacity=0.90, text opacity=1,
            inner sep=1.2pt, rounded corners=0.8pt]
        at (axis cs:35, 0.96) {$\pm0.89$};

      \draw[dashed, black!35, line width=0.4pt]
        (axis cs:70, 0) -- (axis cs:70, 1.781);
      \fill[palCoralDark] (axis cs:70, 1.781) circle [radius=1.4pt];
      \node[font=\scriptsize, color=palCoralDark, anchor=south,
            fill=white, fill opacity=0.90, text opacity=1,
            inner sep=1.2pt, rounded corners=0.8pt]
        at (axis cs:70, 1.85) {$\pm1.78$};

      \node[font=\scriptsize, color=palCoralDark, anchor=west]
        at (axis cs:46, 1.55) {$N\!=\!1$ band: CoV$\in[0.88,1.73]\%$};
    \end{axis}
  \end{tikzpicture}
  \caption{\textbf{Seed-induced $\boldsymbol{95\%}$ confidence interval
  as a function of the reported Inception FID.} For a mean FID computed
  from $N$ independently trained models with $K\!=\!10$ sampling seeds
  each, the normal-approximation half-width is
  $\mathrm{CI}_{95}\!=\!1.96\,\mathrm{CoV}\,F/\sqrt{N}$, where
  $\mathrm{CoV}\!=\!\sigma_{\text{between}}/\mu$ is the scale-invariant
  noise floor reported for Inception FID across the $76$ $(\text{model},
  \text{step})$ cells of the scaling sweep
  (\autoref{tab:supp-cov-bands}, with median $1.30\%$ and $p_{10}\!-\!p_{90}\!=\!0.88\!-\!1.73\%$).
  The solid coral curve uses the median CoV. The shaded band traces the
  $p_{10}\!-\!p_{90}$ envelope at $N\!=\!1$. Three readings: at
  $\text{FID}\!\approx\!15$ (typical converged SiT-L), a single trained
  model carries a $\pm0.38$ FID seed-only CI -- four times the
  $\pm0.1$ gaps that routinely separate published methods. Pulling that
  CI under $\pm0.1$ requires $N\!=\!25$ training seeds. Within-seed
  (sampling-only) contributions are absorbed into the displayed band:
  at $K\!=\!10$ the within-CoV is $\!\approx\!0.4\%$ on this panel and
  contributes less than $5\%$ to the total variance of the seed-mean.}
  \label{fig:supp-fid-to-ci}
\end{figure}

\myparagraph{Reading the figure.}
\autoref{fig:supp-fid-to-ci} plots
$\mathrm{CI}_{95}$ against the reported mean FID for
$N\!\in\!\{1,5,10,25\}$. Three operating points to keep in mind. At
$F\!\approx\!15$ (typical converged SiT-L), a single trained model
carries a $\pm0.38$ FID seed-only $95\%$ CI. At $F\!\approx\!35$
(SiT-B/2 baseline panel) the CI is $\pm0.89$. At $F\!\approx\!70$
(early-training regime) it is $\pm1.78$. The shaded band traces the
$p_{10}\!-\!p_{90}$ envelope of the CoV across the panel, so a cell
that happens to sit at the noisy end of the floor inflates these
numbers by a further $\!\approx\!1.3\!\times\!$.

\myparagraph{Implications for benchmarking.}
The CI shrinks only as $1/\sqrt{N}$. Pulling the CI on a
$F\!\approx\!15$ benchmark below $\pm0.1$ -- the gap that routinely
separates published methods -- requires $N\!\geq\!25$ training seeds.
A single-seed report at this FID admits a $\pm0.38$ uncertainty band
under \emph{the same architecture and recipe}, against which a
$0.1$-FID gain at the next paper is statistically indistinguishable
from noise. \autoref{eq:fid-ci} makes the CI explicit and is easy to
report alongside any FID number. For a metric other than Inception
FID, substitute the matching CoV band from
\autoref{tab:supp-cov-bands}. An
\href{https://kyutai.org/fid-lottery/\#calculator}{online calculator}
evaluates \autoref{eq:fid-ci} for any reported FID and seed count $N$,
returning the seed-only $95\%$ error bar.

\section{Theory of golden-section search for FID(CFG)}
\label{app:golden-section}

This appendix supports \autoref{sec:results-guidance}. The
golden-section procedure is stated formally, its convergence is proved,
and $\mathrm{FID}(\mathrm{CFG})$ is shown to be unimodal under a
Gaussian feature model so that the search returns the global optimum.
Notation: let $f:[a,b]\to\mathbb{R}$ be the objective,
$\varphi=(1+\sqrt5)/2$ the golden ratio, and
$\rho=1/\varphi=\varphi-1\approx 0.618$.

\subsection{Convergence}
\label{app:gss-algorithm}

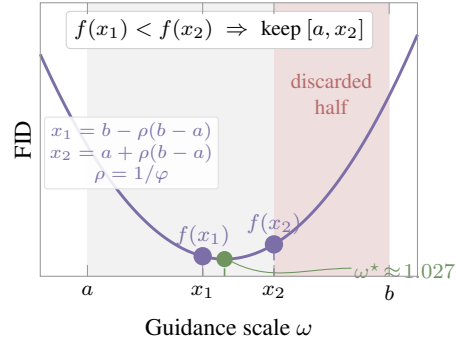
\begin{figure*}[t]
  \centering
  \begin{minipage}[t]{0.48\textwidth}
    \refstepcounter{algorithm}\label{alg:gss}%
    \textbf{Algorithm \thealgorithm:} Golden-section search on
    $\mathrm{FID}(\omega)$~\citep{kiefer1953sequential}.

    \begin{algorithmic}[1]
    \Require Bracket $[a, b]$, tolerance $\varepsilon > 0$
    \State $\rho \gets (\sqrt{5}-1)/2$
    \State $x_{1} \gets b - \rho(b-a)$;\quad $x_{2} \gets a + \rho(b-a)$
    \State $f_{i} \gets \mathrm{FID}(x_{i})$ for $i\in\{1,2\}$
    \While{$b - a > \varepsilon$}
        \If{$f_{1} \leq f_{2}$}
            \Comment{keep $[a, x_{2}]$}
            \State $(b, x_{2}, f_{2}) \gets (x_{2}, x_{1}, f_{1})$
            \State $x_{1} \gets b - \rho(b-a)$;\ $f_{1} \gets \mathrm{FID}(x_{1})$
        \Else
            \Comment{keep $[x_{1}, b]$}
            \State $(a, x_{1}, f_{1}) \gets (x_{1}, x_{2}, f_{2})$
            \State $x_{2} \gets a + \rho(b-a)$;\ $f_{2} \gets \mathrm{FID}(x_{2})$
        \EndIf
    \EndWhile
    \State \Return $(a+b)/2$
    \end{algorithmic}
  \end{minipage}\hfill
  \begin{minipage}[t]{0.48\textwidth}
    \pgfplotsset{width=\linewidth, height=5.6cm}
    \begin{tikzpicture}[baseline=(current bounding box.north)]
      \begin{axis}[
        fidlottery,
        title style={font=\small},
        title={(b) One golden-section step on $\mathrm{FID}(\omega)$},
        xlabel={Guidance scale $\omega$},
        ylabel={FID},
        xmin=0.05, xmax=2.05,
        ymin=7.10, ymax=12.20,
        xtick={0.300, 0.911, 1.289, 1.900},
        xticklabels={$a$, $x_1$, $x_2$, $b$},
        ytick=\empty,
        grid=none,
        axis on top,
        clip=false,
      ]
        \fill[black!5]
          (axis cs:0.300, 7.10) rectangle (axis cs:1.900, 12.20);

        \fill[palCoral, opacity=0.22]
          (axis cs:1.289, 7.10) rectangle (axis cs:1.900, 12.20);
        \node[font=\footnotesize, color=palCoralDark, align=center]
          at (axis cs:1.594, 10.50) {discarded\\half};

        \addplot[domain=0.05:2.05, samples=140, smooth,
                 line width=1.1pt, palLavDark]
          { 7.42 + 4 * (x - 1.027)^2 };

        \draw[dashed, palSageDark, line width=0.6pt]
          (axis cs:1.027, 7.10) -- (axis cs:1.027, 7.42);
        \addplot[only marks, mark=*, mark size=2.6pt,
                 mark options={fill=palSageDark, draw=palSageDark}]
          coordinates {(1.027, 7.42)};
        \node[font=\footnotesize, color=palSageDark, anchor=west,
              inner sep=1pt]
          at (axis cs:1.700, 7.20)
          {$\omega^{\star}\!\approx\!1.027$};
        \draw[palSageDark, line width=0.4pt]
          (axis cs:1.700, 7.25) to[out=180, in=-30] (axis cs:1.040, 7.42);

        \draw[dashed, palLavDark, line width=0.5pt]
          (axis cs:0.911, 7.10) -- (axis cs:0.911, 7.474);
        \draw[dashed, palLavDark, line width=0.5pt]
          (axis cs:1.289, 7.10) -- (axis cs:1.289, 7.695);

        \addplot[only marks, mark=*, mark size=3pt,
                 mark options={fill=palLavDark, draw=palLavDark}]
          coordinates {(0.911, 7.474) (1.289, 7.695)};

        \node[font=\footnotesize, color=palLavDark,
              anchor=south, inner sep=2pt]
          at (axis cs:0.911, 7.55) {$f(x_1)$};
        \node[font=\footnotesize, color=palLavDark,
              anchor=south, inner sep=2pt]
          at (axis cs:1.289, 7.78) {$f(x_2)$};

        \node[font=\footnotesize, fill=white, draw=black!18,
              rounded corners=1.5pt, inner sep=3pt,
              anchor=north]
          at (axis cs:1.000, 12.05)
          {$f(x_1) < f(x_2) \;\Rightarrow\; \text{keep } [a, x_2]$};

        \node[font=\scriptsize, color=palLavDark, align=center, anchor=west,
              fill=white, fill opacity=0.88, text opacity=1,
              inner sep=2pt, rounded corners=1pt, draw=black!12]
          at (axis cs:0.07, 9.40)
          {$x_1 = b - \rho(b-a)$\\$x_2 = a + \rho(b-a)$\\$\rho = 1/\varphi$};
      \end{axis}
    \end{tikzpicture}
  \end{minipage}
  \caption{\textbf{Golden-section search on $\boldsymbol{\mathrm{FID}(\omega)}$
  (referenced briefly from \autoref{sec:results-guidance}).}
  (a)~\autoref{alg:gss}: pseudocode for the bracket-contraction loop
  (the bracket-contraction sequence itself is illustrated in
  \autoref{fig:gss-contraction}).
  (b)~Two interior probes $x_1 = b - \rho(b-a)$ and
  $x_2 = a + \rho(b-a)$ with $\rho = 1/\varphi\!\approx\!0.618$ split
  the bracket $[a, b]$. The side with the larger $f$-value is
  discarded.}
  \label{fig:gss-diagram}
\end{figure*}

\autoref{fig:gss-diagram} states \autoref{alg:gss} and illustrates one
golden-section step. We restate the standard convergence
results~\citep{kiefer1953sequential} and instantiate them on our
setup.

\begin{lemma}[Interval contraction]
\label{lem:contraction}
Let $L_{n}$ denote the bracket width after $n$ iterations of
\autoref{alg:gss}, with $L_{0}=b-a$. Then $L_{n}=\rho^{n} L_{0}$.
\end{lemma}
\begin{proof}
Inspect one iteration. Suppose $f(x_{1})\leq f(x_{2})$ so that the
update is $b\!\leftarrow\!x_{2}$. The new bracket has width
$x_{2}-a=\rho(b-a)=\rho L_{n-1}$. The other branch is symmetric: the
new bracket $[x_{1}, b]$ has width $b-x_{1}=\rho(b-a)$. The recursion
$L_{n}=\rho L_{n-1}$ gives $L_{n}=\rho^{n} L_{0}$ by induction.
\end{proof}

\begin{lemma}[Evaluations to tolerance]
\label{lem:tolerance}
Reaching tolerance $\varepsilon$ from initial bracket width $L_{0}$
takes
$N=2+\bigl\lceil \log(L_{0}/\varepsilon)/\log\varphi\bigr\rceil$
FID evaluations: two for the initial probes, one per iteration
thereafter.
\end{lemma}
\begin{proof}
By \autoref{lem:contraction},
$L_{n}\leq\varepsilon\iff\rho^{n} L_{0}\leq\varepsilon\iff
n\geq\log(L_{0}/\varepsilon)/\log(1/\rho)$. Substituting
$\log(1/\rho)=\log\varphi$ and rounding up gives the iteration count.
Adding the two initial probes yields $N$.
\end{proof}

For our setup, $L_{0}=1$ (CFG bracket $[1,2]$) and $\varepsilon=0.01$
yield $\log(100)/\log\varphi\approx 9.6$, hence $N\approx 12$, against
a measured median of $14$ across the $250$ runs of
\autoref{sec:results-guidance}. The small excess matches a safeguard
in the implementation that adds an extra evaluation when the bracket
width straddles the tolerance. The geometry of one step is illustrated
in \autoref{fig:gss-diagram}. The bracket-contraction sequence is
illustrated in \autoref{fig:gss-contraction}.

\begin{figure}[t]
  \centering
  \pgfplotsset{width=0.65\linewidth, height=5.6cm}
  \begin{tikzpicture}
    \begin{axis}[
      fidlottery,
      title style={font=\small},
      title={Bracket contraction across iterations},
      xlabel={Guidance scale $\omega$},
      ylabel={iteration $n$},
      xmin=0.20, xmax=1.95,
      ymin=-0.7, ymax=6.7,
      xtick={0.5, 1.0, 1.5},
      ytick={0,1,2,3,4,5,6},
      y dir=reverse,
      grid=none,
      axis on top,
      clip=false,
    ]
      \draw[dashed, palSageDark, line width=0.6pt]
        (axis cs:1.027, -0.7) -- (axis cs:1.027, 6.7);
      \node[font=\footnotesize, color=palSageDark, anchor=south]
        at (axis cs:1.027, -0.75) {$\omega^{\star}$};

      \draw[line width=2.2pt, palCoral, opacity=0.55, line cap=round]
        (axis cs:1.289, 0) -- (axis cs:1.900, 0);
      \draw[line width=2.6pt, palLavender, line cap=round]
        (axis cs:0.300, 0) -- (axis cs:1.289, 0);
      \fill[palLavDark] (axis cs:0.911, 0) circle [radius=1.7pt];
      \fill[palLavDark] (axis cs:1.289, 0) circle [radius=1.7pt];

      \draw[line width=2.2pt, palCoral, opacity=0.55, line cap=round]
        (axis cs:0.300, 1) -- (axis cs:0.678, 1);
      \draw[line width=2.6pt, palLavender, line cap=round]
        (axis cs:0.678, 1) -- (axis cs:1.289, 1);
      \fill[palLavDark] (axis cs:0.678, 1) circle [radius=1.7pt];
      \fill[palLavDark] (axis cs:0.911, 1) circle [radius=1.7pt];

      \draw[line width=2.2pt, palCoral, opacity=0.55, line cap=round]
        (axis cs:0.678, 2) -- (axis cs:0.911, 2);
      \draw[line width=2.6pt, palLavender, line cap=round]
        (axis cs:0.911, 2) -- (axis cs:1.289, 2);
      \fill[palLavDark] (axis cs:0.911, 2) circle [radius=1.7pt];
      \fill[palLavDark] (axis cs:1.056, 2) circle [radius=1.7pt];

      \draw[line width=2.2pt, palCoral, opacity=0.55, line cap=round]
        (axis cs:1.145, 3) -- (axis cs:1.289, 3);
      \draw[line width=2.6pt, palLavender, line cap=round]
        (axis cs:0.911, 3) -- (axis cs:1.145, 3);
      \fill[palLavDark] (axis cs:1.055, 3) circle [radius=1.7pt];
      \fill[palLavDark] (axis cs:1.145, 3) circle [radius=1.7pt];

      \draw[line width=2.2pt, palCoral, opacity=0.55, line cap=round]
        (axis cs:1.056, 4) -- (axis cs:1.145, 4);
      \draw[line width=2.6pt, palLavender, line cap=round]
        (axis cs:0.911, 4) -- (axis cs:1.056, 4);
      \fill[palLavDark] (axis cs:1.000, 4) circle [radius=1.7pt];
      \fill[palLavDark] (axis cs:1.056, 4) circle [radius=1.7pt];

      \draw[line width=2.2pt, palCoral, opacity=0.55, line cap=round]
        (axis cs:0.911, 5) -- (axis cs:0.966, 5);
      \draw[line width=2.6pt, palLavender, line cap=round]
        (axis cs:0.966, 5) -- (axis cs:1.056, 5);
      \fill[palLavDark] (axis cs:0.966, 5) circle [radius=1.7pt];
      \fill[palLavDark] (axis cs:1.001, 5) circle [radius=1.7pt];

      \draw[line width=2.2pt, palCoral, opacity=0.55, line cap=round]
        (axis cs:0.966, 6) -- (axis cs:1.000, 6);
      \draw[line width=2.6pt, palLavender, line cap=round]
        (axis cs:1.000, 6) -- (axis cs:1.056, 6);
      \fill[palLavDark] (axis cs:1.000, 6) circle [radius=1.7pt];
      \fill[palLavDark] (axis cs:1.022, 6) circle [radius=1.7pt];

      \draw[-{Latex[length=3.5pt,width=3.5pt]},
            palLavDark, line width=0.5pt, dashed]
        (axis cs:0.911, 0.18) -- (axis cs:0.911, 0.82);
      \node[font=\scriptsize, color=palLavDark, anchor=west]
        at (axis cs:0.925, 0.5) {reused};

      \node[font=\scriptsize, color=palLavDark, anchor=east,
            fill=white, fill opacity=0.88, text opacity=1,
            inner sep=2pt, rounded corners=1pt, draw=black!12]
        at (axis cs:1.93, 6.30)
        {$L_{n+1} = \rho\,L_n$};
    \end{axis}
  \end{tikzpicture}
  \caption{\textbf{Bracket contraction across iterations of
  \autoref{alg:gss}.} Companion to \autoref{fig:gss-diagram}. Kept half
  (lavender) and discarded half (faded coral) of the bracket on the
  illustrative interval $[a_{0},b_{0}]=[0.3,1.9]$. The surviving probe
  of each iteration is reused as one of the two probes for the next
  iteration, so each step past the initial pair costs exactly one new
  FID evaluation. The bracket length contracts by a factor $\rho$ per
  step.}
  \label{fig:gss-contraction}
\end{figure}
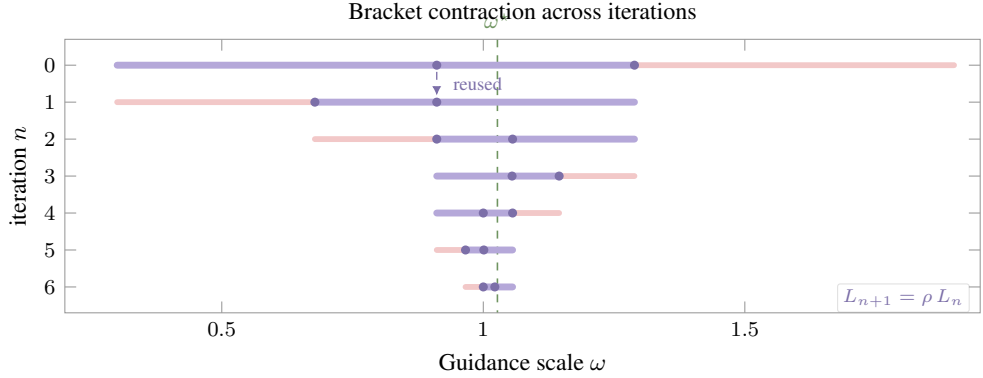

\subsection{Unimodality of \texorpdfstring{$\mathrm{FID}(\mathrm{CFG})$}{FID(CFG)} under a Gaussian feature model}
\label{app:gss-unimodality}

\autoref{alg:gss} returns the global minimum only when $f$ is strictly
unimodal on $[a,b]$. The next proposition derives that
$F(w):=\mathrm{FID}(\mathrm{CFG}=1+w)$ satisfies this condition under
a Gaussian model of the Inception feature distribution, the same model
that underlies the FID definition itself~\citep{heusel2017gans}.

\myparagraph{Setup.}
Conditional flow matching learns scores
$s_\theta(x,c)\!\approx\!\nabla_{x}\log p(x|c)$ and
$s_\theta(x,\emptyset)\!\approx\!\nabla_{x}\log p(x)$. Classifier-free
guidance samples from $p_{w}(x|c)\!\propto\!p(x|c)^{1+w}\,p(x)^{-w}$,
equivalently using the guided score
$\hat s(x,c)=(1+w)s_\theta(x,c)-w\,s_\theta(x,\emptyset)$. Let
$p_{d}=\mathcal{N}(\mu_{d},\Sigma_{d})$ denote the data distribution
in Inception feature space.

\myparagraph{Assumption (A1).}
Both $p(x|c)=\mathcal{N}(\mu_{c},\Sigma)$ and
$p(x)=\mathcal{N}(\mu_{0},\Sigma)$ are Gaussian with the same
covariance $\Sigma$ in feature space.

\begin{proposition}[Unimodality of FID(CFG)]
\label{prop:unimodality}
Under assumption (A1), with $\mu_{c}\neq\mu_{0}$, $F(w)$ is a strictly
convex quadratic on $\mathbb{R}$, hence has a unique global minimum
\begin{equation}
w^{\star}=
-\frac{\langle\mu_{c}-\mu_{d},\,\mu_{c}-\mu_{0}\rangle}{\|\mu_{c}-\mu_{0}\|^{2}},
\label{eq:wstar}
\end{equation}
and is strictly unimodal on every interval.
\end{proposition}
\begin{proof}
Under (A1),
$p_{w}(x|c)\propto\exp\!\bigl(-\tfrac12 (1+w)(x-\mu_{c})^{\top}\Sigma^{-1}(x-\mu_{c})
+\tfrac12 w\,(x-\mu_{0})^{\top}\Sigma^{-1}(x-\mu_{0})\bigr)$.
Collecting quadratic and linear terms in $x$, the precision is
$(1+w)\Sigma^{-1}-w\Sigma^{-1}=\Sigma^{-1}$ (independent of $w$) and
the natural mean is
$\Sigma^{-1}\bigl((1+w)\mu_{c}-w\mu_{0}\bigr)$, so
\begin{equation}
p_{w}(x|c)=\mathcal{N}\!\bigl((1+w)\mu_{c}-w\mu_{0},\,\Sigma\bigr).
\label{eq:pw}
\end{equation}
The Fr\'echet distance between $p_{w}$ and $p_{d}$ in feature space
is
\begin{equation}
F(w)=\bigl\|(1+w)\mu_{c}-w\mu_{0}-\mu_{d}\bigr\|^{2}
+\underbrace{\mathrm{tr}\!\left(\Sigma+\Sigma_{d}-2\bigl(\Sigma\Sigma_{d}\bigr)^{1/2}\right)}_{\text{constant in }w}.
\label{eq:Fw}
\end{equation}
With $u=\mu_{c}-\mu_{0}$ and $b=\mu_{c}-\mu_{d}$, the squared norm
becomes $\|b+wu\|^{2}=\|u\|^{2}w^{2}+2\langle b,u\rangle w+\|b\|^{2}$,
a strictly convex quadratic with leading coefficient $\|u\|^{2}>0$
since $\mu_{c}\neq\mu_{0}$. Convex quadratics are strictly unimodal,
with unique global minimum at
$w^{\star}=-\langle b,u\rangle/\|u\|^{2}$, which gives
\eqref{eq:wstar}. Convexity on $\mathbb{R}$ implies unimodality on
every subinterval.
\end{proof}

\myparagraph{Sanity check.}
If the conditional matches the data ($\mu_{c}=\mu_{d}$), then $b=0$
and $w^{\star}=0$: guidance is unnecessary. The empirical mean
$w^{\star}\!\approx\!0.027$ from \autoref{sec:results-guidance}
implies that the conditional model places its mass within
$\approx\!3\%$ of the guidance direction $\mu_{c}-\mu_{0}$ of the data
mean.

\myparagraph{Beyond (A1).}
When $\Sigma_{c}\neq\Sigma_{0}$, the precision in \eqref{eq:pw}
acquires a $w$-dependent term and the variance contribution to $F(w)$
becomes
$\mathrm{tr}\!\bigl(\Sigma_{w}+\Sigma_{d}-2(\Sigma_{w}\Sigma_{d})^{1/2}\bigr)$
with $\Sigma_{w}^{-1}=(1+w)\Sigma_{c}^{-1}-w\Sigma_{0}^{-1}$.
Convexity of $F$ now requires both summands to be convex. The second
term is convex in $w$ on the domain where $\Sigma_{w}\succ 0$ via the
operator-monotone properties of the matrix geometric
mean~\citep[Theorem~4.1.5]{bhatia2007positive}. The empirical
convergence of all $250$ independent searches to a single tight
cluster ($\sigma_{w^{\star}}=0.045$) is consistent with global
unimodality even outside the equal-covariance regime.

\subsection{Noise robustness}
\label{app:gss-noise}

FID evaluations are themselves stochastic because of the sampling
lottery (\autoref{sec:results-sampling}). Golden-section search
therefore minimises a noisy version $\tilde f(w)=f(w)+\eta(w)$ of the
true objective.

\begin{lemma}[Noise robustness, adapted from \citealt{brent1973algorithms}, §5.5]
\label{lem:noise}
Suppose $f$ is twice continuously differentiable on a neighbourhood of
$w^{\star}$ with $f''(w^{\star})>0$ and the evaluation noise $\eta(w)$
is zero-mean with finite variance $\sigma_{y}^{2}$ across $w$. Run
\autoref{alg:gss} with tolerance $\varepsilon$ on $\tilde f$. The
returned $\hat w$ satisfies
\begin{equation}
\mathbb{E}\bigl[(\hat w-w^{\star})^{2}\bigr]
\,\leq\,\varepsilon^{2}+\frac{2\sigma_{y}}{f''(w^{\star})}+o(1).
\label{eq:noise-bound}
\end{equation}
\end{lemma}
\begin{proof}[Proof sketch]
A second-order Taylor expansion gives
$f(w)\approx f(w^{\star})+\tfrac12 f''(w^{\star})(w-w^{\star})^{2}$
near the optimum. Two probes at $w_{a}, w_{b}$ symmetric about
$w^{\star}$ with $|w_{a}-w^{\star}|=|w_{b}-w^{\star}|=\delta$ give
$f(w_{a})-f(w_{b})=\mathcal{O}(\delta^{3})$, dominated by the noise
difference $\eta_{a}-\eta_{b}\sim\mathcal{N}(0, 2\sigma_{y}^{2})$ when
$\delta\!\lesssim\!(2\sigma_{y}/f''(w^{\star}))^{1/2}$. The algorithm
keeps the wrong half-bracket with constant probability inside that
regime, contributing the second term of \eqref{eq:noise-bound}. The
$\varepsilon^{2}$ term is the deterministic bracket residual.
\end{proof}

\myparagraph{Numerical instantiation.}
With $\sigma_{y}\!\approx\!0.137$ FID (within-seed $\sigma$ from
\autoref{sec:results-sampling}) and an empirical curvature of order
$f''(w^{\star})\!\approx\!\mathcal{O}(10)$ inferred from the spread of
optima around their mean, \eqref{eq:noise-bound} predicts a
noise-induced spread of
$\sqrt{2\!\cdot\!0.137/10}\!\approx\!0.17$, mildly above the observed
$\sigma_{w^{\star}}\!=\!0.045$. The discrepancy suggests the true
curvature near the optimum is steeper than the back-of-envelope
estimate, i.e.\ the FID landscape around the optimal CFG is sharper
than the optimum-spread alone would predict.

\section{Replication across DINOv2 FID and Inception PRDC}
\label{app:supp-metrics}

The main paper anchors the seed-lottery analysis on Inception FID. This
appendix re-runs each of the five \autoref{sec:results} questions on
five complementary metrics: DINOv2 FID~\citep{stein2023exposing} and
Inception precision, recall, density,
coverage~\citep{kynkaanniemi2019improved,naeem2020reliable}. The same
panels of \autoref{sec:setup} are reused, so the columns of
\autoref{tab:supp-overview}, \autoref{tab:supp-disentangle},
\autoref{tab:supp-cov-bands} and \autoref{tab:supp-speedup} parallel
the main-paper tables row-for-row. Three findings emerge.
(i)~DINOv2 FID \emph{strengthens} the seed-lottery narrative.
(ii)~Inception precision, density, and coverage track FID closely on
every angle and often sharpen the conclusions. (iii)~Inception recall
is sampling-dominated on this panel and behaves anomalously throughout.

\subsection{Metric definitions}
\label{app:supp-defs}

\noindent\emph{Six metrics, two families: a Fr\'echet distance computed
in two feature spaces (Inception-V3 and DINOv2), and four PRDC
metrics that compare the two feature sets through $k$-nearest-neighbour
balls.}

\myparagraph{Notation.}
Let $R\!=\!\{r_{1},\!\ldots,\!r_{N}\}\!\subset\!\mathbb{R}^{d}$ be the
\emph{real} feature set (the same $N\!=\!50$k ImageNet train images
on every evaluation in this paper) and
$G\!=\!\{g_{1},\!\ldots,\!g_{N}\}\!\subset\!\mathbb{R}^{d}$ the
\emph{generated} feature set, $50$k samples drawn from a trained SiT.
A frozen feature extractor
$\phi\!:\!\mathcal{X}\!\to\!\mathbb{R}^{d}$ maps both sets into a common
space. The extractor is Inception-V3 (Inception~FID and the four
Inception PRDC metrics) or DINOv2 (DINOv2~FID). For any set $S$ and
$x\!\in\!S$, write $\rho_{k}^{S}(x)\!=\!\|x-x^{(k)}\|$ for the distance
to the $k$-th nearest neighbour of $x$ inside $S$, with $k\!=\!3$ as in
the standard \texttt{prdc} implementation.

\myparagraph{Fr\'echet distances (Inception~FID, DINOv2~FID).}
Both fit a Gaussian to each feature set and take the Fr\'echet distance
between the two Gaussians~\citep{heusel2017gans}:
\begin{equation}
\mathrm{FD}(R, G)\!=\!\|\mu_{R}-\mu_{G}\|^{2}
+\mathrm{tr}\!\left(\Sigma_{R}+\Sigma_{G}-2(\Sigma_{R}\Sigma_{G})^{1/2}\right),
\label{eq:fd}
\end{equation}
where $\mu_{R}, \Sigma_{R}$ and $\mu_{G}, \Sigma_{G}$ are the empirical
mean and covariance of $R$ and $G$ in feature space.
\emph{Inception~FID} uses Inception-V3 features. \emph{DINOv2~FID} uses
DINOv2 features~\citep{stein2023exposing}. The two metrics measure
different distortions and are not comparable in absolute units. DINOv2
features are higher-dimensional and on a different scale, so DINOv2~FID
values are roughly an order of magnitude larger than Inception~FID at
matched generative quality. Both are unbounded above, and lower is better.

\myparagraph{Improved Precision and Recall.}
\citet{kynkaanniemi2019improved} replace the global Fr\'echet matching
with a local manifold-membership test built from $k$-NN balls.
Precision asks ``does each generated point fall inside any real-set
$k$-NN ball?'' (a fidelity question with balls anchored on $R$).
Recall asks ``does each real point fall inside any generated-set $k$-NN
ball?'' (a diversity question with balls anchored on $G$):
\begin{equation}
\mathrm{Precision}\!=\!\frac{1}{|G|}\!\sum_{g \in G}\!\mathbf{1}\!\bigl[\exists\,r\!\in\!R: g\!\in\!B(r, \rho_{k}^{R}(r))\bigr],
\quad
\mathrm{Recall}\!=\!\frac{1}{|R|}\!\sum_{r \in R}\!\mathbf{1}\!\bigl[\exists\,g\!\in\!G: r\!\in\!B(g, \rho_{k}^{G}(g))\bigr],
\label{eq:pr}
\end{equation}
where $B(x, r)\!=\!\{y\!:\!\|y-x\|\!\leq\!r\}$. Both metrics lie in
$[0, 1]$ and higher is better. The two are mirror duals at the equation
level but \emph{not} at the implementation level: precision uses the
$\rho_{k}^{R}$ ball geometry of the fixed real set, while recall uses
the $\rho_{k}^{G}$ ball geometry of the generated set, which is
re-drawn at every sampling seed. The anchor distinction will recur.

\myparagraph{Density and Coverage.}
\citet{naeem2020reliable} note that recall's generated-anchored balls
are sensitive to outlier generated points (one stray $g$ can balloon
$\rho_{k}^{G}(g)$ and inflate the metric) and propose two more robust
variants. Density softens precision by counting how many real-set balls
each generated point falls inside, normalised by $k$:
\begin{equation}
\mathrm{Density}\!=\!\frac{1}{k\,|G|}
\sum_{g \in G}\sum_{r \in R}\!\mathbf{1}\!\bigl[g\!\in\!B(r, \rho_{k}^{R}(r))\bigr].
\label{eq:density}
\end{equation}
Coverage answers recall's diversity question (\emph{is each real
sample covered by some generated point?}) but with the balls
anchored on $R$ instead of $G$:
\begin{equation}
\mathrm{Coverage}\!=\!\frac{1}{|R|}
\sum_{r \in R}\!\mathbf{1}\!\bigl[\exists\,g\!\in\!G: g\!\in\!B(r, \rho_{k}^{R}(r))\bigr].
\label{eq:coverage}
\end{equation}
Coverage lies in $[0, 1]$. Density is non-negative and unbounded above
(empirical values rarely exceed $\!\approx\!1.5$ on this panel, reaching
$1.32$ in the guided panel of \autoref{tab:supp-guidance}). Higher is
better for both. The conceptual mapping is therefore precision
$\leftrightarrow$ density (fidelity axis) and recall $\leftrightarrow$
coverage (diversity axis). The anchor mapping is precision, density,
and coverage all real-anchored, recall alone generated-anchored.

\myparagraph{Anchor structure.}
\autoref{tab:supp-anchor-structure} consolidates the $2{\times}2$
layout, and \autoref{fig:supp-prdc-anchor} draws the four metrics on a
2D toy dataset so the anchor distinction is visible at a glance. Recall
is the only metric whose ball geometry depends on which
$50$k-sample draw of $G$ was used: every other PRDC metric tests
generated points against the same fixed real-set ball geometry on every
evaluation. \citet{naeem2020reliable} introduced coverage precisely to
recover the diversity question from a stable real-anchored ball
geometry, so the structural difference between recall and coverage is
not an accident. It is the design intent. \autoref{app:supp-overview}
then quantifies the consequence: recall is the only metric whose
within-seed $\sigma$ exceeds the binomial estimator floor, by a factor
of $\!\approx\!2$ on this panel.

\input{figures/figS_prdc_diagram.tex}

\begin{table}[t]
  \centering
  \caption{\textbf{The four PRDC metrics in a $\boldsymbol{2\!\times\!2}$
  layout.} Rows separate the fidelity question (does each generated
  sample lie near the real manifold?) from the diversity question (does
  each real sample have a nearby generated neighbour?). Columns
  separate where each metric anchors its $k$-NN balls, on the
  fixed real set $R$ or on the per-evaluation generated set $G$.
  Recall is the only generated-anchored metric.
  \citet{naeem2020reliable} introduced coverage as the
  real-anchored answer to recall's diversity question.}
  \label{tab:supp-anchor-structure}
  \small
  \begin{tabular}{lll}
    \toprule
    & balls anchored on $R$ (fixed) & balls anchored on $G$ (per-seed) \\
    \midrule
    fidelity (test $g\!\in\!\mathcal{B}^{?}$)  & Precision, Density & --- \\
    diversity (test $r\!\in\!\mathcal{B}^{?}$) & Coverage           & Recall \\
    \bottomrule
  \end{tabular}
\end{table}

\subsection{The training lottery dominates for every fidelity metric, not for recall}
\label{app:supp-overview}

\noindent\emph{The $3.2{\times}$ training-vs-sampling asymmetry of
\autoref{sec:results-sampling} grows to $4.8{\times}$ under DINOv2 FID,
shrinks to $\approx\!1{\times}$ under Inception precision/density/coverage,
and inverts to $0.28{\times}$ under Inception recall.}

\myparagraph{Setup.}
The same $25\!\times\!10$ SiT-B/2 panel of
\autoref{sec:results-sampling} is rescored under each metric. Per-seed
means are taken across the ten sampling seeds. $\sigma_{\text{between}}$
is the standard deviation of the $25$ per-seed means, and $\sigma_{\text{within}}$
is the cross-seed average of the per-seed standard deviation.

\myparagraph{DINOv2 FID amplifies the asymmetry.}
The between-to-within ratio rises from $3.19{\times}$ on Inception FID to
$4.82{\times}$ on DINOv2 FID (\autoref{tab:supp-overview}). The relative
floor tightens from $\mathrm{CoV}_{\text{between}}\!=\!1.26\%$ to
$0.86\%$. \autoref{fig:supp-overview-dino} visualises the consequence: the
black per-seed mean ticks stagger over a wider multiple of the violin
height than under Inception FID, making the dominance of the training
lottery even more visible. Switching the feature extractor moves the
needle in the same direction the main paper argues.

\myparagraph{Precision, density, and coverage put the two lotteries near
parity.}
The ratio drops to roughly $1{\times}$ on the three real-anchored PRDC
metrics ($1.14$ for precision, $1.41$ for density, $1.12$ for
coverage, \autoref{tab:supp-overview}). Multi-sampling-seed CIs cover
roughly half of the true seed-induced envelope on these three metrics,
so the sampling-only confidence interval is no longer asymptotically
irrelevant the way it is for FID. \autoref{fig:supp-overview-precision}
shows the parity: violin heights and per-seed-mean staircases occupy
comparable vertical extents.

\myparagraph{Recall inverts the asymmetry.}
On Inception recall the ratio is $0.28{\times}$: the within-seed sampling
$\sigma$ is over three times the between-seed training $\sigma$.
\autoref{fig:supp-overview-recall} shows individual violins taller than the
spread of the per-seed means. Recall is the only metric in the bundle
for which the sampling-seed CI is the right CI to report on a fixed
trained model, and the only metric for which the main-paper claim
inverts.

\myparagraph{Why recall inverts: the $\boldsymbol{k}$-NN balls move with
the sampling seed.}
The asymmetry traces to where each PRDC metric anchors its
$k$-nearest-neighbour structure. Improved precision, density, and
coverage~\citep{kynkaanniemi2019improved,naeem2020reliable} build their
$k$-NN balls on the \emph{real} feature set, which is the same
$50$k-image ImageNet train manifold across all $250$ evaluations: changing
the sampling seed only redraws the $50$k generated points tested against
an unchanged ball structure, so the $50$k indicator outcomes are nearly
independent and the within-seed $\sigma$ is bounded by the binomial
estimator floor $\sqrt{p(1-p)/N}$. Improved recall builds its balls on
the \emph{generated} feature set, which is re-drawn every sampling seed:
the entire ball geometry shifts, and the indicators of ``real point
$r_{i}$ is covered'' are correlated across $i$ through the shared moving
balls, so when the geometry shifts many real points flip together. The
effective sample size is much smaller than $50$k.

\myparagraph{Empirical check.}
\autoref{tab:supp-recall-binomial} compares each within-seed $\sigma$
with the binomial floor $\sqrt{p(1-p)/N}$ at $N\!=\!50\,000$.
Precision and coverage land at or below the floor. Recall sits at
roughly twice the floor, the only metric in the bundle whose
sampling-seed estimator carries non-negligible
correlated-ball-geometry variance on top of the binomial term. The
same panel saturates the recall axis at $\mu\!=\!0.314$
(\autoref{tab:supp-overview}), so the converged SiT-B/2 networks all
reach similar coverage of the real distribution: the between-seed
$\sigma$ is small \emph{and} the within-seed $\sigma$ is structurally
inflated, which together produce the denominator inversion.

\begin{table}[t]
  \centering
  \caption{\textbf{Within-seed $\boldsymbol{\sigma}$ versus the binomial
  estimator floor for each PRDC metric.} The floor is $\sqrt{p(1-p)/N}$
  with $N\!=\!50\,000$ and $p\!=\!\mu$. The four metrics differ in where
  they anchor their $k$-NN balls (real vs. generated): only recall is
  generated-anchored, and only recall sits well above the binomial floor.}
  \label{tab:supp-recall-binomial}
  \small
  \begin{tabular}{lrrrl}
    \toprule
    Metric & $\mu$ & $\sigma_{\text{within}}$ &
    $\sqrt{p(1-p)/N}$ & ratio (obs / floor)\\
    \midrule
    Inception Precision $\uparrow$  & 0.485 & 0.00284 & 0.00224 & 1.27$\times$ (real-anchored)\\
    Inception Recall $\uparrow$     & 0.314 & 0.00438 & 0.00208 & \textbf{2.11$\times$ (generated-anchored)}\\
    Inception Density $\uparrow$    & 0.442 & 0.00374 & 0.00222 & 1.68$\times$ (real-anchored)\\
    Inception Coverage $\uparrow$   & 0.224 & 0.00156 & 0.00186 & 0.84$\times$ (real-anchored)\\
    \bottomrule
  \end{tabular}
\end{table}


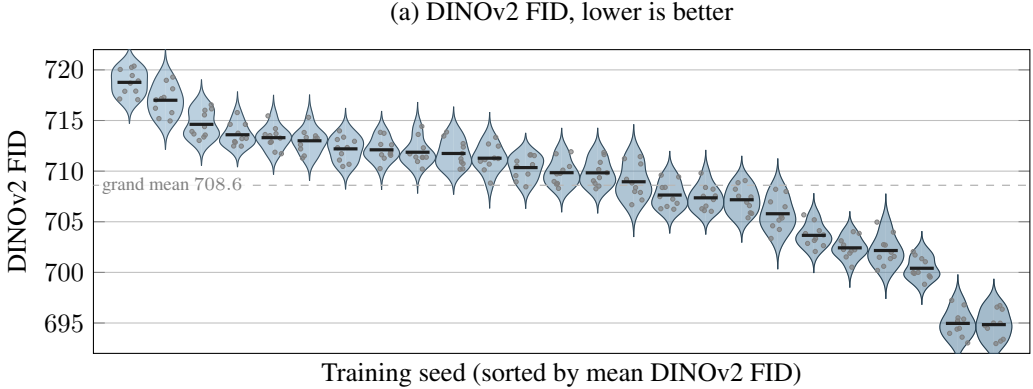
\begin{figure}[t]
  \centering
  \pgfplotsset{width=\linewidth, height=5.6cm}
  \begin{tikzpicture}
    \violinsetoptions[scaled]{
      title={(a) DINOv2 FID, lower is better},
      xlabel={Training seed (sorted by mean DINOv2 FID)},
      ylabel={DINOv2 FID},
      xmin=0.0, xmax=26.0,
      ymin=692, ymax=722,
      xtick={1,5,10,15,20,25},
      ytick={695,700,705,710,715,720},
      ymajorgrids,
    }
    \violinplotwholefile[%
      col sep=tab,
      primary color=palSlateDark,
      secondary color=palSlate,
      indexes={s01,s02,s03,s04,s05,s06,s07,s08,s09,s10,s11,s12,s13,s14,s15,s16,s17,s18,s19,s20,s21,s22,s23,s24,s25},
      labels={,,,,,,,,,,,,,,,,,,,,,,,,},
      spacing=1.0,
    ]{figures/fig_data/figS_violins_dinofid.tsv}
    \begin{axis}[
      xmin=0.0, xmax=26.0, ymin=692, ymax=722,
      axis line style={draw=none}, tick style={draw=none},
      xticklabels={,,}, yticklabels={,,},
      xmajorticks=false, ymajorticks=false,
      axis on top, clip=false,
    ]
      \draw[dashed, color=black!28, line width=0.5pt]
        (axis cs:0.0,708.6) -- (axis cs:26.0,708.6);
      \node[anchor=west, font=\scriptsize, color=black!50,
            fill=white, fill opacity=0.85, text opacity=1, inner sep=1pt]
            at (axis cs:0.15,708.6) {grand mean $708.6$};
      \addplot+[only marks, mark=*, mark size=0.8pt,
        mark options={fill=black!50, draw=black!50, line width=0pt, fill opacity=0.65},
      ] table[col sep=tab, x=xj, y=y]
        {figures/fig_data/figS_strip_dinofid.tsv};
      \addplot+[only marks, mark=-, mark size=4.5pt,
        mark options={draw=black!88, line width=1.2pt},
      ] table[col sep=tab, x=x, y=mean]
        {figures/fig_data/figS_means_dinofid.tsv};
    \end{axis}
  \end{tikzpicture}
  \caption{\textbf{Seed lottery on the same SiT-B/2 panel under DINOv2
  FID.} Each violin is one of $25$ trained models, sorted left-to-right
  by per-seed mean. Violin shape traces the within-seed sampling
  distribution. The black tick is the per-seed mean. The training-seed
  mean ticks span $694.8$ to $718.8$ around grand mean $708.6$. The
  between-to-within ratio is
  $\sigma_{\text{between}}/\sigma_{\text{within}}\!=\!4.82{\times}$, against
  $3.19{\times}$ for Inception FID, so the staircase of mean ticks
  stretches over a wider multiple of the violin height than in
  \autoref{fig:overview}.}
  \label{fig:supp-overview-dino}
\end{figure}

\begin{figure}[t]
  \centering
  \pgfplotsset{width=\linewidth, height=5.6cm}
  \begin{tikzpicture}
    \violinsetoptions[scaled]{
      title={(b) Inception Precision, higher is better},
      xlabel={Training seed (sorted by mean Inception Precision)},
      ylabel={Inception Precision (x100)},
      xmin=0.0, xmax=26.0,
      ymin=47.5, ymax=49.8,
      xtick={1,5,10,15,20,25},
      ytick={47.8,48.2,48.6,49.0,49.4},
      ymajorgrids,
    }
    \violinplotwholefile[%
      col sep=tab,
      primary color=palLavDark,
      secondary color=palLavender,
      indexes={s01,s02,s03,s04,s05,s06,s07,s08,s09,s10,s11,s12,s13,s14,s15,s16,s17,s18,s19,s20,s21,s22,s23,s24,s25},
      labels={,,,,,,,,,,,,,,,,,,,,,,,,},
      spacing=1.0,
    ]{figures/fig_data/figS_violins_incprecision.tsv}
    \begin{axis}[
      xmin=0.0, xmax=26.0, ymin=47.5, ymax=49.8,
      axis line style={draw=none}, tick style={draw=none},
      xticklabels={,,}, yticklabels={,,},
      xmajorticks=false, ymajorticks=false,
      axis on top, clip=false,
    ]
      \draw[dashed, color=black!28, line width=0.5pt]
        (axis cs:0.0,48.49) -- (axis cs:26.0,48.49);
      \node[anchor=west, font=\scriptsize, color=black!50,
            fill=white, fill opacity=0.85, text opacity=1, inner sep=1pt]
            at (axis cs:0.15,48.49) {grand mean $0.485$};
      \addplot+[only marks, mark=*, mark size=0.8pt,
        mark options={fill=black!50, draw=black!50, line width=0pt, fill opacity=0.65},
      ] table[col sep=tab, x=xj, y=y]
        {figures/fig_data/figS_strip_incprecision.tsv};
      \addplot+[only marks, mark=-, mark size=4.5pt,
        mark options={draw=black!88, line width=1.2pt},
      ] table[col sep=tab, x=x, y=mean]
        {figures/fig_data/figS_means_incprecision.tsv};
    \end{axis}
  \end{tikzpicture}
  \caption{\textbf{Seed lottery under Inception Precision.} Y-axis values
  are multiplied by $100$ for readability. Per-seed means span $0.480$ to
  $0.491$ around grand mean $0.485$. Violin height (within-seed sampling
  spread) and the staircase of mean ticks (between-seed training spread)
  are comparable, since
  $\sigma_{\text{between}}/\sigma_{\text{within}}\!=\!1.14{\times}$. A
  multi-sampling-seed CI therefore covers roughly half of the seed-induced
  envelope on precision, in contrast with FID where it covers a quarter.}
  \label{fig:supp-overview-precision}
\end{figure}
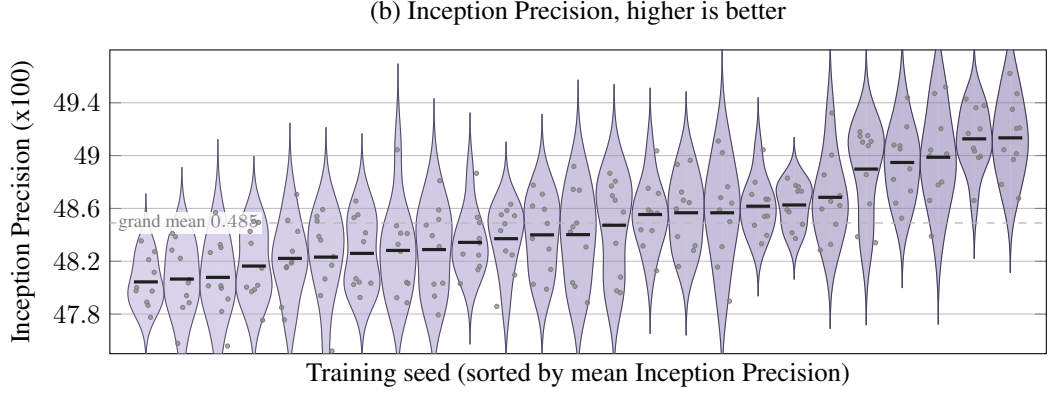

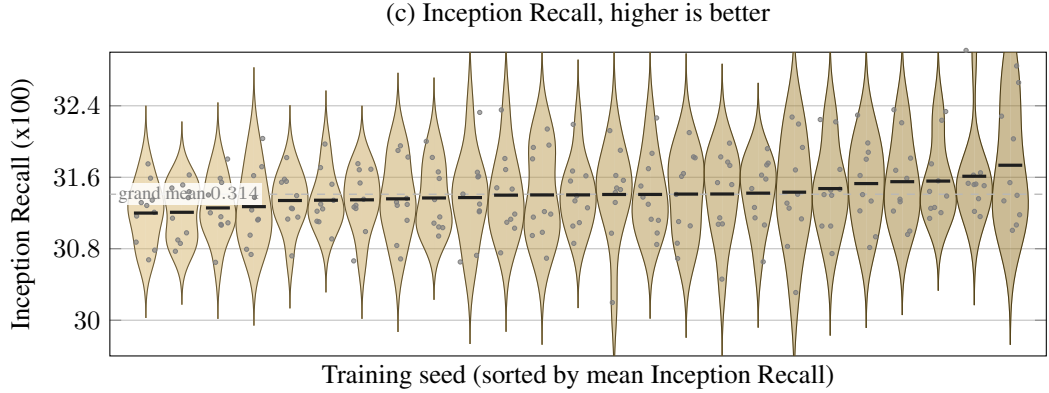
\begin{figure}[t]
  \centering
  \pgfplotsset{width=\linewidth, height=5.6cm}
  \begin{tikzpicture}
    \violinsetoptions[scaled]{
      title={(c) Inception Recall, higher is better},
      xlabel={Training seed (sorted by mean Inception Recall)},
      ylabel={Inception Recall (x100)},
      xmin=0.0, xmax=26.0,
      ymin=29.6, ymax=33.0,
      xtick={1,5,10,15,20,25},
      ytick={30.0,30.8,31.6,32.4},
      ymajorgrids,
    }
    \violinplotwholefile[%
      col sep=tab,
      primary color=palOchreDark,
      secondary color=palOchre,
      indexes={s01,s02,s03,s04,s05,s06,s07,s08,s09,s10,s11,s12,s13,s14,s15,s16,s17,s18,s19,s20,s21,s22,s23,s24,s25},
      labels={,,,,,,,,,,,,,,,,,,,,,,,,},
      spacing=1.0,
    ]{figures/fig_data/figS_violins_increcall.tsv}
    \begin{axis}[
      xmin=0.0, xmax=26.0, ymin=29.6, ymax=33.0,
      axis line style={draw=none}, tick style={draw=none},
      xticklabels={,,}, yticklabels={,,},
      xmajorticks=false, ymajorticks=false,
      axis on top, clip=false,
    ]
      \draw[dashed, color=black!28, line width=0.5pt]
        (axis cs:0.0,31.41) -- (axis cs:26.0,31.41);
      \node[anchor=west, font=\scriptsize, color=black!50,
            fill=white, fill opacity=0.85, text opacity=1, inner sep=1pt]
            at (axis cs:0.15,31.41) {grand mean $0.314$};
      \addplot+[only marks, mark=*, mark size=0.8pt,
        mark options={fill=black!50, draw=black!50, line width=0pt, fill opacity=0.65},
      ] table[col sep=tab, x=xj, y=y]
        {figures/fig_data/figS_strip_increcall.tsv};
      \addplot+[only marks, mark=-, mark size=4.5pt,
        mark options={draw=black!88, line width=1.2pt},
      ] table[col sep=tab, x=x, y=mean]
        {figures/fig_data/figS_means_increcall.tsv};
    \end{axis}
  \end{tikzpicture}
  \caption{\textbf{Seed lottery under Inception Recall (the inversion
  case).} Y-axis values are multiplied by $100$ for readability. Per-seed
  means span $0.312$ to $0.317$ around grand mean $0.314$. Each violin is
  taller than the staircase of mean ticks:
  $\sigma_{\text{between}}/\sigma_{\text{within}}\!=\!0.28{\times}$, an
  inversion of the FID asymmetry. On Inception recall the right CI to
  report on a fixed trained model is the sampling-only one. This is the
  opposite recommendation from FID and from the other PRDC metrics
  (\autoref{tab:supp-overview}).}
  \label{fig:supp-overview-recall}
\end{figure}

\begin{table}[t]
  \centering
  \caption{\textbf{Sampling vs. training lottery on the
  $\boldsymbol{25\!\times\!10}$ SiT-B/2 panel.} Companion to
  \autoref{fig:overview} and the headline of \autoref{sec:results-sampling}.
  $\sigma_{\text{between}}$ is the standard deviation of the $25$
  per-seed means. $\sigma_{\text{within}}$ is the within-seed sampling
  standard deviation, averaged across the $25$ training seeds.
  $\mathrm{CoV}_{\text{between}}\!=\!\sigma_{\text{between}}/|\mu|$.
  Higher-better metrics are marked $\uparrow$, lower-better $\downarrow$.
  Bold marks the largest ratio (most training-lottery-dominated metric)
  and the smallest CoV in each direction column.}
  \label{tab:supp-overview}
  \small
  \begin{tabular}{lrrrrr}
    \toprule
    Metric & $\mu$ & $\sigma_{\text{between}}$ & $\sigma_{\text{within}}$ &
    $\mathrm{CoV}_{\text{between}}$ (\%) &
    $\sigma_{\text{between}}/\sigma_{\text{within}}$\\
    \midrule
    Inception FID $\downarrow$           & 34.74        & 0.438     & 0.137     & 1.26 & 3.19$\times$ \\
    DINOv2 FID $\downarrow$              & 708.6        & 6.106     & 1.268     & \textbf{0.86} & \textbf{4.82$\times$} \\
    Inception Precision $\uparrow$       & 0.485        & 3.24e-3   & 2.84e-3   & 0.67 & 1.14$\times$ \\
    Inception Recall $\uparrow$          & 0.314        & 1.21e-3   & 4.38e-3   & \textbf{0.39} & 0.28$\times$ \\
    Inception Density $\uparrow$         & 0.442        & 5.26e-3   & 3.74e-3   & 1.19 & 1.41$\times$ \\
    Inception Coverage $\uparrow$        & 0.224        & 1.75e-3   & 1.56e-3   & 0.78 & 1.12$\times$ \\
    \bottomrule
  \end{tabular}
\end{table}

\subsection{Flow-matching noise dominates for the fidelity metrics}
\label{app:supp-disentangle}

\noindent\emph{The \emph{noise > init > data} hierarchy of
\autoref{sec:results-decomposition} carries over to DINOv2 FID and to
Inception precision, density, and coverage with the same sub-additive
combination. Recall is too sampling-dominated to disentangle.}

\myparagraph{Setup.}
The same four single-source conditions of
\autoref{sec:results-decomposition} (\textsc{vary-noise},
\textsc{vary-init}, \textsc{vary-data}, \textsc{vary-noisedata}) are
rescored under each metric. \autoref{tab:supp-disentangle} reports the
per-condition between-seed $\sigma$ as a percentage of the fully-stochastic
baseline ($\sigma_{\text{vary-all}}$ from
\autoref{tab:supp-overview}).

\myparagraph{The hierarchy holds for the four well-behaved metrics.}
Inception FID, DINOv2 FID, Inception Precision, Inception Density, and
Inception Coverage all give the same ranking
\textsc{vary-noise} $>$ \textsc{vary-init} $>$ \textsc{vary-data}. The
flow-matching noise alone reproduces about three-quarters of the
baseline $\sigma$ on every metric. Init alone recovers about
two-thirds, and data order alone about half (per-metric numbers in
\autoref{tab:supp-disentangle}). Adding the three single-source
$\sigma$ values in quadrature overshoots $\sigma_{\text{vary-all}}$ by
$7$--$19\%$, mirroring the $14\%$ overshoot on Inception FID. Three
takeaways follow. (i)~The noise dominance is a property of the loss
formulation, not of the feature extractor. (ii)~The sub-additivity is
a stable two-digit number across feature spaces. (iii)~Init
quality is the second-cheapest knob in every case.

\myparagraph{Recall fails to decompose.}
On Inception Recall, three of the four single-source conditions produce a
between-seed $\sigma$ \emph{larger} than the vary-all baseline
($112$--$126\%$), and the naive sum of squares overshoots by $90\%$.
\autoref{tab:supp-overview} explained why: the between-seed component
is itself smaller than the within-seed component, so each
per-condition $\sigma$ estimate is dominated by the same sampling
jitter rather than isolating its source. The variance-decomposition
framework requires the between-seed component to dominate and
therefore does not apply to recall on this panel.

\begin{table}[t]
  \centering
  \caption{\textbf{Variance decomposition across feature spaces.}
  Companion to \autoref{tab:variance-decomposition}. Per-condition
  $\sigma_{\text{between}}$ as a percentage of the fully-stochastic
  $\sigma_{\text{vary-all}}$, plus the naive sum-of-squares overshoot
  $\sqrt{\sigma_{\text{noise}}^{2}+\sigma_{\text{init}}^{2}+\sigma_{\text{data}}^{2}}/\sigma_{\text{vary-all}}-1$.
  Underlines mark the leading single source per metric.}
  \label{tab:supp-disentangle}
  \small
  \begin{tabular}{lrrrrr}
    \toprule
    Metric & vary-noise (\%) & vary-init (\%) & vary-data (\%) & vary-noisedata (\%) & overshoot (\%) \\
    \midrule
    Inception FID $\downarrow$       & \underline{76.3} & 67.0 & 49.8 & 59.2 & +13.1 \\
    DINOv2 FID $\downarrow$          & \underline{76.9} & 72.4 & 55.1 & 59.7 & +19.2 \\
    Inception Precision $\uparrow$   & \underline{66.0} & 64.4 & 54.9 & 57.9 & +7.3  \\
    Inception Density $\uparrow$     & \underline{73.3} & 64.0 & 53.1 & 57.7 & +10.9 \\
    Inception Coverage $\uparrow$    & \underline{72.3} & 69.2 & 59.4 & 56.1 & +16.4 \\
    \midrule
    Inception Recall $\uparrow$      & 124.1 & 90.2 & 112.6 & 126.1 & +90.3 \\
    \bottomrule
  \end{tabular}
\end{table}

\subsection{Golden-section search transfers cleanly to DINOv2 FID}
\label{app:supp-guidance}

\noindent\emph{The CoV-halving claim of \autoref{sec:results-guidance} is
specific to Inception FID. On DINOv2 FID the relative noise floor is
already low ($0.86\%$ unguided) and stays at $0.94\%$ at the
Inception-FID-optimal CFG. The recovered CFG also shifts every PRDC
metric's operating point by a large amount.}

\myparagraph{Setup.}
The same per-(training, sampling) golden-section search of
\autoref{alg:gss} is rerun on the same panel. The metric panel of
each guided sample is then rescored under all five complementary
metrics. The search optimises Inception FID. The supp metrics are
\emph{readouts at the Inception-FID-optimal CFG}, not separately
optimised.

\myparagraph{DINOv2 FID's floor is not improved by Inception-FID guidance.}
GS-FID drops Inception FID's $\mathrm{CoV}_{\text{between}}$ from
$1.26\%$ to $0.67\%$ (\autoref{sec:results-guidance}). On DINOv2 FID,
$\mathrm{CoV}_{\text{between}}$ moves from $0.86\%$ to $0.94\%$ on the
same CFG-selected samples, statistically indistinguishable. The
$0.86\%$ unguided DINOv2-FID floor already sits below the $1$--$2\%$
band, so the $\approx\!14$ extra FID evaluations per cell that
GS-FID costs do not buy a tighter DINOv2-FID estimate. A practitioner
who reports DINOv2 FID does not need a per-seed CFG search to claim
sub-$1\%$ CoV.

\myparagraph{The recovered CFG sets the PRDC operating point.}
Classifier-free guidance trades recall for fidelity, and the
golden-section optimum picks an operating point that strongly favours
fidelity. On Inception precision, density, and coverage the mean
roughly doubles or triples under GS-FID, while recall drops by more
than a factor of two (\autoref{tab:supp-guidance}). The PRDC numbers
under GS-FID therefore measure a \emph{different operating point}
than the unguided numbers (a high-fidelity, low-recall regime)
and should not be compared with unguided PRDC across studies.

\begin{table}[t]
  \centering
  \caption{\textbf{Effect of golden-section CFG selection on each metric.}
  Companion to \autoref{sec:results-guidance}. ``Unguided'' is the
  $25\!\times\!10$ panel. ``GS-FID'' is the same panel rescored at the
  per-(training, sampling) CFG that minimises Inception FID. Means
  shift sharply on PRDC because guidance reweights fidelity vs.\ recall.
  The CoV column shows that the relative noise floor changes
  substantially only for Inception FID.}
  \label{tab:supp-guidance}
  \small
  \begin{tabular}{lrrrr}
    \toprule
    Metric & $\mu$ unguided & $\mu$ GS-FID & CoV unguided (\%) & CoV GS-FID (\%) \\
    \midrule
    Inception FID $\downarrow$       & 34.74  & 7.42   & 1.26 & \textbf{0.67} \\
    DINOv2 FID $\downarrow$          & 708.6  & 289.8  & 0.86 & 0.94 \\
    Inception Precision $\uparrow$   & 0.485  & 0.843  & 0.67 & 0.24 \\
    Inception Recall $\uparrow$      & 0.314  & 0.135  & 0.39 & 1.16 \\
    Inception Density $\uparrow$     & 0.442  & 1.320  & 1.19 & 0.50 \\
    Inception Coverage $\uparrow$    & 0.224  & 0.295  & 0.78 & 0.17 \\
    \bottomrule
  \end{tabular}
\end{table}

\subsection{The relative noise floor is metric-dependent but scale-invariant}
\label{app:supp-scaling}

\noindent\emph{The $1$--$2\%$ CoV band of \autoref{sec:results-scaling}
covers DINOv2 FID, Inception precision, density, and coverage at every
checkpoint of every model size. Recall keeps a $1$--$1.8\%$ CoV on
SiT-B/L without shrinking with compute.}

\myparagraph{Setup.}
The same $200$k--$2$M scaling sweep across SiT-S/B/L/XL with the clean
seed sets of \autoref{tab:scaling-sweep} is rescored under each metric.
\autoref{fig:supp-cov-bands} plots $\mathrm{CoV}_{\text{between}}$ against
training step in one mini-panel per metric.
\autoref{tab:supp-cov-bands} reports the min/median/max CoV over the
$19$ checkpoints per (metric, model) cell, plus the $\sigma$-shrink
factor between $200$k and $2$M.

\myparagraph{DINOv2 FID is tighter than Inception FID at every scale.}
The DINOv2 FID band sits inside the Inception FID band across the
$76$ cells (per-size ranges in \autoref{tab:supp-cov-bands}).
$\sigma$-shrink factors are similar for the two extractors, so the
absolute reduction in spread with compute is comparable in both
feature spaces. The difference is on the level of the floor itself.
DINOv2 FID is the tighter benchmark protocol on this family at every
model size and every checkpoint past $200$k.

\myparagraph{Precision, density, and coverage hit the lowest floors at scale.}
Inception precision, density, and coverage on SiT-XL all settle below
$0.4\%$ CoV by $1$M steps. The median CoVs of precision and coverage
on SiT-XL drop below $0.3\%$, with maxima below $0.7\%$. Three
fidelity-axis metrics on a large model produce the tightest
seed-induced spread on this panel, sitting an order of magnitude
below the $3$--$4\%$ headline gain typical of recipe-comparison
studies.

\myparagraph{Recall does not respond to compute.}
On SiT-B and SiT-L, recall $\sigma$ \emph{grows} between $200$k and $2$M
($\sigma$-shrink below $1{\times}$, \autoref{tab:supp-cov-bands}). On
SiT-XL it stays flat. The $1$--$1.8\%$ recall CoV band is therefore a
property of the metric, not of an under-trained regime.
\autoref{fig:supp-rank-stability} confirms the consequence on the rank
axis: Spearman $\rho$ between the recall ranking at step $t$ and at
$2$M stays below $0.6$ at every pre-final checkpoint on every model
size, and dips below zero on SiT-S/L/XL early in training. Rank
stability is essentially absent for recall, in line with its
sampling-dominated character on this panel.

\myparagraph{The Gaussian shape carries over to every metric.}
$(\max\!-\!\min)/\sigma$ stays in $[2.9, 5.5]$ across all $76$ cells of
each metric, bracketing the $3.5$--$3.9$ a Gaussian sample of size
$n\!\in\![19, 25]$ predicts. The seed lottery is a smooth Gaussian-shaped
spread for each metric, not a tail of training failures, and the absence
of heavy tails is a property of the panel and not the choice of metric.

\input{figures/figS_supp_cov_band.tex}
\input{figures/figS_supp_rank_stability.tex}

\begin{table}[t]
  \centering
  \caption{\textbf{Per-(metric, model) CoV bands across the
  $\boldsymbol{200}$k--$\boldsymbol{2}$M sweep.} Companion to
  \autoref{tab:scaling-sweep}. Each cell reports the min / median / max
  $\mathrm{CoV}_{\text{between}}$ across the $19$ checkpoints, plus the
  $\sigma$-shrink factor $\sigma(200\mathrm{k})/\sigma(2\mathrm{M})$.
  Shrink factors below $1$ (\textbf{bold}) indicate metrics whose
  spread \emph{grows} with compute.}
  \label{tab:supp-cov-bands}
  \small
  \begin{tabular}{llcccc}
    \toprule
    Metric & Model & min CoV (\%) & median CoV (\%) & max CoV (\%) & $\sigma$-shrink \\
    \midrule
    Inception FID $\downarrow$
      & SiT-S  & 0.74 & 0.89 & 1.12 & 2.46$\times$ \\
      & SiT-B  & 1.01 & 1.13 & 1.74 & 1.68$\times$ \\
      & SiT-L  & 1.42 & 1.65 & 2.06 & 1.68$\times$ \\
      & SiT-XL & 1.23 & 1.44 & 1.74 & 2.16$\times$ \\
    \midrule
    DINOv2 FID $\downarrow$
      & SiT-S  & 0.48 & 0.57 & 0.66 & 1.69$\times$ \\
      & SiT-B  & 0.57 & 0.73 & 1.39 & 1.07$\times$ \\
      & SiT-L  & 1.11 & 1.47 & 1.59 & 1.29$\times$ \\
      & SiT-XL & 1.25 & 1.43 & 1.51 & 1.59$\times$ \\
    \midrule
    Inception Precision $\uparrow$
      & SiT-S  & 0.52 & 0.64 & 1.18 & 1.30$\times$ \\
      & SiT-B  & 0.40 & 0.53 & 0.79 & 1.18$\times$ \\
      & SiT-L  & 0.29 & 0.42 & 0.71 & 1.35$\times$ \\
      & SiT-XL & 0.21 & 0.30 & 0.67 & 2.27$\times$ \\
    \midrule
    Inception Recall $\uparrow$
      & SiT-S  & 0.34 & 0.42 & 0.73 & 1.94$\times$ \\
      & SiT-B  & 1.17 & 1.40 & 1.83 & \textbf{0.85$\times$} \\
      & SiT-L  & 1.04 & 1.36 & 1.74 & \textbf{0.86$\times$} \\
      & SiT-XL & 0.31 & 0.45 & 0.50 & 1.03$\times$ \\
    \midrule
    Inception Density $\uparrow$
      & SiT-S  & 0.85 & 1.06 & 1.73 & \textbf{0.87$\times$} \\
      & SiT-B  & 0.69 & 0.96 & 1.21 & \textbf{0.81$\times$} \\
      & SiT-L  & 0.63 & 0.82 & 1.33 & 1.17$\times$ \\
      & SiT-XL & 0.40 & 0.58 & 1.24 & 2.25$\times$ \\
    \midrule
    Inception Coverage $\uparrow$
      & SiT-S  & 0.63 & 0.83 & 1.66 & 1.28$\times$ \\
      & SiT-B  & 0.57 & 0.71 & 1.19 & 1.03$\times$ \\
      & SiT-L  & 0.35 & 0.46 & 0.92 & 1.85$\times$ \\
      & SiT-XL & 0.18 & 0.21 & 0.66 & 2.80$\times$ \\
    \bottomrule
  \end{tabular}
\end{table}

\subsection{Cherry-picking saves \texorpdfstring{$\boldsymbol{1.2}$--$\boldsymbol{2.9}{\boldsymbol{\times}}$}{1.2--2.9x} on every well-behaved metric}
\label{app:supp-speedup}

\noindent\emph{The lucky-seed speedup of \autoref{sec:results-speedup}
generalises across DINOv2 FID and Inception precision, density, and
coverage, with the largest savings ($2.0$--$2.9{\times}$ on SiT-L/XL)
on the fidelity PRDC metrics. Recall is compute-invariant on this
panel, so the framing does not apply.}

\myparagraph{Setup.}
For each metric we replicate \autoref{fig:lucky-speedup}: the
\emph{unlucky} seed is the worst seed at $2$M (the slowest-converging
seed in this panel for the given metric), and the \emph{lucky} seed is
the best at $2$M. The speedup is the ratio between $2$M and the first
checkpoint at which the lucky seed first matches the unlucky-seed
target. Direction is flipped for higher-better metrics so ``lucky''
always names the seed practitioners would prefer.

\myparagraph{The speedup transfers and grows on the fidelity axes.}
DINOv2 FID gives speedups in the same band as Inception FID
($1.18$--$1.67{\times}$ vs $1.18$--$1.82{\times}$,
\autoref{tab:supp-speedup}), so the lucky-seed effect is not specific
to one feature extractor. Inception precision, density, and coverage
push past it on the larger models, reaching $2.0$--$2.9{\times}$ on
SiT-L and SiT-XL. The ``free training-time speedup'' of
\autoref{sec:results-speedup} therefore strengthens when the
benchmark of interest is a precision/density/coverage number rather
than an FID number.

\myparagraph{Recall is compute-invariant.}
Inception recall yields nominal speedups of $6$--$10{\times}$ because
the lucky seed already exceeds the unlucky-seed-at-$2$M target at
$200$k--$300$k. This is not a statement about the value of compute but
about recall being approximately compute-invariant on this panel
(\autoref{tab:supp-cov-bands}, $\sigma$-shrink near or below $1$).
The lucky-speedup framing does not apply to recall and the entries are
flagged as degenerate.

\begin{table}[t]
  \centering
  \caption{\textbf{Lucky-seed speedup across metrics.} Companion to
  \autoref{fig:lucky-speedup}. Each cell reports the speedup factor
  $2\mathrm{M}/t^{\star}$, where $t^{\star}$ is the first checkpoint at
  which the lucky seed reaches the unlucky-seed-at-$2$M target. Recall
  entries are degenerate (compute-invariant, see
  \autoref{app:supp-scaling}).}
  \label{tab:supp-speedup}
  \small
  \begin{tabular}{lcccc}
    \toprule
    Metric & SiT-S & SiT-B & SiT-L & SiT-XL \\
    \midrule
    Inception FID $\downarrow$       & 1.25$\times$ & 1.18$\times$ & 1.82$\times$ & 1.82$\times$ \\
    DINOv2 FID $\downarrow$          & 1.18$\times$ & 1.25$\times$ & 1.67$\times$ & 1.54$\times$ \\
    Inception Precision $\uparrow$   & 1.25$\times$ & 1.43$\times$ & \textbf{2.50$\times$} & \textbf{2.50$\times$} \\
    Inception Density $\uparrow$     & 1.25$\times$ & 1.33$\times$ & \textbf{2.86$\times$} & 2.00$\times$ \\
    Inception Coverage $\uparrow$    & 1.18$\times$ & 1.33$\times$ & \textbf{2.86$\times$} & 2.22$\times$ \\
    \midrule
    \multicolumn{5}{l}{\emph{Recall (degenerate, lucky seed already exceeds unlucky-at-$2$M target by $200$--$300$k):}} \\
    Inception Recall $\uparrow$      & (6.67$\times$) & (10.0$\times$) & (10.0$\times$) & (10.0$\times$) \\
    \bottomrule
  \end{tabular}
\end{table}

\section{What does FID look like?}
\label{app:fid-visualization}

The figures in this appendix complement the quantitative results of
\autoref{sec:results} with a purely visual diagnostic: how does the
appearance of generated samples change as the FID of the underlying
checkpoint moves across its range, and does that change look the same
with and without classifier-free guidance?

\myparagraph{Setup.}
We pick $10$ ImageNet classes (golden retriever, tabby cat, macaw,
flamingo, cheeseburger, ice cream, volcano, alp, geyser, daisy) and
$10$ fixed initial-noise tensors. The figures show the first $6$ of
those noise seeds to keep each panel within page bounds. Every
(class, noise) pair is decoded with $50$-step
\texttt{flow\_euler\_sampler} at $16$ FID levels.
For the unguided panel, the levels are sampled along the full DiT
training trajectory and span $\mathrm{FID}\!\in\![13.5,\,87.4]$ (DiT-XL
fully trained $\rightarrow$ DiT-S undertrained). We uniformly subsample
these $16$ from the $32$ available checkpoints. For the guided panel,
the levels are DiT-XL checkpoints evaluated at the
golden-section-selected best CFG, spanning $\mathrm{FID}\!\in\![3.55,\,11.3]$.
The 16 FID-ordered checkpoints are split into two stacked half-panels
of $8$ columns each, so each figure is $8$ columns wide and $12$ rows
tall: the top half holds the lower-FID octave and the bottom half the
higher-FID octave. Each half-panel carries its own pastel colorbar, but
both share a single global FID scale, so the colour assigned to a
column is comparable across the two halves.

\myparagraph{What the galleries show.}
Holding initial noise fixed across columns isolates the effect of
training-time / guidance changes on the rendered image. The unguided
panels show the visual quality degrading monotonically along the FID
axis -- a smooth, well-formed object near the leftmost columns
collapses into class-correct but textureless blobs, then into noisy
patches, by the rightmost columns. The guided panels show that even
the highest-FID (worst) checkpoint produces visually clean, recognisable
samples once an appropriate CFG is applied -- consistent with the
\autoref{sec:results-guidance} observation that classifier-free
guidance compresses the entire FID range it is computed over.

The two condition-specific axes are shown back-to-back per class:
\autoref{fig:fid-vis-guided-00_golden_retriever}--\ref{fig:fid-vis-guided-09_daisy}
for the guided panels and
\autoref{fig:fid-vis-unguided-00_golden_retriever}--\ref{fig:fid-vis-unguided-09_daisy}
for the unguided panels.

\newcommand{\fidVisFigure}[3]{%
  \begin{figure}[p]
    \centering
    \includegraphics[width=\linewidth,height=0.82\textheight,keepaspectratio]{figures/fig_data/fid_class_panels/#1/class_#2.jpg}
    \caption{\textbf{#3.} Rows are $6$ fixed initial-noise seeds, and columns step through $8$ FID-ordered checkpoints. The two halves continue along the same FID axis. Pastel colorbars report per-column FID. See \autoref{app:fid-visualization} for the full setup.}
    \label{fig:fid-vis-#1-#2}
  \end{figure}%
  \clearpage
}

\fidVisFigure{guided}{00_golden_retriever}{Guided FID gallery -- golden retriever}
\fidVisFigure{guided}{01_tabby_cat}{Guided FID gallery -- tabby cat}
\fidVisFigure{guided}{02_macaw}{Guided FID gallery -- macaw}
\fidVisFigure{guided}{03_flamingo}{Guided FID gallery -- flamingo}
\fidVisFigure{guided}{04_cheeseburger}{Guided FID gallery -- cheeseburger}
\fidVisFigure{guided}{05_ice_cream}{Guided FID gallery -- ice cream}
\fidVisFigure{guided}{06_volcano}{Guided FID gallery -- volcano}
\fidVisFigure{guided}{07_alp}{Guided FID gallery -- alp}
\fidVisFigure{guided}{08_geyser}{Guided FID gallery -- geyser}
\fidVisFigure{guided}{09_daisy}{Guided FID gallery -- daisy}

\fidVisFigure{unguided}{00_golden_retriever}{Unguided FID gallery -- golden retriever}
\fidVisFigure{unguided}{01_tabby_cat}{Unguided FID gallery -- tabby cat}
\fidVisFigure{unguided}{02_macaw}{Unguided FID gallery -- macaw}
\fidVisFigure{unguided}{03_flamingo}{Unguided FID gallery -- flamingo}
\fidVisFigure{unguided}{04_cheeseburger}{Unguided FID gallery -- cheeseburger}
\fidVisFigure{unguided}{05_ice_cream}{Unguided FID gallery -- ice cream}
\fidVisFigure{unguided}{06_volcano}{Unguided FID gallery -- volcano}
\fidVisFigure{unguided}{07_alp}{Unguided FID gallery -- alp}
\fidVisFigure{unguided}{08_geyser}{Unguided FID gallery -- geyser}
\fidVisFigure{unguided}{09_daisy}{Unguided FID gallery -- daisy}



\end{document}